%% file: main.tex
\documentclass{article} %
\usepackage{iclr2027_conference,times}

\input{math_commands.tex}

\usepackage{hyperref}
\usepackage{url}

\usepackage{longtable}
\usepackage{microtype} 
\usepackage{booktabs}
\usepackage{graphicx}
\usepackage{float}
\usepackage{multirow}
\usepackage{makecell}
\usepackage{wrapfig}
\usepackage{enumitem}
\usepackage{subcaption}
\usepackage{pgffor} %
\usepackage{amsmath, amssymb, amsthm, bm, mathtools}
\usepackage{amsfonts}
\usepackage{nicefrac}
\usepackage{xcolor}
\usepackage{pgfplots}
\pgfplotsset{compat=1.18}
\usepgfplotslibrary{groupplots}
\usetikzlibrary{calc} %
\usetikzlibrary{decorations.pathreplacing, patterns}
\usetikzlibrary{arrows.meta} %
\definecolor{myred}{RGB}{235, 58, 52}
\definecolor{myblue}{RGB}{0, 128, 255}
\definecolor{myorange}{RGB}{235, 143, 52}
\definecolor{mygreen}{RGB}{0, 204, 0}
\definecolor{greenbar}{RGB}{30, 100, 60}
\definecolor{mypurple}{RGB}{190, 83, 212}

\definecolor{cAnyPrec}{RGB}{243, 190, 114}   %
\definecolor{cAnyBCQ}{RGB}{196, 161, 214}    %
\definecolor{cMatGPTQ}{RGB}{235, 138, 138}   %

\definecolor{cAnyPrecfig10}{RGB}{230, 155, 55}    %
\definecolor{cAnyBCQfig10}{RGB}{145, 95, 175}     %
\definecolor{cMatGPTQfig10}{RGB}{210, 75, 75}      %

\theoremstyle{plain}
\newtheorem{theorem}{Theorem}
\newtheorem{definition}{Definition}

\usepackage{algorithm}
\usepackage{algpseudocode}

\title{Multi-Bitwidth Quantization \\ for LLMs Using Additive Codebooks}

\author{Liza Babaoglu, Shuangyi Chen, Ashish Khisti \\
Department of Electrical and Computer Engineering, University of Toronto \\
{\small \texttt{\{liza.babaoglu,shuangyi.chen\}@mail.utoronto.ca}, \texttt{akhisti@ece.utoronto.ca}}
}

\iclrfinalcopy %
\begin{document}
\maketitle
\fancyhead{}
\renewcommand{\headrulewidth}{0pt}

\begin{abstract}
As large language models (LLMs) are increasingly deployed across heterogeneous hardware with varying resource constraints, the ability to adaptively manage the performance-efficiency trade-off without retraining is critical. We propose \textbf{Drop-by-Drop}, a novel \textit{multi-bitwidth} post-training quantization framework enabling inference-time precision control over LLM weights from a single quantized model. As theoretical motivation, we prove that Gaussian sources are successively refinable under a weighted mean squared error distortion motivated by LLM loss functions: successive refinement across distortion levels incurs no rate penalty relative to encoding at each level separately. Drop-by-Drop approximates this hierarchical structure in practice by incorporating Matryoshka-style supervision into additive codebook training, inducing an ordering in which codebook prefixes yield accurate \textit{partial} reconstructions at each precision level. Furthermore, a block-Hadamard rotation brings the weight distribution closer to our Gaussian source assumption. The result is a single model that serves multiple bitwidths by dropping codebooks, reducing storage and quantization cost relative to the static per-bitwidth models. Across Qwen, LLaMA, Gemma, and Mistral, Drop-by-Drop achieves lower perplexity than state-of-the-art multi-bitwidth methods with competitive zero-shot accuracy, while our specialized kernels further reduce decoding latency.
\end{abstract}

\section{Introduction}
The rapid development of large language models (LLMs) has advanced natural language
processing. Advancements in both algorithmic design and large-scale computational infrastructure have enabled open-access models such as LLaMA~\citep{llama3}, Mistral~\citep{mistral}, Gemma~\citep{gemma2024}, and Qwen~\citep{QWEN} to achieve strong performance across a wide range of language tasks~\citep{informationretrieval, softhard}. Despite their impressive capabilities, LLMs must ultimately be deployed for inference in real-world systems.

\begin{wrapfigure}{r}{0.5\textwidth}
\centering
\vspace{-10pt}
\resizebox{\linewidth}{!}{%
\begin{tikzpicture}[
  font=\footnotesize,
  base/.style={draw, rounded corners=2pt, line width=0.4pt, align=center,
               minimum height=7.5mm, minimum width=11mm, inner sep=2pt},
  model/.style={base, fill=black!6},
  llm/.style={base, fill=black!6, minimum width=15mm, minimum height=8.5mm, inner sep=2pt},
  gpu/.style={base, fill=black!3, minimum width=10mm},
  target/.style={base, fill=teal!12, minimum width=27mm, minimum height=23mm},
  slab/.style={draw, line width=0.35pt, fill=white, rounded corners=1pt,
               minimum width=24mm, minimum height=4.2mm, inner sep=0pt,
               font=\scriptsize},
  stage/.style={base, fill=white, minimum width=15mm, minimum height=5.5mm, font=\scriptsize},
  ar/.style={-{Latex[length=1.5mm]}, line width=0.4pt},
  lbl/.style={font=\scriptsize, inner sep=1pt},
  outlbl/.style={font=\small\bfseries, inner sep=1pt}
]

\node[lbl, anchor=west, font=\footnotesize\bfseries] at (-0.9, 4.7)
  {Conventional method (static)};

\node[llm] (base1) at (0, 3.2) {LLM\\ weights};

\node[model] (q1) at (2.0, 4.0) {\(B_1\)};
\node[model] (q2) at (2.0, 3.2) {\(B_2\)};
\node[model] (q3) at (2.0, 2.4) {\(B_3\)};

\node[gpu] (f1) at (3.6, 4.0) {GPU};
\node[gpu] (f2) at (3.6, 3.2) {GPU};
\node[gpu] (f3) at (3.6, 2.4) {GPU};

\node[target] (t1) at (6.0, 3.2) {};
\node[lbl, anchor=north] at (t1.north) [yshift=-0.8mm] {Serving node / device};
\node[slab] at (6.0, 3.50) {\(B_1\) checkpoint};
\node[slab] at (6.0, 3.05) {\(B_2\) checkpoint};
\node[slab] at (6.0, 2.60) {\(B_3\) checkpoint};

\draw[ar] (base1) -- (q1);
\draw[ar] (base1) -- (q2);
\draw[ar] (base1) -- (q3);
\draw[ar] (q1) -- (f1);
\draw[ar] (q2) -- (f2);
\draw[ar] (q3) -- (f3);
\draw[ar] (f1) -- (f1 -| t1.west);
\draw[ar] (f2) -- (f2 -| t1.west);
\draw[ar] (f3) -- (f3 -| t1.west);

\node[outlbl, anchor=west] (p1) at (7.8, 3.8) {\(B_1\)};
\node[outlbl, anchor=west] (p2) at (7.8, 3.2) {\(B_2\)};
\node[outlbl, anchor=west] (p3) at (7.8, 2.6) {\(B_3\)};
\draw[ar] (t1.east |- p1) -- (p1.west);
\draw[ar] (t1.east |- p2) -- (p2.west);
\draw[ar] (t1.east |- p3) -- (p3.west);

\draw[line width=0.4pt, dashed, gray!60] (-0.9, 1.6) -- (8.1, 1.6);
\node[stage] at (2.0, 1.6) {Quantize};
\node[stage] at (3.6, 1.6) {Finetune};
\node[stage] at (6.0, 1.6) {Deploy};

\node[llm]   (base2) at (0, 0.2) {LLM\\ weights};
\node[model] (q4)    at (2.0, 0.2) {\(B_3\)};
\node[gpu]   (f4)    at (3.6, 0.2) {GPU};

\node[target, minimum height=17mm] (t2) at (6.0, 0.2) {};
\node[lbl, anchor=north] at (t2.north) [yshift=-0.8mm] {Serving node / device};
\node[slab] at (6.0, 0.05) {\(B_3\) checkpoint};

\draw[ar] (base2) -- (q4);
\draw[ar] (q4) -- (f4);
\draw[ar] (f4) -- (f4 -| t2.west);

\node[outlbl, anchor=west] (o1) at (7.8, 0.7)  {\(B_1\)};
\node[outlbl, anchor=west] (o2) at (7.8, 0.2)  {\(B_2\)};
\node[outlbl, anchor=west] (o3) at (7.8, -0.3) {\(B_3\)};
\draw[ar] (t2.east |- o1) -- (o1.west);
\draw[ar] (t2.east |- o2) -- (o2.west);
\draw[ar] (t2.east |- o3) -- (o3.west);

\node[lbl, anchor=west, font=\footnotesize\bfseries] at (-0.9, -1.0)
  {Drop-by-Drop method (multi-bitwidth)};

\node[font=\footnotesize, align=center] at (3.7, -1.55)
  {\(B_1\): memory-limited \quad \(B_2\): typical load \quad \(B_3\): full accuracy};

\end{tikzpicture}%
}
\caption{Deploying multiple precision levels}
\label{resource_adaptive}
\vspace{-10pt}
\end{wrapfigure}

Modern LLMs often contain tens or hundreds of billions of parameters, imposing substantial memory and computational demand. Deployment is particularly challenging across heterogeneous platforms that include large cloud servers, personal devices, and edge hardware with strict memory, latency, or power constraints~\citep{MobileDevices, EdgeDevices}. Quantization addresses part of this challenge by compressing model parameters into finite-precision representations, replacing high-precision floating point values. However, most quantization methods produce models at a fixed compression level, which is rarely optimal for all deployment settings. Maintaining separately quantized models for different resource constraints is often impractical due to storage and maintenance overhead. Motivated by this, we explore \textit{multi-bitwidth quantization} for LLM
inference, where a single model operates at multiple precision levels
without retraining or storing a separate checkpoint per bitwidth. As
illustrated in Figure~\ref{resource_adaptive}, supporting $K$ bitwidths
conventionally requires running quantization and finetuning $K$ distinct
times, incurring the overhead of storing $K$ individual checkpoints. In
contrast, our approach runs once and stores a single checkpoint that
accommodates several operating points, at a storage cost set only by the highest supported bitwidth.

From an information-theoretic perspective, multi-bitwidth quantization naturally relates to the concept of \textit{successive source refinement}, where progressively adding bits yields increasingly accurate reconstructions. While classical results establish successive refinability for certain sources and distortion measures~\citep{successiveref}, its connection to multi-bitwidth quantization for LLMs has not been explored. We bridge this gap by connecting successive refinement theory with practical LLM quantization.

To enable this precision control in LLMs, we propose \textbf{Drop-by-Drop}, a novel multi-bitwidth post-training quantization framework that allows inference-time control over the precision of model weights. Our method builds on AQLM~\citep{vahe}, which represents weights as a sum of codewords drawn from multiple learned codebooks using additive quantization. While AQLM achieves strong accuracy at low bitwidths, it relies on a fixed number of codebooks and therefore produces a static precision profile. Drop-by-Drop addresses this limitation by introducing Matryoshka-style supervision~\citep{matryoshkaloss} during codebook training, explicitly encouraging intermediate subsets of codebooks to produce accurate \textit{partial} reconstructions. This design enables progressive compression by simply dropping codebooks during inference, allowing a single quantized model to operate across multiple resource budgets. Our main contributions are summarized below:

\begin{enumerate}[leftmargin=*]
\item \textbf{Theoretical Foundation:} In Section~\ref{section:theoretical_Contribution}, we establish a theorem proving that Gaussian sources are successively refinable under a weighted mean squared error distortion measure motivated by the loss functions used in LLMs. This extension motivates the design of our algorithmic framework.

\item \textbf{Algorithmic Framework:} In Section~\ref{section:application}, we propose \textbf{Drop-by-Drop}, which enables inference-time precision control from a single quantized model by progressively dropping additive codebooks.

\item \textbf{Empirical Evaluation:} In Section~\ref{section:experiments}, we demonstrate that Drop-by-Drop maintains low perplexity and strong accuracy across the supported bitwidths, with lower decoding latency and graceful performance degradation under resource constraints, without retraining or recalibration.
\end{enumerate}

\section{Preliminaries}

\paragraph{Rate and Distortion.} Let $\mathsf{W}$ be a discrete random variable taking values in $\mathcal{W}$, with probability mass function  $p(w)$, $w \in \mathcal{W}$. In the lossy setting, the source is encoded at rate $R$ (in bits per source symbol) and reconstructed as $\widehat{\mathsf{W}}$. 
The fidelity of this reconstruction is quantified by a distortion function
$d : \mathcal{W} \times \widehat{\mathcal{W}} \to \mathbb{R}$,
and performance is measured by the expected distortion
$\mathbb{E}[d(\mathsf{W}, \widehat{\mathsf{W}})]$.
Let $D \ge 0$ denote the maximum allowable expected distortion. For the distortion function $d(w,\widehat{w})$ defined on $\mathcal{W}\times\widehat{\mathcal{W}}$, the rate-distortion function $R(D)$ and its inverse $D(R)$ are given by

\vspace{-4mm}

\[
R(D) = \min_{p(\widehat{w}|w)} I(\mathsf{W};\widehat{\mathsf{W}}) 
\quad \text{and} \quad
D(R) = \min_{p(\widehat{w}|w):\, I(\mathsf{W};\widehat{\mathsf{W}})\le R} 
\mathbb{E}\!\left[d(\mathsf{W},\widehat{\mathsf{W}})\right].
\]

\vspace{-3mm}

The minimization is over all $p(\widehat{w}\mid w)$ subject to $\sum_{w,\widehat{w}} p(w)\,p(\widehat{w}\mid w)\, d(w,\widehat{w}) \le D$. These are single-letter characterizations; the operational definition is in \citet[Ch.~10]{textbook}.

\paragraph{Successive Refinement.} A source $(\mathsf{W}, d)$ is successively refinable if, for any finite distortion chain $
D_1 > D_2 > \cdots > D_K > 0,$
there exists a single embedded code whose $k$-prefix achieves the optimal rate–distortion pair $(R(D_k), D_k)$ for all $k=1,\dots,K$, with no rate loss compared to encoding directly at $D_k$. That is, one code can operate optimally at multiple distortion levels. We define successive refinability for general sources in Definition~\ref{def:SR-general}, and for Gaussian sources under mean squared error (MSE) distortion in Theorem~\ref{thm:gaussian-sr}.

\paragraph{Loss Formulation for LLM Weight Quantization.}
LLMs are composed of stacked layers, each performing a linear transformation followed by a non-linearity. In this work, we work on quantizing the LLM weights of a pre-trained model. For a given layer, let $\boldsymbol{\mathsf{W}} \in \mathbb{R}^{d_{\mathrm{out}} \times d_{\mathrm{in}}}$ denote its weight matrix, where $d_{\mathrm{in}}$ and $d_{\mathrm{out}}$ denote the input and output dimensions of the layer, respectively. We have a fixed data matrix
$\boldsymbol{X} \in \mathbb{R}^{d_{\mathrm{in}} \times n},$
collected from representative inputs. For the first layer, $\boldsymbol{X}$ corresponds to token embeddings of the dataset. Whereas, for subsequent layers, $\boldsymbol{X}$ consists of intermediate activations produced by preceding layers. The corresponding layer outputs are
$\boldsymbol{\mathsf{Y}} = \boldsymbol{\mathsf{W}} \boldsymbol{X} \in \mathbb{R}^{d_{\mathrm{out}} \times n}.$ Notation description is in Appendix~\ref{app:notation}.

Let $\widehat{\boldsymbol{\mathsf{W}}}$ denote a \emph{quantized reconstruction} of $\boldsymbol{\mathsf{W}}$. To evaluate the impact of quantization on the layer outputs, we define the \emph{weighted mean squared error (WMSE)}:

\vspace{-4mm}
$$d_{\mathrm{WMSE}}(\boldsymbol{\mathsf{W}}, \widehat{\boldsymbol{\mathsf{W}}})
\;\coloneqq\;
\mathbb{E}\!\left[\bigl\|(\boldsymbol{\mathsf{W}} - \widehat{\boldsymbol{\mathsf{W}}})\boldsymbol{X}\bigr\|_F^2\right].$$

\vspace{-4mm}

When $\boldsymbol{X}=\boldsymbol{I}$, the identity matrix, this reduces to the standard MSE between $\boldsymbol{W}$ and $\widehat{\boldsymbol{W}}$. However, in practice, we do not have access to the distribution of $\boldsymbol{W}$ and observe only a single realization:

\vspace{-4mm}

\begin{equation}
\label{eq:wmse-task}
L_{\mathrm{WMSE}}(\boldsymbol{\mathsf{W}}, \widehat{\boldsymbol{\mathsf{W}}})
\;\coloneqq\;
\frac{1}{d_{\mathrm{out}} n} 
\bigl\| (\boldsymbol{\mathsf{W}} - \widehat{\boldsymbol{\mathsf{W}}}) \boldsymbol{X} \bigr\|_F^2
=
\frac{1}{d_{\mathrm{out}}} 
\operatorname{tr} \Big[
(\boldsymbol{\mathsf{W}} - \widehat{\boldsymbol{\mathsf{W}}}) 
\frac{1}{n} \boldsymbol{X} \boldsymbol{X}^\top 
(\boldsymbol{\mathsf{W}} - \widehat{\boldsymbol{\mathsf{W}}})^\top
\Big].
\end{equation}

\section{Background \& Related Work}
\label{sec:background}

\subsection{The Multi-Bitwidth Quantization Landscape}
\label{sec:multi-bitwidth}
Post-training quantization (PTQ) compresses a pre-trained large language model (LLM) using a small calibration set, avoiding the compute overhead of quantization-aware training (QAT) at billion-parameter scale. Well established PTQ methods compress weights to lower precision with minimal accuracy loss. However, each of these methods fixes a single target bitwidth: serving $K$ operating precisions requires calibrating, storing, and deploying $K$ independent models. Our interest is in the multi-bitwidth setting. Among the vast literature on PTQ, a smaller set of methods achieves multi-bitwidth operation. Closest to our setting are the following: Any-Precision LLM~\citep{anyprecisionllm} designs a non-uniform scalar quantizer where lower-bit representations act as nested prefixes of higher-bit ones. AnyBCQ~\citep{anybcq} shares frozen binary bit-planes across precisions. MatGPTQ~\citep{matgptq} formulates a one-shot PTQ objective across multiple bitwidths, using cross-bit error compensation to generate a sliceable model in a single pass. 

AnyBCQ and Any-Precision LLM do not work with the weighted loss in Equation~\ref{eq:wmse-task}, instead they use different adaptations for their own quantizers. AnyBCQ fits its bit-planes via a data-unaware objective and applies a data-aware calibration loss to scale factors in a secondary stage. Meanwhile, Any-Precision LLM inherits SqueezeLLM's diagonal approximation of the empirical Fisher information matrix~\citep{squeezellm}, which omits the off-diagonal terms that couple the weights~\citep{guidedquant}. Another weakness is that all three frameworks rely on scalar quantization (SQ), whose performance is tightly constrained by information-theoretic bounds~\citep{llvq}. In the low-bit regime, vector quantization (VQ) achieves a lower distortion than SQ at the same rate~\cite[Ch.~10]{vqbetter}. Empirically, scaling the vector dimension, such as grouping weights into 2D, 4D, and 8D blocks, yields monotonic improvements in signal-to-quantization-noise ratio~\citep{gptvq, unisvq}, resulting in lower perplexity~\citep{vahe, QUIPsharp, ccq} and higher zero-shot accuracy~\citep{VPTQ, lcqat}. Therefore, we utilize AQLM as our additive-codebook building block, quantizing weights in 8-dimensional vectors and building on their WMSE objective in Equation~\ref{eq:wmse-task}.

\subsection{Additive Quantization and AQLM}
\label{sec:additive-aqlm}

\begin{wrapfigure}[10]{r}{0.42\textwidth}
  \vspace{-1.3cm}
  \centering
  \includegraphics[width=\linewidth]{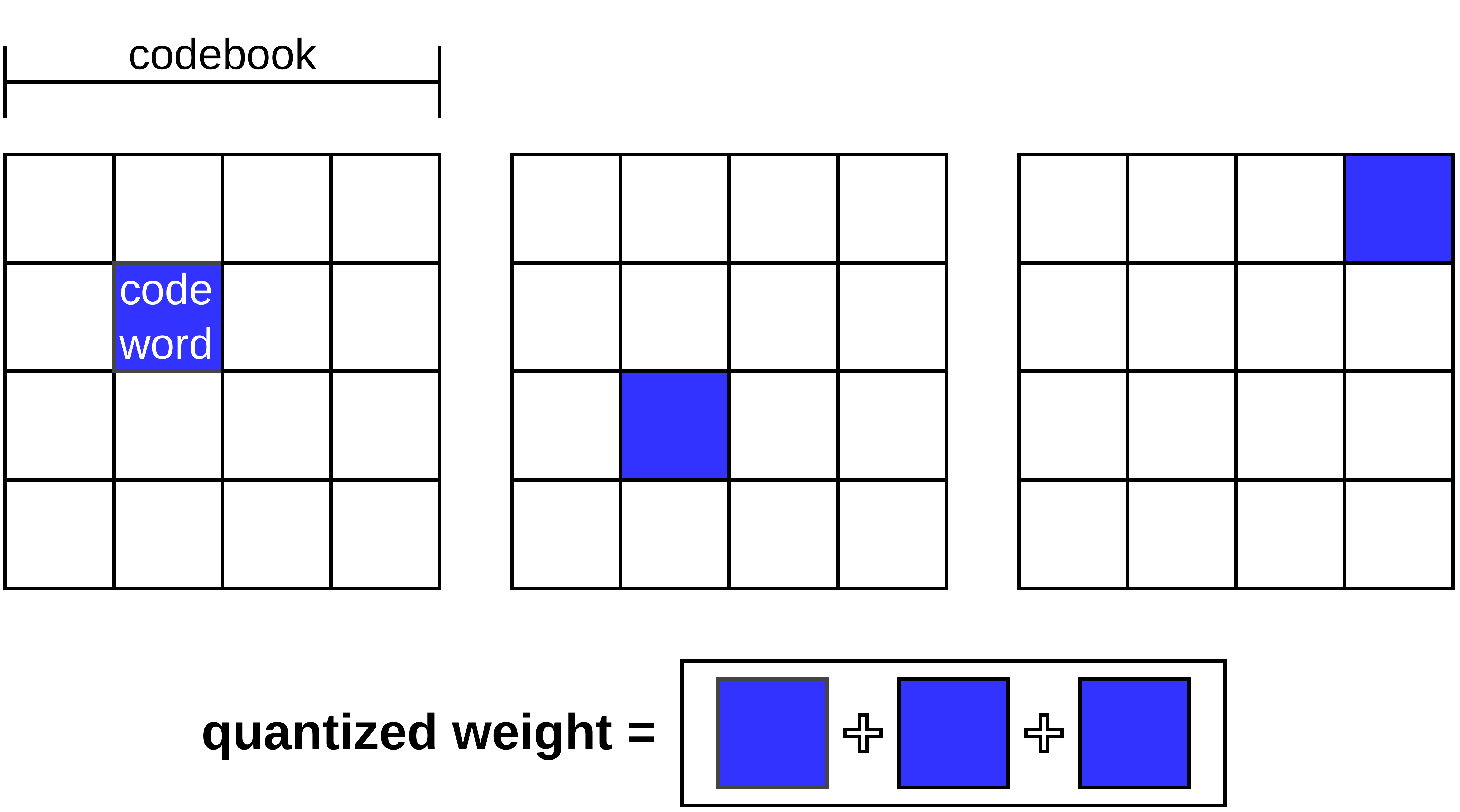}
  \caption{Illustration of additive codebook-based reconstruction.}
  \label{fig:codebook_visual.png}
\end{wrapfigure} Among vector approaches, additive quantization (AQ)~\citep{babenko} removes the strict ordering imposed by residuals by representing weights as \textit{sums} of multiple independently selected vectors drawn from learned codebooks. These components are optimized jointly rather than sequentially, resulting in a more flexible representation with higher expressive capacity at extreme compression ratios. Originally developed for approximate nearest neighbor search, AQ improves approximation quality over standard product and scalar quantization by reconstructing each weight group as a sum of vectors from multiple codebooks~\citep{babenko}. Building on this idea, AQLM applies multi-codebook additive PTQ to compress LLM weights, representing each weight tensor as a sum of quantized vectors selected from distinct codebooks.

Given a linear layer with weight matrix $\boldsymbol{\mathsf{W}}\in \mathbb{R}^{d_{\text{out}} \times d_{\text{in}}}$ and data inputs $\boldsymbol{X} \in \mathbb{R}^{d_{\text{in}} \times n}$, AQLM aims to find quantized weights $\widehat{\boldsymbol{\mathsf{W}}}$ that minimize the squared error between the original and quantized outputs:
$
\min_{\widehat{\boldsymbol{\mathsf{W}}}} \; \left\| \boldsymbol{\mathsf{W}} \boldsymbol{X} - \widehat{\boldsymbol{\mathsf{W}}} \boldsymbol{X} \right\|_F^2.$ In AQLM, the weight matrix \( \boldsymbol{\mathsf{W}} \) is split into groups of \( g \) consecutive elements. Each group is approximated by the sum of \( M \) vectors, each selected from a learned codebook \( C_1, C_2, \ldots, C_M \). Each codebook contains \( 2^B \) codewords, where \( B \) is the number of bits used to index the codewords. The selection for each group is encoded by \( M \) one-hot vectors \( b_{ijm} \), where \( b_{ijm} \) indicates the codeword chosen from the \( m \)-th codebook for the \( j \)-th group in the \( i \)-th output row. Formally, a group is reconstructed as $\sum_{m=1}^M C_m b_{ijm}$. Thus, the assignments \( b \) specify the combination of codewords whose sum reconstructs each weight group. Figure~\ref{fig:codebook_visual.png} shows how each weight group is reconstructed by summing selected codewords from multiple codebooks. To form the full quantized weight matrix \( \widehat{\boldsymbol{\mathsf{W}}} \), AQLM concatenates the reconstructed groups along the input dimension, where \( \oplus \) denotes vector concatenation across the \( d_{\mathrm{in}}/g \) groups per output row:
$\widehat{\boldsymbol{\mathsf{W}}}_i = \sum_{m=1}^M C_m b_{i,1,m} \oplus \cdots \oplus \sum_{m=1}^M C_m b_{i,d_{\mathrm{in}}/g,m}.$ The optimization process of AQLM is described in Appendix~\ref{app:AQLM-opt-process}. Next, we review Matryoshka Representation Learning, the supervision scheme we adapt to AQLM's additive codebooks to achieve the multi-bitwidth inference.

\subsection{Matryoshka Representation Learning}
\label{sec:matryoshka}
Matryoshka Representation Learning~\citep{matryoshkaloss} enables a single model to adaptively use representations at varying fidelity by encouraging intermediate prefixes to approximate the final output, supporting coarse-to-fine inference. Related work includes Matryoshka Query Transformer for Large Vision-Language Models~\citep{LLMVisionmat}, which applies the same principle in multi-modal settings, and MatQuant~\citep{Matquant}, which extends Matryoshka supervision to quantized models via QAT across multiple bitwidths. MatQuant exploits a hierarchy that is inherent to the integer datatype itself: the two most significant bits of int4 equal int2, and int4 in turn forms the most significant bits of int8, so lower-precision models are recovered by bit-slicing. MatGPTQ~\citep{matgptq} replaces MatQuant's QAT with the one-shot PTQ objective described above, casting the same bit-slicing hierarchy as a post-training rather than a training-time procedure. We extend Matryoshka supervision to additive vector codebooks. Since additive codebooks are optimized jointly and have no inherent ordering, we consider each codebook prefix to serve as a standalone lower-bitwidth model, rather than dropping codebooks post-hoc.

\section{Our Theoretical Contribution}
\label{section:theoretical_Contribution}

\begin{wrapfigure}[16]{r}{0.55\textwidth}
\centering
\vspace{-0.6cm}
\tikzset{inactive/.style={pattern=north east lines, pattern color=purple!50!black}}

\begin{subfigure}[t]{0.24\textwidth}
\centering
\begin{tikzpicture}[scale=0.52]
\pgfmathsetmacro{\lami}{5.5}
\pgfmathsetmacro{\lamii}{5.0}
\pgfmathsetmacro{\lamiii}{6.4}
\pgfmathsetmacro{\muone}{4.2}
\pgfmathsetmacro{\mutwo}{2.8}
\pgfmathsetmacro{\muthree}{1.4}
\def\bw{1.4}

\fill[blue!10]      (0,0)         rectangle (\bw,\lami);
\fill[blue!30]      (0,\muthree)  rectangle (\bw,\lami);
\fill[blue!10!white](0,0)         rectangle (\bw,\muthree);
\draw[blue!60, thick] (0,0) rectangle (\bw,\lami);
\node[below, font=\scriptsize] at (\bw/2, 0) {$1$};
\node[above, font=\scriptsize, blue!80] at (\bw/2, \lami) {$\lambda_1$};

\draw[red!70!black, dashed, thin] (0,\muone) -- (\bw,\muone);
\draw[red!70!black, dashed, thin] (0,\mutwo) -- (\bw,\mutwo);

\draw[decorate, decoration={brace, amplitude=3pt, mirror}, purple!70!black]
    (\bw+0, \muone) -- (\bw+0, \mutwo)
    node[midway, xshift=-10pt, font=\scriptsize, purple!70!black] {$\Delta R_1^{(2)}$};

\draw[decorate, decoration={brace, amplitude=3pt, mirror}, purple!70!black]
    (\bw+0, \mutwo) -- (\bw+0, \muthree)
    node[midway, xshift=-10pt, font=\scriptsize, purple!70!black] {$\Delta R_1^{(3)}$};

\fill[blue!10]      (1.8,0)         rectangle (1.8+\bw,\lamii);
\fill[blue!30]      (1.8,\muthree)  rectangle (1.8+\bw,\lamii);
\fill[blue!10!white](1.8,0)         rectangle (1.8+\bw,\muthree);
\draw[blue!60, thick] (1.8,0) rectangle (1.8+\bw,\lamii);
\node[below, font=\scriptsize] at (1.8+\bw/2, 0) {$2$};
\node[above, font=\scriptsize, blue!80] at (1.8+\bw/2, \lamii) {$\lambda_2$};

\draw[red!70!black, dashed, thin] (1.8,\muone) -- (1.8+\bw,\muone);
\draw[red!70!black, dashed, thin] (1.8,\mutwo) -- (1.8+\bw,\mutwo);

\draw[decorate, decoration={brace, amplitude=3pt, mirror}, purple!70!black]
    (1.8+\bw+0, \muone) -- (1.8+\bw+0, \mutwo)
    node[midway, xshift=-10pt, font=\scriptsize, purple!70!black] {$\Delta R_2^{(2)}$};

\draw[decorate, decoration={brace, amplitude=3pt, mirror}, purple!70!black]
    (1.8+\bw+0, \mutwo) -- (1.8+\bw+0, \muthree)
    node[midway, xshift=-10pt, font=\scriptsize, purple!70!black] {$\Delta R_2^{(3)}$};

\fill[blue!10]      (3.6,0)         rectangle (3.6+\bw,\lamiii);
\fill[blue!30]      (3.6,\muthree)  rectangle (3.6+\bw,\lamiii);
\fill[blue!10!white](3.6,0)         rectangle (3.6+\bw,\muthree);
\draw[blue!60, thick] (3.6,0) rectangle (3.6+\bw,\lamiii);
\node[below, font=\scriptsize] at (3.6+\bw/2, 0) {$3$};
\node[above, font=\scriptsize, blue!80] at (3.6+\bw/2, \lamiii) {$\lambda_3$};

\draw[red!70!black, dashed, thin] (3.6,\muone) -- (3.6+\bw,\muone);
\draw[red!70!black, dashed, thin] (3.6,\mutwo) -- (3.6+\bw,\mutwo);

\draw[decorate, decoration={brace, amplitude=3pt, mirror}, purple!70!black]
    (3.6+\bw+0, \muone) -- (3.6+\bw+0, \mutwo)
    node[midway, xshift=-10pt, font=\scriptsize, purple!70!black] {$\Delta R_3^{(2)}$};

\draw[decorate, decoration={brace, amplitude=3pt, mirror}, purple!70!black]
    (3.6+\bw+0, \mutwo) -- (3.6+\bw+0, \muthree)
    node[midway, xshift=-10pt, font=\scriptsize, purple!70!black] {$\Delta R_3^{(3)}$};

\draw[red!70!black, thick, dashed] (-0.5,\muone)   -- (5.1,\muone);
\draw[red!70!black, thick, dashed] (-0.5,\mutwo)   -- (5.1,\mutwo);
\draw[red!70!black, thick, dashed] (-0.5,\muthree) -- (5.1,\muthree);

\node[left, font=\scriptsize, red!70!black] at (-0.5,\muone)   {$\mu_1$};
\node[left, font=\scriptsize, red!70!black] at (-0.5,\mutwo)   {$\mu_2$};
\node[left, font=\scriptsize, red!70!black] at (-0.5,\muthree) {$\mu_3$};

\draw[->, thick] (-0.5,0) -- (-0.5,6.5) node[above, font=\scriptsize] {$\lambda_i$};
\draw[->, thick] (-0.5,0) -- (5.3,0)    node[right, font=\scriptsize] {$i$};

\node[below, font=\scriptsize] at (2.7,-0.8) {Successive Refinement};

\end{tikzpicture}
\end{subfigure}%
\hspace{0.5cm}%
\begin{subfigure}[t]{0.24\textwidth}
\centering
\begin{tikzpicture}[scale=0.52]
\pgfmathsetmacro{\lami}{5.5}
\pgfmathsetmacro{\lamii}{5.0}
\pgfmathsetmacro{\lamiii}{6.4}
\pgfmathsetmacro{\muthree}{1.4}
\def\bw{1.4}

\fill[teal!10]      (0,0)        rectangle (\bw,\lami);
\fill[teal!35]      (0,\muthree) rectangle (\bw,\lami);
\fill[teal!10!white](0,0)        rectangle (\bw,\muthree);
\draw[teal!60!black, thick] (0,0) rectangle (\bw,\lami);
\node[below, font=\scriptsize] at (\bw/2, 0) {$1$};
\node[above, font=\scriptsize, teal!60!black] at (\bw/2, \lami) {$\lambda_1$};

\fill[teal!10]      (1.8,0)        rectangle (1.8+\bw,\lamii);
\fill[teal!35]      (1.8,\muthree) rectangle (1.8+\bw,\lamii);
\fill[teal!10!white](1.8,0)        rectangle (1.8+\bw,\muthree);
\draw[teal!60!black, thick] (1.8,0) rectangle (1.8+\bw,\lamii);
\node[below, font=\scriptsize] at (1.8+\bw/2, 0) {$2$};
\node[above, font=\scriptsize, teal!60!black] at (1.8+\bw/2, \lamii) {$\lambda_2$};

\fill[teal!10]      (3.6,0)        rectangle (3.6+\bw,\lamiii);
\fill[teal!35]      (3.6,\muthree) rectangle (3.6+\bw,\lamiii);
\fill[teal!10!white](3.6,0)        rectangle (3.6+\bw,\muthree);
\draw[teal!60!black, thick] (3.6,0) rectangle (3.6+\bw,\lamiii);
\node[below, font=\scriptsize] at (3.6+\bw/2, 0) {$3$};
\node[above, font=\scriptsize, teal!60!black] at (3.6+\bw/2, \lamiii) {$\lambda_3$};

\draw[red!70!black, thick, dashed] (-0.5,\muthree) -- (5.1,\muthree);
\node[left, font=\scriptsize, red!70!black] at (-0.5,\muthree) {$\mu_3$};

\draw[->, thick] (-0.5,0) -- (-0.5,6.5) node[above, font=\scriptsize] {$\lambda_i$};
\draw[->, thick] (-0.5,0) -- (5.3,0)    node[right, font=\scriptsize] {$i$};

\node[below, font=\scriptsize] at (2.7,-0.8) {Single-shot};
\end{tikzpicture}
\end{subfigure}

\begin{tikzpicture}[scale=0.55]
    \fill[blue!30]  (5,0) rectangle (5.5,0.35);
    \draw[blue!60]  (5,0) rectangle (5.5,0.35);

    \fill[teal!35]  (5.6,0) rectangle (6.1,0.35);
    \draw[teal!60!black] (5.6,0) rectangle (6.1,0.35);

    \node[font=\scriptsize] at (6.75,0.175) {Rate};

    \fill[blue!10]  (8.0,0) rectangle (8.5,0.35);
    \draw[blue!60]  (8.0,0) rectangle (8.5,0.35);

    \fill[teal!10!white] (8.6,0) rectangle (9.1,0.35);
    \draw[teal!60!black]  (8.6,0) rectangle (9.1,0.35);

    \node[font=\scriptsize] at (10.35,0.175) {Distortion};
\end{tikzpicture}
\caption{Rate-distortion equivalence}
\label{fig:sr_equals_single_shot}
\end{wrapfigure}

We now characterize the \emph{successive refinability of Gaussian sources under
weighted MSE}, culminating in Theorem~\ref{thm:WMSE-SR-main}. We model each weight
matrix $\boldsymbol{\mathsf{W}}$ as Gaussian entries. This assumption is standard in rate-distortion
analysis~\citep{sakrison1968rate}, admitting exact rate-distortion
characterizations and known successive refinability under MSE. In practice, LLM
weight distributions are only \emph{approximately} Gaussian: as shown in
Appendix~\ref{app:gaussianplots}, raw weights follow a broadly bell-shaped
profile but with heavier tails. Section~\ref{section:application} discusses how a
Block-Hadamard rotation reduces this gap prior to quantization. At the same time, practical PTQ methods optimize data-aware weighted MSE objectives (WMSE), such as Equation~\ref{eq:wmse-task}, where the weighting arises from the data matrix $\boldsymbol{X}$. This provides theoretical motivation for designing progressive and adaptive quantization schemes. Full proofs and visuals are in Appendix~\ref{app:detailed_theorems}.

\begin{theorem}
\label{thm:WMSE-SR-main}
Let $\boldsymbol{\mathsf{W}}$ be a zero-mean Gaussian random matrix with i.i.d. (independent and identically distributed) entries, and consider the weighted mean squared error distortion with a fixed matrix $\boldsymbol{X}$, $d_{\mathrm{WMSE}}(\boldsymbol{\mathsf{W}}, \widehat{\boldsymbol{\mathsf{W}}})
=
\mathbb{E}\!\left[
\|(\boldsymbol{\mathsf{W}} - \widehat{\boldsymbol{\mathsf{W}}})\boldsymbol{X}\|_F^2
\right]$. Then, the source $(\boldsymbol{\mathsf{W}}, d_{\mathrm{WMSE}})$ is successively refinable.
\end{theorem}

The central challenge is that the WMSE distortion couples all entries of
$\boldsymbol{\mathsf{W}}$ through the fixed data matrix $\boldsymbol{X}$,
making it non-obvious that successive refinement is achievable without a
rate penalty. The proof resolves this by reducing to a standard Gaussian MSE problem, then revealing the nested structure of reverse water-filling. The reverse water-filling method is an optimization technique used in information theory and communication systems which allocates more bits (rate) to \textit{high-variance} channels to minimize total distortion under a rate constraint. 

\begin{proof}[Proof Sketch.]
Introduce the transformed source $\boldsymbol{\mathsf{Y}} =
\boldsymbol{\mathsf{W}}\boldsymbol{X}$. The rate-distortion function for $\boldsymbol{\mathsf{W}}$ under
$d_{\mathrm{WMSE}}$ equals the standard Gaussian MSE rate-distortion
function for $\boldsymbol{\mathsf{Y}}$, with a one-to-one correspondence
between their optimal codebooks at every distortion level $D$. This follows
from the distortion equivalence: the optimal $\widehat{\boldsymbol{\mathsf{W}}}$
is recovered by inverting $\boldsymbol{\mathsf{Y}} =
\boldsymbol{\mathsf{W}}\boldsymbol{X}$, and satisfies
    $d_{\mathrm{WMSE}}(\boldsymbol{\mathsf{W}}, \widehat{\boldsymbol{\mathsf{W}}})
    \;=\;
    \mathbb{E}\!\left[
    \bigl\|(\boldsymbol{\mathsf{W}} - \widehat{\boldsymbol{\mathsf{W}}})
    \boldsymbol{X}\bigr\|_F^2
    \right]
    \;=\;
    \mathbb{E}\!\left[
    \bigl\|\boldsymbol{\mathsf{Y}} - \widehat{\boldsymbol{\mathsf{Y}}}
    \bigr\|_F^2
    \right],$ by Theorem~\ref{thm:Sakrison}.
    
It therefore suffices to solve the standard MSE rate-distortion problem
for $\boldsymbol{\mathsf{Y}}$. To do so, we vectorize
$\boldsymbol{\mathsf{y}} \coloneqq \operatorname{vec}(\boldsymbol{\mathsf{Y}})$
and compute that $\boldsymbol{\mathsf{y}}$ is a
zero-mean coloured Gaussian with covariance
    $\boldsymbol{\Sigma}_{\mathsf{y}}
    \;=\;
    \bigl(\sigma_w^2\,\boldsymbol{X}^\top\boldsymbol{X}\bigr)
    \otimes \boldsymbol{I}_{d_{\mathrm{out}}},$
where $\otimes$ denotes the Kronecker product. Since $\boldsymbol{\mathsf{y}}$
is correlated, reverse water-filling cannot be applied directly.
Diagonalizing via eigen-decomposition $\boldsymbol{\Sigma}_{\mathsf{y}} =
\boldsymbol{U}\,\mathrm{diag}(\lambda_1,\lambda_2,\ldots)\,\boldsymbol{U}^\top$
yields independent scalar Gaussian components $\mathsf{v}_i \sim
\mathcal{N}(0,\lambda_i)$, to which reverse water-filling applies. The
optimal rate-distortion tradeoff for $\boldsymbol{\mathsf{y}}$ under
MSE is therefore achieved by assigning per-component
distortion $D_i = \min(\mu, \lambda_i)$, where the "water level" $\mu$ is
chosen so that $\sum_i D_i = D$ (Theorem~\ref{thm:Sakrison}). The water level $\mu$ vertically partitions the variance $\lambda_i$ of each component into two distinct segments: the portion below the water level represents the remaining distortion, while the portion extending above the water level determines the allocated rate.

To establish successive refinability, we use superscripts $k$ to index
refinement stages and subscripts $i$ to index scalar components. For nested
overall distortion levels $D^{(1)} > \cdots > D^{(K)}$, reverse water-filling
assigns per-component distortions $D_i^{(k)} = \min(\mu_k, \lambda_i)$, corresponding to a
strictly decreasing sequence of water levels $\mu_1 > \cdots > \mu_K$. For
each component $\mathsf{v}_i$, define mutually independent Gaussian increments
\[\mathsf{z}_i^{(k)} \;\sim\; \mathcal{N}\!\left(0,\,
    D_i^{(k-1)} - D_i^{(k)}\right), \qquad  D_i^{(0)} \coloneqq \lambda_i,\]  and construct successive reconstructions via
$\widehat{\mathsf{v}}_i^{(0)} = 0,$ $\widehat{\mathsf{v}}_i^{(k)} = \widehat{\mathsf{v}}_i^{(k-1)} +
    \mathsf{z}_i^{(k)}.$
Each increment $\mathsf{z}_i^{(k)}$ carries exactly the new information
needed to improve the reconstruction of $\mathsf{v}_i$ from distortion
$D_i^{(k-1)}$ to the finer distortion $D_i^{(k)}$; it is independent of
all previous increments and of the current reconstruction
$\widehat{\mathsf{v}}_i^{(k-1)}$. The incremental rate $\Delta R_i^{(k)}$
is therefore the number of bits required to describe $\mathsf{z}_i^{(k)}$,
i.e.\ the mutual information between $\mathsf{v}_i$ and the refined
reconstruction $\widehat{\mathsf{v}}_i^{(k)}$ given the previous one:
    $\Delta R_i^{(k)}
    \;=\; I\!\left(\mathsf{v}_i;\,\widehat{\mathsf{v}}_i^{(k)}
    \,\middle|\, \widehat{\mathsf{v}}_i^{(k-1)}\right)
    \;=\; \frac{1}{2}\log\frac{D_i^{(k-1)}}{D_i^{(k)}}.$
By construction, $\mathbb{E}[(\mathsf{v}_i -
\widehat{\mathsf{v}}_i^{(k)})^2] = D_i^{(k)}$ at every stage, and since
each increment is independent of all previous ones, the Markov chain
$\widehat{\mathsf{v}}_i^{(1)} \to \cdots \to \widehat{\mathsf{v}}_i^{(K)}
\to \mathsf{v}_i$ holds (Theorem~\ref{thm:WMSE-SR}).
The below sum then confirms there is no rate penalty:
\[
    \sum_{k=1}^{K} \Delta R_i^{(k)}
    \;=\; \frac{1}{2}\log\frac{\lambda_i}{D_i^{(1)}}
    + \frac{1}{2}\log\frac{D_i^{(1)}}{D_i^{(2)}}
    + \cdots
    + \frac{1}{2}\log\frac{D_i^{(K-1)}}{D_i^{(K)}}
    \;=\; \frac{1}{2}\log\frac{\lambda_i}{D_i^{(K)}},
\]
which equals exactly the single-shot rate at $D_i^{(K)}$. As illustrated
in Figure~\ref{fig:sr_equals_single_shot}, the $\Delta R^{(k)}$ brackets
partition the total rate between refinement stages identically
to single-shot coding at $D_K$, confirming that refinement incurs no rate
overhead. Finally, at each stage $k$, the reconstruction
$\widehat{\boldsymbol{\mathsf{W}}}^{(k)}$ is recovered from
$\widehat{\boldsymbol{\mathsf{Y}}}^{(k)}$
by inverting $\boldsymbol{\mathsf{Y}} = \boldsymbol{\mathsf{W}}\boldsymbol{X}$ (depending on the
shape of $\boldsymbol{X}$, see Theorem~\ref{thm:Sakrison}). Since this inversion map is identical at every
stage $k$, the Markov chain on $\boldsymbol{\mathsf{y}}$ carries over to $\boldsymbol{\mathsf{W}}$,
$\widehat{\boldsymbol{\mathsf{W}}}^{(1)} \;\to\; \cdots \;\to\;
    \widehat{\boldsymbol{\mathsf{W}}}^{(K)} \;\to\; \boldsymbol{\mathsf{W}},$ and the distortion equivalence gives
$d_{\mathrm{WMSE}}(\boldsymbol{\mathsf{W}},
\widehat{\boldsymbol{\mathsf{W}}}^{(k)}) = D_k$ at each stage,
establishing successive refinability of
$(\boldsymbol{\mathsf{W}}, d_{\mathrm{WMSE}})$.
\end{proof}

\section{Our Application: Drop-by-Drop (DbyD)}
\label{section:application}
Motivated by our information-theoretic result, we now focus on a \textit{practical} quantization scheme that supports multiple rate–distortion operating points without retraining, under the same distortion measure $(\boldsymbol{\mathsf{W}}, d_{\mathrm{WMSE}})$ used during quantization. AQLM naturally instantiates this structure: each additive codebook can correspond to a refinement layer. However, it is our Drop-by-Drop loss that ensures this \textit{droppable} property holds in practice, where codebooks can be dropped at inference to provide a smooth tradeoff between compression rate and reconstruction quality.

\begin{wrapfigure}[8]{r}{0.38\linewidth}
\vspace{-0.4cm}
\centering
\definecolor{dbyd}{RGB}{190,115,60}
\resizebox{\linewidth}{!}{%
\begin{tikzpicture}[
  lvl/.style={draw=dbyd!70!black, thick, rounded corners=3pt},
  lbl/.style={font=\scriptsize, inner sep=1pt}
]
\draw[lvl, fill=dbyd!8]  (-0.45,-0.95) rectangle (4.80,0.95);
\draw[lvl, fill=dbyd!18] (-0.30,-0.80) rectangle (3.60,0.80);
\draw[lvl, fill=dbyd!30] (-0.15,-0.65) rectangle (2.40,0.65);
\draw[lvl, fill=dbyd!45] (0,-0.50)     rectangle (1.20,0.50);
\node[lbl] at (0.6,  0.2) {$C_1$};
\node[lbl] at (1.8,  0.2) {$+\,C_2$};
\node[lbl] at (3.0,  0.2) {$\cdots$};
\node[lbl] at (4.2,  0.2) {$+\,C_M$};
\node[lbl] at (0.6, -0.2) {$\alpha_1 L^{(1)}$};
\node[lbl] at (1.8, -0.2) {$\alpha_2 L^{(2)}$};
\node[lbl] at (3.0, -0.2) {$\cdots$};
\node[lbl] at (4.2, -0.2) {$\alpha_M L^{(M)}$};
\end{tikzpicture}}
\caption{Visualization of $\mathcal{L}_{\mathrm{DbyD}}$.}
\label{fig:doll}
\end{wrapfigure}

We introduce the \textbf{Drop-by-Drop loss}, a training objective inspired by
Matryoshka Representation Learning that modifies the codebook update phase of
the AQLM pipeline. It enforces a
hierarchical organization in which early codebooks capture coarse but
fundamental weight features and later ones conduct progressively finer
refinements. To induce this ordering, we supervise \textit{partial
reconstructions}: for each stage $k = 1, \dotsc, M$, we reconstruct the weights
from only the first $k$ codebooks, so that $C_1$ contributes at every level and
the representation refines incrementally. Figure~\ref{fig:doll} depicts these
stages as nested boxes, where nesting encodes prefix containment and darker
shading indicates higher distortion; sizes and shades are illustrative. Here,
$C_m \in \mathbb{R}^{g \times 2^B}$ is the $m$-th codebook and
$b_{ijm} \in \{0,1\}^{2^B}$ is a one-hot vector indicating the selected
codeword:
\vspace{-2mm}
\[\widehat{\boldsymbol{\mathsf{W}}}_i^{(k)} = \sum_{m=1}^k C_m b_{i,1,m} \oplus \cdots \oplus \sum_{m=1}^k C_m b_{i,d_{\mathrm{in}}/g,m}\] 

\vspace{-4mm}

This yields an intermediate approximation 
$\widehat{\boldsymbol{\mathsf{W}}}^{(k)}$, exposing the model during training to all 
deployment-time precision levels. For each partial reconstruction, we measure distortion using a task-weighted MSE:

\vspace{-5mm}

$$
L^{(k)} 
= 
\frac{1}{d_{\text{out}}\, n} 
\left\|
(\boldsymbol{\mathsf{W}} - \boldsymbol{\widehat{\mathsf{W}}}^{(k)})\,\boldsymbol{X}
\right\|_F^2
=
\frac{1}{d_{\text{out}}}
\operatorname{tr}\!\left[
(\boldsymbol{\mathsf{W}} - \boldsymbol{\widehat{\mathsf{W}}}^{(k)})
\frac{1}{n}\boldsymbol{X}\boldsymbol{X}^{\!\top}
(\boldsymbol{\mathsf{W}} - \boldsymbol{\widehat{\mathsf{W}}}^{(k)})^{\!\top}
\right].
$$

\vspace{-2mm}

The overall Drop-by-Drop objective aggregates these distortions across refinement stages: 
$\mathcal{L}_{\mathrm{DbyD}} \;=\; \sum_{k=1}^{M} \alpha_k \, L^{(k)},$
where $\alpha_k \ge 0$ are weighting coefficients, normalized such that
$\sum_{k=1}^{M} \alpha_k = 1$, controlling the relative importance of coarse
versus fine reconstructions. This training logic is summarized in Algorithm~\ref{alg:drop_by_drop1}, and the full algorithm can be found in Appendix~\ref{alg:full_alg1} where coloured text shows the changes made to AQLM. 

Furthermore, we include a block-wise Walsh-Hadamard rotation as an optional
pre-processing step. It reduces the excess
kurtosis of the weight distribution, bringing it closer to the Gaussian source
assumed by Theorem~\ref{thm:WMSE-SR-main}. The analysis is in Appendix~\ref{app:gaussianplots}.

At inference time, Drop-by-Drop lets a single deployed model be reconstructed at multiple precisions without retraining. Given a target budget, we dequantize the weights from a prefix of the first $k$ codebooks ($k \leq M$), where $k$ is specified when the model is loaded. 

\vspace{-0.3cm}
\begin{algorithm}[h] %
\caption{Drop-by-Drop Quantization for Resource-Constrained Inference}
\label{alg:drop_by_drop1}
\begin{algorithmic}[1]
\Require model, data
\State $X_{\text{block}} :=$ model.input\_embeddings(data)

    \For{layer $\in$ linear\_layers(block)}
        \State $W :=$ layer.weight, \quad $d_{\text{in}} := \text{columns}(W)$
        \State $XX^\top := \text{accumulate\_hessian(layer, } X_{\text{block}})$
        \State \textcolor{teal}{Select block size $B_H$ and sample random signs $S \sim \mathcal{U}(\{-1,1\})^{d_{\text{in}}}$}
        \State \textcolor{teal}{$W \gets \text{BlockFWHT}_{1\text{D}}(W \odot S, B_H)$
         \State $XX^\top \gets \text{BlockFWHT}_{2\text{D}}(XX^\top \odot (SS^\top), B_H)$}
        \State {$C, b, s :=$ k\_means\_initialize($W$)} 
        \While{loss improves by at least $\tau$}
        \For{\textcolor{blue}{$k = 1, \dotsc, M$}}
            \State \textcolor{blue}{$\widehat{W}^{(k)} \gets$ partial reconstruction using the first $k$ codebooks}
            \State \textcolor{blue}{${L}^{(k)} \gets 
              \frac{1}{d_{\text{out}}}
            \operatorname{tr}\!\left[
                (W - \widehat{W}^{(k)})
                \left(\frac{1}{n} X X^{\top}\right)
                (W - \widehat{W}^{(k)} )^{\top}
            \right]$}
        \EndFor
            \State \textcolor{blue}{Total loss: $\mathcal{L}_\mathrm{DbyD} = \sum_k  \alpha_k{L}^{(k)}$}
            \State {Update codebooks $C$ by gradient descent, update assignments $b$ by beam search}
        \EndWhile
        \State layer.weight := AQLMFormat($C, b, s$)
        \State \textcolor{teal}{layer.rotation\_signs := $S$, \quad layer.block\_size := $B_H$}
    \EndFor
\end{algorithmic}
\end{algorithm}

\section{Experiments}
\label{section:experiments}
\paragraph{Experimental Setup \& Hyperparameters.} Our experiments were conducted on multi-GPU systems equipped with NVIDIA L40S GPUs. We mostly used 4 GPUs with CUDA support. We perform quantization using $M \in \{3, 4, 5\}$ codebooks, each with a bitwidth of 8 and a group size of 8. We used 1024 calibration samples with a validation split of 32 samples, optimizing a relative MSE objective with a tolerance of $0.01$. All quantized models are evaluated against the original FP16 base model using a batch size of 8.
\paragraph{Method Definitions.} We follow AQLM and partition the weights into groups of 8, applying additive quantization with $K$ codebooks of 8 bits each per group. Selecting $K$ codebooks per group therefore yields an effective rate of $K$ bits per weight. Hence, the $3\times8$, $4\times8$, and $5\times8$ settings, where the first number denotes the number of codebooks and the second the bit-width of each, correspond to 3-, 4-, and 5-bit configurations, respectively. For clarity, we define the compared methods:

    1. \textit{AQLM (Individual)}: Separate AQLM models trained independently per target configuration. This is the ideal performance bound, but requires training and storing $N$ models for $N$ levels.

     2. \textit{AQLM (Drop)}: An AQLM model trained at maximum quantization (e.g., $5\times8$), with codebooks dropped at inference time to simulate lower quantization levels.

     3. \textit{Drop-by-Drop (Uniform)}: A single model trained once at maximum quantization using the proposed Drop-by-Drop method, designed to support codebook reduction at inference time. Uniform Matryoshka loss weights are used, \textbf{Uni}: $\alpha_1 = \alpha_2 = \alpha_3 = \alpha_4 = \alpha_5 = \frac{1}{5}.$

    4. \textit{Drop-by-Drop (Weighted)}: Variants of DbyD using non-uniform Matryoshka loss weights to better approximate the performance of individually tuned AQLM models. Based on our ablation studies in Appendix~\ref{app:ablation-study}, we report \textbf{35W}: $\alpha_3 =\alpha_5 = \frac{1}{2}$, all other $\alpha_k = 0$.
    
     5. \textit{Drop-by-Drop (Hadamard + Weighted)}: The DbyD (35W) variant with a block-wise Walsh-Hadamard rotation applied to the weights prior to quantization.

\begin{figure}[h]
    \centering
    \hspace{0.85cm} %
    \begin{tikzpicture}
        \begin{groupplot}[
            group style={
                group size=4 by 2,
                horizontal sep=0.8cm, %
                vertical sep=1cm,
                xlabels at=edge bottom,
                ylabels at=edge left
            },
            width=3.5cm, %
            height=3.5cm,
            xlabel={Codebooks},
            ylabel={Perplexity},
            symbolic x coords={3x8, 4x8, 5x8},
            xtick=data,
            ymajorgrids=true,
            grid style=dashed,
            tick label style={font=\tiny}, %
            label style={font=\footnotesize},
            title style={font=\bfseries\scriptsize, yshift=-0.5ex}, %
            every axis plot/.append style={
                mark=*,
                mark size=1.5pt, %
                thick
            },
            cycle list={
                {myblue, mark=square*},
                {myred, mark=square},
                {myorange, mark=o},
                {mypurple, mark=*},
                {mygreen, mark=diamond*, mark options={scale=1.4}}
            },
            legend columns=-1, %
            legend style={
                draw=none, 
                /tikz/every even column/.append style={column sep=0.15cm}, %
                font=\scriptsize %
            },
            legend to name=CommonLegendC4
        ]

        \nextgroupplot[title=Gemma 2B, ymin=11, ymax=19]
        \addplot coordinates {(3x8, 15.13) (4x8, 12.91) (5x8, 12.50)};
        \addplot coordinates {(3x8, 18.37) (4x8, 13.28) (5x8, 12.50)};
        \addplot coordinates {(3x8, 18.26) (4x8, 14.02) (5x8, 12.47)};
        \addplot coordinates {(3x8, 15.93) (4x8, 13.59) (5x8, 12.06)};
        \addplot coordinates {(3x8, 17.07) (4x8, 13.23) (5x8, 11.79)};

        \nextgroupplot[title=Mistral 7B, ymin=7, ymax=18]
        \addplot coordinates {(3x8, 8.23) (4x8, 7.73) (5x8, 7.93)};
        \addplot coordinates {(3x8, 20.00) (4x8, 8.74) (5x8, 7.93)}; %
        \addplot coordinates {(3x8, 16.09) (4x8, 8.31) (5x8, 7.70)};
        \addplot coordinates {(3x8, 9.14) (4x8, 8.08) (5x8, 7.58)};
        \addplot coordinates {(3x8, 8.32) (4x8, 7.71) (5x8, 7.54)};

        \nextgroupplot[title=MetaLlama 8B, ymin=8, ymax=30]
        \addplot coordinates {(3x8, 10.73) (4x8, 9.29) (5x8, 8.71)};
        \addplot coordinates {(3x8, 50.00) (4x8, 28.22) (5x8, 8.71)}; %
        \addplot coordinates {(3x8, 20.61) (4x8, 10.66) (5x8, 9.01)};
        \addplot coordinates {(3x8, 12.97) (4x8, 10.39) (5x8, 8.78)};
        \addplot coordinates {(3x8, 11.53) (4x8, 9.38) (5x8, 8.63)};

        \nextgroupplot[title=Qwen 0.5B, ymin=16, ymax=60]
        \addplot coordinates {(3x8, 21.87) (4x8, 18.73) (5x8, 17.87)};
        \addplot coordinates {(3x8, 58.76) (4x8, 20.91) (5x8, 17.87)};
        \addplot coordinates {(3x8, 44.37) (4x8, 21.36) (5x8, 18.30)};
        \addplot coordinates {(3x8, 37.58) (4x8, 21.48) (5x8, 18.00)};
        \addplot coordinates {(3x8, 27.67) (4x8, 19.99) (5x8, 17.87)};

        \nextgroupplot[title=Qwen 3B, ymin=10, ymax=25]
        \addplot coordinates {(3x8, 12.94) (4x8, 11.72) (5x8, 11.35)};
        \addplot coordinates {(3x8, 40.00) (4x8, 22.32) (5x8, 11.35)}; %
        \addplot coordinates {(3x8, 14.94) (4x8, 12.44) (5x8, 11.50)};
        \addplot coordinates {(3x8, 14.32) (4x8, 12.24) (5x8, 11.39)};
        \addplot coordinates {(3x8, 13.96) (4x8, 11.93) (5x8, 11.32)};

        \nextgroupplot[title=Qwen 7B, ymin=10, ymax=16]
        \addplot coordinates {(3x8, 11.85) (4x8, 10.94) (5x8, 10.61)};
        \addplot coordinates {(3x8, 15.10) (4x8, 11.79) (5x8, 10.61)};
        \addplot coordinates {(3x8, 13.91) (4x8, 11.66) (5x8, 10.80)};
        \addplot coordinates {(3x8, 13.24) (4x8, 11.39) (5x8, 10.63)};
        \addplot coordinates {(3x8, 12.27) (4x8, 10.93) (5x8, 10.51)};

        \nextgroupplot[title=Qwen 14B, ymin=8.5, ymax=13]
        \addplot coordinates {(3x8, 10.01) (4x8, 9.33) (5x8, 9.10)};
        \addplot coordinates {(3x8, 12.02) (4x8, 9.77) (5x8, 9.10)};
        \addplot coordinates {(3x8, 11.44) (4x8, 9.78) (5x8, 9.22)};
        \addplot coordinates {(3x8, 10.82) (4x8, 9.71) (5x8, 9.26)};
        \addplot coordinates {(3x8, 10.38) (4x8, 9.40) (5x8, 9.06)};

        \nextgroupplot[title=Qwen 32B, ymin=8.5, ymax=11]
        \addplot+[color=myblue] coordinates {(3x8, 9.46) (4x8, 8.98) (5x8, 8.84)};
        \addlegendentry{AQLM (Ind.)}
        
        \addplot+[color=myred] coordinates {(3x8, 10.54) (4x8, 9.34) (5x8, 8.84)};
        \addlegendentry{AQLM (Drop)}
        
        \addplot+[color=myorange] coordinates {(3x8, 10.10) (4x8, 9.25) (5x8, 8.89)};
        \addlegendentry{DbyD (Uni.)}
        
        \addplot+[color=mypurple] coordinates {(3x8, 9.77) (4x8, 9.17) (5x8, 8.84)};
        \addlegendentry{DbyD (35W)}

        \addplot+[color=mygreen, mark=diamond*, mark options={scale=1.4}] coordinates {(3x8, 9.49) (4x8, 8.98) (5x8, 8.79)}; 
        \addlegendentry{DbyD (Had.+35W)}

        \end{groupplot}

        \node at ($(group c1r2.south west)!0.5!(group c4r2.south east)$) [yshift=-1.2cm] {\ref{CommonLegendC4}};
        
    \end{tikzpicture}
    \caption{Perplexity on \textsc{C4} ($\downarrow$). Extreme outliers for AQLM (Drop) are clipped.}
    \label{fig:perplexity_c4}
\end{figure}

\paragraph{Evaluation Benchmarks.}
We evaluate Drop-by-Drop across a diverse set of LLMs spanning multiple architectures and scales, including \textsc{Gemma-2B}, \textsc{Meta-LLaMA-3-8B}, \textsc{Mistral-7B-v0.3}, and the \textsc{Qwen2.5} family (\textsc{0.5B-Instruct}, \textsc{3B}, \textsc{7B-Instruct}, \textsc{14B-Instruct}, and \textsc{32B-Instruct}). Intrinsic language modeling performance is measured using perplexity on \textsc{C4} and \textsc{WikiText-2}~\citep{c4, wikitext2}. Extrinsic evaluation is conducted via zero-shot accuracy on commonsense and reasoning benchmarks, including \textsc{ARC} (Easy and Challenge), \textsc{HellaSwag}, \textsc{Winogrande}, and \textsc{PIQA}~\citep{arc, hellaswag, winogrande, piqa}, reported as the mean across these five tasks. Perplexity trends agree across \textsc{WikiText-2} and \textsc{C4}, so the main text reports \textsc{C4} perplexity and mean zero-shot accuracy. \textsc{WikiText-2} perplexity, per-task accuracy, and the numerical values behind every figure are in Appendix~\ref{app:detailed}. All our results are averaged over \textit{three} independent runs.

\begin{figure}[h]
    \centering
    \vspace{-3mm}
    \hspace{0.2cm}
    \begin{tikzpicture}
        \begin{groupplot}[
            group style={
                group size=4 by 2, %
                horizontal sep=0.8cm, %
                vertical sep=1cm,
                xlabels at=edge bottom,
                ylabels at=edge left
            },
            width=3.5cm, %
            height=3.5cm,
            xlabel={Codebooks},
            ylabel={Accuracy (\%)},
            symbolic x coords={3x8, 4x8, 5x8},
            xtick=data,
            ymajorgrids=true,
            grid style=dashed,
            tick label style={font=\tiny},
            label style={font=\footnotesize},
            title style={font=\bfseries\scriptsize, yshift=-0.5ex},
            every axis plot/.append style={
                mark=*,
                mark size=1.5pt,
                thick
            },
            cycle list={
                {myblue, mark=square*},
                {myred, mark=square},
                {myorange, mark=o},
                {mypurple, mark=*},
                {mygreen, mark=diamond*, mark options={scale=1.4}}
            }
        ]

        \nextgroupplot[title=Gemma 2B, ymin=55, ymax=67]
        \addplot coordinates {(3x8, 60.37) (4x8, 63.95) (5x8, 65.59)};
        \addplot coordinates {(3x8, 61.22) (4x8, 64.42) (5x8, 65.59)};
        \addplot coordinates {(3x8, 56.51) (4x8, 62.38) (5x8, 64.62)};
        \addplot coordinates {(3x8, 60.61) (4x8, 63.90) (5x8, 65.48)};
        \addplot coordinates {(3x8, 60.17) (4x8, 63.48) (5x8, 65.13)};

        \nextgroupplot[title=Mistral 7B, ymin=40, ymax=75]
        \addplot coordinates {(3x8, 70.97) (4x8, 72.46) (5x8, 71.72)};
        \addplot coordinates {(3x8, 44.84) (4x8, 69.94) (5x8, 71.72)};
        \addplot coordinates {(3x8, 61.78) (4x8, 69.96) (5x8, 72.15)};
        \addplot coordinates {(3x8, 67.48) (4x8, 70.81) (5x8, 72.85)};
        \addplot coordinates {(3x8, 70.20) (4x8, 72.32) (5x8, 73.08)};

        \nextgroupplot[title=MetaLlama 8B, ymin=35, ymax=75]
        \addplot coordinates {(3x8, 70.00) (4x8, 71.00) (5x8, 72.42)};
        \addplot coordinates {(3x8, 38.08) (4x8, 65.54) (5x8, 72.42)};
        \addplot coordinates {(3x8, 55.80) (4x8, 68.52) (5x8, 71.62)};
        \addplot coordinates {(3x8, 64.81) (4x8, 69.60) (5x8, 72.09)};
        \addplot coordinates {(3x8, 67.81) (4x8, 72.36) (5x8, 73.37)};

        \nextgroupplot[title=Qwen 0.5B, ymin=40, ymax=56]
        \addplot coordinates {(3x8, 49.96) (4x8, 53.60) (5x8, 53.25)};
        \addplot coordinates {(3x8, 44.02) (4x8, 52.80) (5x8, 53.25)};
        \addplot coordinates {(3x8, 45.88) (4x8, 51.86) (5x8, 53.62)};
        \addplot coordinates {(3x8, 48.47) (4x8, 53.14) (5x8, 53.86)};
        \addplot coordinates {(3x8, 48.36) (4x8, 52.40) (5x8, 54.17)};

        \nextgroupplot[title=Qwen 3B, ymin=35, ymax=70]
        \addplot coordinates {(3x8, 65.48) (4x8, 66.68) (5x8, 67.83)};
        \addplot coordinates {(3x8, 39.64) (4x8, 62.11) (5x8, 67.83)};
        \addplot coordinates {(3x8, 62.53) (4x8, 65.85) (5x8, 68.18)};
        \addplot coordinates {(3x8, 63.06) (4x8, 67.28) (5x8, 68.17)};
        \addplot coordinates {(3x8, 64.32) (4x8, 67.80) (5x8, 68.11)};

        \nextgroupplot[title=Qwen 7B, ymin=65, ymax=75]
        \addplot coordinates {(3x8, 70.83) (4x8, 72.37) (5x8, 73.59)};
        \addplot coordinates {(3x8, 66.77) (4x8, 71.40) (5x8, 73.59)};
        \addplot coordinates {(3x8, 67.30) (4x8, 71.32) (5x8, 72.75)};
        \addplot coordinates {(3x8, 69.54) (4x8, 72.52) (5x8, 73.72)};
        \addplot coordinates {(3x8, 68.66) (4x8, 71.37) (5x8, 73.28)};

        \nextgroupplot[title=Qwen 14B, ymin=68, ymax=78]
        \addplot coordinates {(3x8, 74.89) (4x8, 75.94) (5x8, 76.85)};
        \addplot coordinates {(3x8, 70.00) (4x8, 74.73) (5x8, 76.85)};
        \addplot coordinates {(3x8, 73.92) (4x8, 75.80) (5x8, 76.03)};
        \addplot coordinates {(3x8, 73.93) (4x8, 75.99) (5x8, 76.76)};
        \addplot coordinates {(3x8, 75.07) (4x8, 76.63) (5x8, 77.20)};

        \nextgroupplot[title=Qwen 32B, ymin=70, ymax=76]
        \addplot coordinates {(3x8, 74.94) (4x8, 74.50) (5x8, 74.53)};
        \addplot coordinates {(3x8, 71.08) (4x8, 73.10) (5x8, 74.53)};
        \addplot coordinates {(3x8, 71.93) (4x8, 72.56) (5x8, 74.20)};
        \addplot coordinates {(3x8, 72.94) (4x8, 73.98) (5x8, 74.74)};
        \addplot coordinates {(3x8, 73.88) (4x8, 75.22) (5x8, 74.98)};

        \end{groupplot}

    \end{tikzpicture}
    \caption{Average zero-shot accuracy ($\uparrow$) across 5 aforementioned tasks}
    \label{fig:accuracy_zeroshot}
\end{figure}

\paragraph{DbyD vs.\ AQLM comparison.}
Figure~\ref{fig:perplexity_c4} and Figure~\ref{fig:accuracy_zeroshot} compare performance across quantization levels. \textbf{AQLM (Ind.)} serves as the specialized performance bound, while \textbf{AQLM (Drop)} provides a baseline illustrating the behavior of standard AQLM when codebooks are removed without hierarchical training. As codebooks are dropped, AQLM (Drop) diverges catastrophically from the specialized bound, particularly at the $3\times8$ level, confirming that standard AQLM does not provide a multi-bitwidth setting. In contrast, the Drop-by-Drop variants consistently outperform this baseline. The
\textbf{DbyD (35W)} configuration substantially narrows the gap to the individually
trained AQLM (Ind.) models, showing that targeted Matryoshka loss weighting lets a single
flexible model serve multiple bitwidths with only a modest loss relative to training each
bitwidth separately. Applying a block-Hadamard rotation before the WMSE-based supervision,
\textbf{DbyD (Had.+35W)}, narrows this gap further. In perplexity, DbyD (Had.+35W) is the
strongest DbyD variant on nearly all models and bitwidths, approaching AQLM (Ind.) and
substantially outperforming AQLM (Drop); the exception is \textsc{Gemma-2B}, where DbyD (35W)
is better at $3\times8$ on both \textsc{WikiText2} and \textsc{C4}. Zero-shot
accuracy follows a similar trend: at $4\times8$ and $5\times8$, DbyD (Had.+35W) closely
matches AQLM (Ind.) within standard error.

\begin{figure}[h]
\centering
\begin{minipage}[b]{0.49\textwidth}
    \centering
    \begin{tikzpicture}
        \begin{axis}[
            ybar=0.5pt,
            width=\linewidth,
            height=4.6cm,
            bar width=3.2pt,
            enlarge x limits=0.10,
            symbolic x coords={Q-0.5, G-2, Q-3, M-7, Q-7, L-8, Q-14, Q-32},
            xtick=data,
            xticklabel style={
                font=\scriptsize,
                rotate=30,
                anchor=north east,
                inner sep=1pt
            },
            yticklabel style={font=\small},
            ymin=0, ymax=38,
            ytick={0,10,20,30},
            ymajorgrids=true,
            grid style={dashed, gray!30},
            axis line style={draw=none},
            tick style={draw=none},
            legend style={
                at={(0.01,0.99)},
                anchor=north west,
                draw=none,
                fill=white,
                fill opacity=0.9,
                font=\scriptsize,
                cells={anchor=west}
            },
            legend image code/.code={
                \draw[#1, draw=none] (0cm,-0.05cm) rectangle (0.22cm,0.16cm);
            },
        ]

        \addplot[fill=green!60!black, draw=none] coordinates {
            (Q-0.5, 0.50) (G-2, 2.29) (Q-3, 2.36)
            (M-7, 4.91)   (Q-7, 6.27) (L-8, 6.47)
            (Q-14, 11.39) (Q-32, 22.64)
        };
        \addlegendentry{\textbf{DbyD (Ours)}}

        \addplot[fill=cAnyPrec, draw=none] coordinates {
            (Q-0.5, 0.52) (G-2, 2.24) (Q-3, 2.31)
            (M-7, 4.71)   (Q-7, 5.99) (L-8, 6.18)
            (Q-14, 10.83) (Q-32, 21.56)
        };
        \addlegendentry{AnyPrecLLM}

        \addplot[fill=cAnyBCQ, draw=none] coordinates {
            (Q-0.5, 0.58) (G-2, 2.69) (Q-3, 2.97)
            (M-7, 6.50)   (Q-7, 7.64) (L-8, 7.97)
            (Q-14, 14.25) (Q-32, 29.70)
        };
        \addlegendentry{AnyBCQ}

        \addplot[fill=cMatGPTQ, draw=none] coordinates {
            (Q-0.5, 0.62) (G-2, 2.89) (Q-3, 3.22)
            (M-7, 7.11)   (Q-7, 0)    (L-8, 8.58)
            (Q-14, 15.42) (Q-32, 32.44)
        };
        \addlegendentry{MatGPTQ}

        \node[
            font=\bfseries\small\color{cMatGPTQ!85!black},
            anchor=south
        ] at ([xshift=5.6pt]axis cs:Q-7, 0.2) {\textsf{x}};

        \end{axis}
    \end{tikzpicture}
    \caption{SOTA disk usage (GB) comparison (\(\downarrow\)).}
    \label{fig:disk_usage_comparison}
\end{minipage}
\hfill
\begin{minipage}[b]{0.49\textwidth}
    \centering
    \begin{tikzpicture}
    \begin{axis}[
        width=\linewidth,
        height=4.0cm,
        ymode=log,
        ymin=0.97, ymax=3.1,
        symbolic x coords={Q-0.5, G-2, Q-3, M-7, Q-7, L-8, Q-14, Q-32},
        xtick=data,
        enlarge x limits=0.10,
        xticklabel style={
            font=\scriptsize,
            rotate=30,
            anchor=north east,
            inner sep=1pt
        },
        ytick={1,1.5,2,3}, yticklabels={\(1\times\),\(1.5\times\),\(2\times\),\(3\times\)},
        minor tick num=0,
        yticklabel style={font=\small},
        ymajorgrids=true,
        grid style={dashed, gray!30},
        axis line style={draw=none},
        tick style={draw=none},
        legend style={
            at={(0.5,1.03)},
            anchor=south,
            draw=none,
            fill=none,
            font=\scriptsize,
            cells={anchor=west},
            column sep=5pt
        },
        legend columns=2,
        every axis plot/.append style={thick, mark size=1.6pt},
    ]
    \addplot[gray!80!black, mark=*] coordinates {
        (Q-0.5, 1.50) (G-2, 2.31) (Q-3, 2.12) (M-7, 2.70)
        (Q-7, 2.62)   (L-8, 2.70) (Q-14, 2.77) (Q-32, 2.78)
    };
    \addlegendentry{cuBLAS (FP16)}
    \addplot[cAnyPrecfig10, mark=triangle*] coordinates {
        (Q-0.5, 1.11) (G-2, 1.15) (Q-3, 1.17) (M-7, 1.14)
        (Q-7, 1.20)   (L-8, 1.14) (Q-14, 1.18) (Q-32, 1.15)
    };
    \addlegendentry{AnyPrecLLM}
    \addplot[cAnyBCQfig10, mark=square*] coordinates {
        (Q-0.5, 1.11) (G-2, 1.15) (Q-3, 1.13) (M-7, 1.17)
        (Q-7, 1.16)   (L-8, 1.16) (Q-14, 1.17) (Q-32, 1.15)
    };
    \addlegendentry{AnyBCQ}
    \addplot[cMatGPTQfig10, mark=diamond*] coordinates {
        (Q-0.5, 1.17) (G-2, 1.12) (Q-3, 1.14) (M-7, 1.06)
        (Q-7, 1.06)   (L-8, 1.06) (Q-14, 1.05) (Q-32, 1.05)
    };
    \addlegendentry{MatGPTQ}
    \addplot[black!50, densely dotted, forget plot] coordinates {(Q-0.5, 1) (Q-32, 1)};
    \end{axis}
    \end{tikzpicture}
    \caption{Latency speedup of \texttt{codeNx8\_Had}.}
    \label{fig:hadamard_speedup}
\end{minipage}
\end{figure}

\paragraph{DbyD vs.\ SOTA comparison.}
We compare DbyD (Had.+35W) against three closely related state-of-the-art (SOTA)
post-training quantization methods that support multi-bitwidth inference:
AnyBCQ, AnyPrecLLM, and MatGPTQ. DbyD offers advantages across several dimensions.
\textit{Quality:} DbyD achieves the lowest perplexity at nearly every evaluated
precision across all model scales (Table~\ref{tab:dense-sota},
Appendix~\ref{app:ppldetailedsota}), while maintaining zero-shot accuracy
comparable to the SOTA methods (Appendix~\ref{app:accdetailedsota}).
\textit{Storage:} a single DbyD checkpoint supports all evaluated bitwidths,
with disk usage comparable to or lower than the corresponding SOTA checkpoints
(Figure~\ref{fig:disk_usage_comparison}).
\textit{Latency:} our \texttt{codeNx8} kernels outperform every evaluated SOTA
kernel across all layers and bitwidths, with speedups of $1.02$--$1.46\times$.
The online Hadamard rotation adds negligible overhead (Figure~\ref{fig:hadamard_speedup};
Appendix~\ref{app:latencydetailedsota}).
\textit{Quantization resources:} DbyD quantizes all evaluated models on L40S GPUs,
whereas AnyPrecLLM and MatGPTQ require an H100 for
\textsc{Qwen2.5-32B}, and AnyBCQ requires an H100 for both
\textsc{Qwen2.5-14B} and \textsc{Qwen2.5-32B}.

\begin{table}[H]
\small
\centering
\caption{SOTA perplexity comparison ($\downarrow$) on \textsc{C4}.}
\setlength{\tabcolsep}{4pt}
\begin{tabular}{lcrrrrrrrr}
\toprule
\multirow{2}{*}{\textbf{Method}} & \multirow{2}{*}{\textbf{Bits}} &
\multicolumn{5}{c}{\textbf{Qwen2.5}} &
\textbf{Gemma} & \textbf{Mistral} & \textbf{Llama} \\
\cmidrule(lr){3-7}
& & \textbf{0.5B} & \textbf{3B} & \textbf{7B} & \textbf{14B} & \textbf{32B} &
\textbf{2B} & \textbf{7B} & \textbf{8B} \\
\midrule
DbyD (Had.+35W) & \multirow{4}{*}{5} & \textbf{17.87} & \textbf{11.32} & \textbf{10.51} & \textbf{9.06} & \textbf{8.79} & \textbf{11.79}         & \textbf{7.54} & \textbf{8.63} \\
AnyBCQ          &                    & 24.70          & 14.96         & \textit{nan}   & 11.96         & 11.63         & 25.66          & 17.48         & 10.24         \\
AnyPrecLLM      &                    & 22.44          & 13.77         & 12.93          & 11.01         & 10.70         & 13.67 & 8.52          & 9.68          \\
MatGPTQ         &                    & 23.25          & 13.62         & --             & 11.07         & 10.59         & 22.62          & 9.18          & 9.71          \\
\midrule
DbyD (Had.+35W) & \multirow{4}{*}{4} & \textbf{19.99} & \textbf{11.93} & \textbf{10.93} & \textbf{9.40} & \textbf{8.98} & \textbf{13.23}        & \textbf{7.71} & \textbf{9.38} \\
AnyBCQ          &                    & 25.89          & 15.36         & \textit{nan}   & 12.15         & 11.78         & 25.85          & 14.28         & 10.81         \\
AnyPrecLLM      &                    & 24.15          & 14.26         & 13.33          & 11.28         & 10.89         & 14.40 & 8.67          & 10.28         \\
MatGPTQ         &                    & 26.72          & 14.22         & --             & 11.47         & 10.81         & 23.25          & 9.95          & 10.60         \\
\midrule
DbyD (Had.+35W) & \multirow{4}{*}{3} & \textbf{27.67} & \textbf{13.96} & \textbf{12.27} & \textbf{10.38} & \textbf{9.49} & \textbf{17.07}       & \textbf{8.32} & \textbf{11.53} \\
AnyBCQ          &                    & 30.62          & 16.86         & \textit{nan}   & 12.98         & 12.30         & 27.79          & 10.49         & 13.34         \\
AnyPrecLLM      &                    & 34.78          & 16.71         & 15.52          & 12.44         & 11.64         & 17.58 & 9.39          & 13.14         \\
MatGPTQ         &                    & 80.06          & 134.12        & --             & 13.23         & 11.90         & 52.50          & 11.95         & 14.93         \\
\bottomrule
\end{tabular}
\vspace{2pt}
\parbox{\textwidth}{\footnotesize \textit{nan} denotes numerical failure with official code and default settings; -- denotes a non-converging configuration.}
\label{tab:dense-sota}
\end{table}

\vspace{-1cm}

\section{Discussion and Future Work}

\paragraph{Discussion.}
Drop-by-Drop demonstrates that a single quantized model can serve multiple resource budgets
without retraining or maintaining separate checkpoints. Conventional methods such as AQLM require
a separate model per target configuration, which in our setup means three independently trained
and stored models at $3\times8$, $4\times8$, and $5\times8$; DbyD requires a single training run
and a single checkpoint. Against state-of-the-art multi-bitwidth methods, it achieves the lowest
perplexity at nearly every precision with comparable zero-shot accuracy, a comparable or smaller
storage footprint, and lower decoding latency through our \texttt{codeNx8} kernels. The precision-performance trade-off is
selected at inference time simply by the number of codebooks used, degrading gracefully as the
bitwidth decreases.

\paragraph{Limitations and Future Work.}
While Drop-by-Drop closely approaches individually trained models across the evaluated
bitwidths, performance at even lower bitwidths leaves room for further improvement.
Our current implementation builds on AQLM's additive codebooks; extending Drop-by-Drop
to other quantization schemes, including ones that are cheaper/easier to quantize, is a natural direction.
More broadly, the theoretical motivation based on successive refinability for Gaussian
sources under weighted mean-squared error could be extended to other source
distributions and distortion measures. Finally, the proposed multi-bitwidth representation
enables new deployment strategies, including selecting bitwidths based on available
memory, compute, or latency budgets, allocating different precisions across layers
according to their sensitivity or resource constraints, and dynamically trading model
quality for inference speed as deployment conditions change.

 \subsection*{AI use statement}

In this work, we used generative AI tools to \textit{proofread} our theoretical proof. We have \textbf{not} used generative AI tools for generating synthetic data sets, helping develop theoretical models or conceptual frameworks, formulating mathematical claims, providing critical ingredients for proving mathematical claims, assisting in the writing of proofs, proposing or refining hypotheses, designing or providing feedback on research  methodology or experiments, implementing methods, cleaning and reformating dataset, supporting qualitative and thematic data analysis, or interpreting results; and assistance with translation is not applicable to this work. 

We used generative AI tools to identify relevant literature, to optimize software implementation, edit LaTeX, and proofread writing. We have reviewed all AI-assisted work. We
checked LLM suggested research papers and conducted our own literature reading. All LLM code and hyperparameter suggestions were tested and verified. We take responsibility for the final content of this work, including text, claims or artifacts produced with the aid of generative AI.

\subsection*{Reproducibility statement}
For our theoretical results, Appendices~\ref{app:theory} and~\ref{app:detailed_theorems} give the assumptions and full proofs behind the main theorem in Section~\ref{section:theoretical_Contribution}. For our practical design, our implementation builds on the AQLM codebase, with modifications released at \href{https://anonymous.4open.science/r/DropbyDrop2026/}{\texttt{https://anonymous.4open.science/r/DropbyDrop2026/}}, primarily affecting \texttt{aq\_engine.py}, \texttt{src/aq.py}, \texttt{src/modelutils.py} and \texttt{quantize.sh}. The experimental setup, hyperparameters, and evaluation benchmarks are listed in Section~\ref{section:experiments}. The repository contains the code needed to reproduce our quantized models; however the checkpoints and specialized inference kernels (Appendix~\ref{app:latencydetailedsota}) will be made publicly available on Hugging Face and GitHub upon publication.

\bibliography{iclr2027_conference}
\bibliographystyle{iclr2027_conference} 

 \clearpage %
\appendix

\section{Notation Summary}
\label{app:notation}

A summary of the notation used in our practical LLM weight-quantization setup is provided in Table~\ref{tab:notation} for reference and convenience of the reader. Matrices are denoted by bold uppercase letters (e.g., $\boldsymbol{W}$), vectors by bold lowercase letters (e.g., $\boldsymbol{x}$), and scalars by lowercase letters (e.g., $x$). Random variables are represented using sans-serif font (e.g., $\mathsf{W}$), while deterministic quantities use the standard math font.

\begin{table}[H]
\centering
\small
\caption{Notation summary}
\begin{tabular}{ll}
\toprule
\textbf{Symbol} & \textbf{Description} \\
\midrule
$\boldsymbol{\mathsf{W}} \in \mathbb{R}^{d_{\mathrm{out}} \times d_{\mathrm{in}}}$ & Weight matrix of a linear layer (Random variable) \\
$\widehat{\boldsymbol{\mathsf{W}}}$ & Quantized reconstruction of $\boldsymbol{\mathsf{W}}$ (Random variable) \\
$\boldsymbol{X} \in \mathbb{R}^{d_{\mathrm{in}} \times n}$ & Representative input data matrix \\
$d_{\mathrm{WMSE}}(\boldsymbol{\mathsf{W}},\widehat{\boldsymbol{\mathsf{W}}})$ & Weighted mean squared reconstruction error (WMSE) \\
$L_{\mathrm{WMSE}}(\boldsymbol{\mathsf{W}},\widehat{\boldsymbol{\mathsf{W}}})$ & Normalized, single-realization WMSE training loss \\

\midrule
\midrule
$\boldsymbol{\mathsf{Y}} = \boldsymbol{\mathsf{W}}\boldsymbol{X} \in \mathbb{R}^{d_{\mathrm{out}} \times n}$ & Layer output activations (Random variable) \\
$\boldsymbol{\mathsf{y}} = \operatorname{vec}(\boldsymbol{\mathsf{Y}}) \in \mathbb{R}^{d_{\mathrm{out}} \cdot n}$ & Vectorized layer output (Random variable) \\
$\boldsymbol{\mathsf{w}} = \operatorname{vec}(\boldsymbol{\mathsf{W}}) \in \mathbb{R}^{d_{\mathrm{out}} \cdot d_{\mathrm{in}}}$ & Vectorized weight matrix (Random variable) \\
\midrule
\midrule
$\widehat{\boldsymbol{\mathsf{W}}}^{(k)}$ & Reconstruction of $\boldsymbol{\mathsf{W}}$ at refinement stage $k$ \\
$\widehat{\boldsymbol{\mathsf{Y}}}^{(k)}$ & Output produced by the $k$-stage reconstruction \\
$D_k$ & Distortion level at refinement stage $k$ \\
$\Sigma_Y$ & Covariance of vectorized output $\boldsymbol{\mathsf{Y}}$ \\
$\lambda_i$ & Eigenvalues of $\Sigma_Y$ in Gaussian eigen-decomposition \\
$\mathsf{v}_i$ & $i$-th scalar Gaussian component in the eigen-basis ($\mathsf{v}_i \sim \mathcal{N}(0, \lambda_i)$) \\
\midrule
\midrule
$d_{\mathrm{in}}, d_{\mathrm{out}}$ & Input and output dimensions of the layer \\
$n$ & Number of calibration samples \\
$g$ & Group size for partitioning weights \\
$M$ & Number of additive codebooks \\
$C_m \in \mathbb{R}^{g \times 2^B}$ & $m$-th codebook containing $2^B$ codewords \\
$B$ & Number of bits used to index codewords \\
$b_{ijm} \in \{0,1\}^{2^B}$ & One-hot vector selecting a codeword from codebook $m$ \\
\midrule
\midrule
$k$ & Number of active codebooks at inference ($k \le M$) \\
$L^{(k)}$ & Drop-by-Drop loss at refinement stage $k$ \\
$\alpha_k$ & Matryoshka loss coefficient for refinement stage $k$ \\
\bottomrule
\end{tabular}
\label{tab:notation}
\end{table}

\clearpage
\section{AQLM's Optimization Process}
\label{app:AQLM-opt-process}
The quantization process in AQLM learns both the codebooks \( C \) and the assignments \( b \), where \( b \) encodes which codewords from each codebook are selected to approximate groups of weights. Initialization begins with a residual K-means procedure. Firstly, 
a standard K-means clustering is applied to the grouped weights to generate initial cluster centers and assignments~\citep{chen2010approximate}. Then, the quantization errors are computed as the differences between weights and their assigned centroids. Lastly, subsequent codebooks are iteratively initialized by clustering these \textit{residual} errors, enabling each additional codebook to compensate for errors left by the previous ones. This hierarchical initialization helps achieve an effective residual-based representation.

Learning proceeds through iterative alternating optimization of the discrete codes and continuous codebooks, consisting of three main phases:
\begin{enumerate}
    \item {\textbf{Code Optimization:}} The discrete assignment vectors \( b \) are refined using a beam search algorithm, which searches for combinations of codewords that minimize the quantization error between the original and quantized model outputs. This step focuses on minimizing the error in model outputs induced by quantization.
    \item {\textbf{Codebook Update:}} Given fixed codes, the codebooks \( C \) are updated by solving a least squares problem that minimizes the difference between the original and quantized model outputs. This gradient-based step optimizes the continuous parameters of the codebooks to better fit the data.
    \item {\textbf{Intra-Layer Fine-Tuning:}} Finally, to further reduce residual quantization errors, AQLM performs fine-tuning at the transformer block level. This phase \textit{jointly} adjusts the non-quantized model parameters and the codebooks, ensuring that the outputs of the quantized model closely match those of the original~\citep{shuangyi}.
\end{enumerate} 

\paragraph{Estimating Model Size.} To evaluate the practical memory footprint of an AQLM-quantized model, we estimate the total storage cost for a given codebook configuration. The memory required for a quantized weight matrix comprises three primary components: the codebooks, the discrete codes, and the per-unit scales. 

Specifically, for a weight matrix with input dimension $d_{in}$, output dimension $d_{out}$, and group size $g$, using $M$ codebooks corresponding to $B$-bit codes, the total amount of memory (in bits) is calculated as follows. The codebooks and scales are stored using standard 16-bit floating-point precision (FP16):

\begin{itemize}
    \item \textbf{Codebooks:} $g \cdot M \cdot 2^B \cdot 16$
    \item \textbf{Codes:} $d_{out} \cdot (d_{in}/g) \cdot B \cdot M$
    \item \textbf{Scales:} $d_{out} \cdot 16$
\end{itemize}

The \textit{average bits per parameter} ($\bar{b}$) is then computed by dividing the total size in bits by the total number of parameters:

\begin{equation}
\bar{b} = \frac{\text{size in bits}}{\text{number of parameters}} = \frac{16 * g *M* 2^B + d_{out}* (d_{in}/g)*B*M+ 16 *d_{out}}{d_{out} *d_{in}}
\label{eq:avg_bits}
\end{equation}

As a concrete example, consider the \texttt{mlp.gate\_proj} layer of a Llama 2 70B model ($d_{in} = 8192$, $d_{out} = 28672$). With a group size $g=8$ and two 8-bit codebooks ($M=2, B=8$), Equation~\ref{eq:avg_bits} yields approximately $2.002$ bits per parameter. Typically, the storage cost is dominated by the codes, while the codebooks and scales induce a relatively small memory overhead~\citep{vahe}.

\clearpage
\section{Empirical Validation of the Gaussian Source Assumption}
\label{app:gaussianplots}
\subsection{Weight Distribution Visualization}
Figure~\ref{fig:plotsss} shows distributions of LLM weight parameters from the first two transformer layers of \textsc{Qwen2.5-7B-Instruct}. Each layer contains seven weight matrices - four attention and three multi-layer perceptron (MLP) projections.

\begin{figure}[h]
    \centering
    \def\projections{{"self_attn.q_proj", "self_attn.k_proj", "self_attn.v_proj", "self_attn.o_proj", "mlp.gate_proj", "mlp.up_proj", "mlp.down_proj"}}

    \begin{minipage}[t]{0.45\textwidth}
        \centering
        \foreach \l [count=\i from 0] in {0,...,6}{
            \pgfmathsetmacro{\c}{\projections[\i]}
            \includegraphics[width=0.57\textwidth]{imgs/plots/March_1_FITTED_scaled_weight_hist_layer_\l_Linear.png} \\
            \vspace{-0.5em}
            {\tiny Layer 0: \texttt{\detokenize\expandafter{\c}}} \\ \vspace{1em}
        }
    \end{minipage}
    \hfill
    \begin{minipage}[t]{0.45\textwidth}
        \centering
        \foreach \l [count=\i from 0] in {7,...,13}{
            \pgfmathsetmacro{\c}{\projections[\i]}
            \includegraphics[width=0.57\textwidth]{imgs/plots/March_1_FITTED_scaled_weight_hist_layer_\l_Linear.png} \\
            \vspace{-0.5em}
            {\tiny Layer 1: \texttt{\detokenize\expandafter{\c}}} \\ \vspace{1em}
        }
    \end{minipage}

    \caption{\textsc{Qwen2.5-7B-Instruct} weights fitted to Gaussian distributions}
    \label{fig:plotsss}
\end{figure}

\subsection{Kurtosis Analysis and Walsh-Hadamard Rotations}
In Section~\ref{section:theoretical_Contribution}, we established that a Gaussian source 
is successively refinable under WMSE distortion. In practice, however, LLM weight distributions are 
leptokurtic: they have heavier tails than a Gaussian, with excess kurtosis 
values ranging from $+0.7$ to $+6.1$ across projection types, as shown in 
Table~\ref{tab:kurtosis}. The closer the weight distribution is to Gaussian, 
the more tightly the successive refinement guarantee applies, which 
motivates pre-processing the weights before quantization.

Prior work has shown that applying a random orthogonal transformation, particularly a Hadamard transform, to LLM weight matrices redistributes weight magnitude uniformly across coordinates, enforcing a better match between the processed weight distribution and the Gaussian distribution~\cite{ashkboos2024quarot, QUIPsharp, chee2023quip, malinovskii2025higgs, palette}.
Motivated by our theoretical result that Gaussian sources are successively 
refinable under WMSE, we incorporate a block-wise Walsh-Hadamard rotation 
into the Drop-by-Drop pipeline as an optional pre-processing step, extending 
this technique to the multi-bitwidth additive codebook setting. Table~\ref{tab:kurtosis} confirms the effect: model-average excess kurtosis 
drops across all five \textsc{Qwen} model sizes. 

\begin{table}[h]
\centering
\caption{Excess kurtosis of raw and block-Hadamard-rotated weights across 
the Qwen model family, measured per projection type. By definition, a pure 
Gaussian has excess kurtosis of $0$. Rotation consistently reduces 
leptokurtosis across all models and projections, bringing the weight 
distributions closer to the Gaussian source assumed by our theory.}
\small
\begin{tabular}{llcccccc}
\toprule
\multirow{2}{*}{\textbf{Projection}} &
& \multicolumn{2}{c}{\textbf{Qwen-0.5B}}
& \multicolumn{2}{c}{\textbf{Qwen-3B}}
& \multicolumn{2}{c}{\textbf{Qwen-7B}} \\
\cmidrule(lr){3-4}
\cmidrule(lr){5-6}
\cmidrule(lr){7-8}
& & Raw & Rotated & Raw & Rotated & Raw & Rotated \\
\midrule
q\_proj    && +5.5267 & +2.6913 & +1.9793 & +1.2025 & +2.0034 & +1.0335 \\
k\_proj    && +3.1480 & +1.5132 & +2.0037 & +1.0912 & +2.4605 & +1.0302 \\
v\_proj    && +1.9380 & +0.1212 & +0.9248 & +0.0675 & +1.1541 & +0.1662 \\
o\_proj    && +6.0979 & +1.6146 & +1.8030 & +0.4334 & +3.2082 & +0.5911 \\
gate\_proj && +1.9406 & +0.7544 & +2.1402 & +1.5026 & +5.5342 & +0.2676 \\
up\_proj   && +1.8242 & +0.7688 & +2.6955 & +1.9819 & +0.7627 & +0.2531 \\
down\_proj && +1.7905 & +0.3604 & +2.9241 & +0.2324 & +2.2574 & +0.1585 \\
\midrule
\textbf{Model Avg}
&& \textbf{+3.1809} & \textbf{+1.1177}
& \textbf{+2.0672} & \textbf{+0.9302}
& \textbf{+2.4829} & \textbf{+0.5000} \\
\bottomrule
\end{tabular}

\vspace{0.5em}

\begin{tabular}{llcccc}
\toprule
\multirow{2}{*}{\textbf{Projection}} &
& \multicolumn{2}{c}{\textbf{Qwen-14B}}
& \multicolumn{2}{c}{\textbf{Qwen-32B}} \\
\cmidrule(lr){3-4}
\cmidrule(lr){5-6}
& & Raw & Rotated & Raw & Rotated \\
\midrule
q\_proj    && +3.0286 & +0.5989 & +2.5299 & +0.5457 \\
k\_proj    && +2.3982 & +0.7719 & +2.2827 & +0.7528 \\
v\_proj    && +0.7241 & +0.1547 & +0.7523 & +0.1479 \\
o\_proj    && +1.6667 & +0.4958 & +1.5147 & +0.4745 \\
gate\_proj && +1.9359 & +0.9059 & +1.9455 & +0.6978 \\
up\_proj   && +2.1167 & +1.0816 & +2.7827 & +1.0451 \\
down\_proj && +2.8015 & +0.6878 & +2.6904 & +0.5914 \\
\midrule
\textbf{Model Avg}
&& \textbf{+2.0960} & \textbf{+0.6709}
& \textbf{+2.0712} & \textbf{+0.6079} \\
\bottomrule
\end{tabular}
\label{tab:kurtosis}
\end{table}

\clearpage

\section{Theoretical Foundations and Successive Refinement}
\label{app:theory}
\subsection{Definition of Successive Refinement} 
The successive refinement problem, introduced by~\cite{successiveref}, asks whether a source 
$(\mathsf{W}, d)$ can be encoded in stages with no rate loss. Specifically, for any finite 
distortion chain $D_1 > D_2 > \cdots > D_K > 0$, we ask whether there exists a single 
embedded code whose $k$-prefix achieves the optimal rate-distortion pair $(R(D_k), D_k)$ 
for all $k = 1, \dots, K$, with no rate penalty compared to encoding directly at $D_k$. 
For target distortions $D_1 > D_2$, this means encoding ${\mathsf{W}}^N$ first at rate $R_1$ 
with distortion $D_1$, then refining by an additional rate $R_2 - R_1$ to reach distortion 
$D_2$, without exceeding the total rate $R(D_2)$. The following definition formalizes this notion.

\begin{definition}[Successive Refinability of General Sources] 
\label{def:SR-general}
A source is said to be successively refinable from $D_1$ to $D_2$ if there exists a sequence of encoders
\[
i_N : \mathcal{W}^N \to \{1,\ldots,2^{NR_1}\}, \qquad
j_N : \mathcal{W}^N \to \{1,\ldots,2^{\,N(R_2-R_1)}\},
\]
and respective reconstruction functions (decoders)
\[
g_{1N} : \{1,\ldots,2^{NR_1}\} \to \widehat{\mathcal{W}}_1^{\,N},
\qquad
g_{2N} : \{1,\ldots,2^{NR_1}\} \times \{1,\ldots,2^{\,N(R_2 - R_1)}\}
\to \widehat{\mathcal{W}}_2^{\,N}.
\]
such that the reconstructions
\[
\widehat{\mathsf{W}}_1^{\,N}
= g_{1N}\!\big(i_N(\mathsf{W}^{\,N})\big), 
\qquad
\widehat{\mathsf{W}}_2^{\,N}
= g_{2N}\!\big(i_N(\mathsf{W}^{\,N}),\, j_N(\mathsf{W}^{\,N})\big),
\]
satisfy
\[
\limsup_{N\to\infty}
\mathbb{E}\!\left[d(\mathsf{W}^{\,N},\,\widehat{\mathsf{W}}_1^{\,N})\right]
\le D(R_1) + \varepsilon_N,
\qquad
\limsup_{N\to\infty}
\mathbb{E}\!\left[d(\mathsf{W}^{\,N},\,\widehat{\mathsf{W}}_2^{\,N})\right]
\le D(R_2) + \varepsilon_N,
\]
where $\varepsilon_N \to 0$ as $N \to \infty$ and $D(R)$ is the distortion–rate function. Then, 

$$
R_1 = R(D_1) + \varepsilon_1, 
\qquad 
R_{2|1} = R(D_2) - R(D_1) + \varepsilon_2,
$$
and hence the total rate satisfies
$$
R_{\text{total}} = R_1 + R_{2|1} = R(D_2) + \varepsilon_1 + \varepsilon_2,$$
where $\varepsilon_1, \varepsilon_2 > 0$ are arbitrarily small constants. Note that $R(D_1)$ and $R(D_2)$ are theoretical rate–distortion limits, whereas 
$R_1$, $R_{2|1}$, and $R_{\text{total}}$ are empirical operational rates 
of the actual coding scheme. Here, the rate in stages matches the single-shot rate required to achieve distortion \(D_2\).
Thus, encoding in two stages (first at rate \(R_1\) and then refining by an additional \(R_{2|1}\)) achieves the same operational point as single-shot encoding at rate \(R(D_2)\), with no rate penalty. More formally, the point $(R_1, R_2) = (R(D_1), R(D_2))$ is asymptotically achievable by a two-stage description in which 
the second stage refines the first. Generally, a problem defined by a source distribution $p(\mathsf{w})$ and distortion measure $d(w,\widehat{w})$ 
is said to be successively refinable,
if successive refinement from distortion $D_1$ to $D_2$ 
is achievable for every $D_1 \ge D_2$.
\end{definition}

\subsection{Classical Result: Gaussian Sources are Successively Refinable}
Next, we show that Gaussian sources are successively refinable. We first establish the single-source case in Theorem~\ref{thm:gaussian-sr}.

\begin{theorem}[Successive Refinability of Gaussian Sources]
\label{thm:gaussian-sr}
A Gaussian source 
$\mathsf{w} \sim \mathcal{N}(0,\sigma^2)$ 
under the mean squared error (MSE) distortion measure
$d(w,\widehat{w}) = (w - \widehat{w})^2$
is successively refinable~\cite{equitzthesis}.
\end{theorem}

Since the Gaussian source achieves its rate-distortion 
function by a single-letter test channel, the $N$-block 
problem separates into $N$ independent scalar problems.  
For notational clarity, we therefore work with the 
single-letter random variables 
$\mathsf{w}$, 
$\widehat{\mathsf{w}}^{(1)}$, and 
$\widehat{\mathsf{w}}^{(2)}$. The superscripts denote refinement stages.

The Gaussian rate-distortion function is
$$R(D) = \frac{1}{2}\log\!\frac{\sigma^2}{D}, 
\qquad 0 < D \le \sigma^2.$$

The scalar Gaussian test channel induces the Markov chain
$
\widehat{\mathsf{w}}^{(1)}
\;\rightarrow\;
\widehat{\mathsf{w}}^{(2)}
\;\rightarrow\;
\mathsf{w},$
which ensures that each refinement layer depends only on its predecessor. Concretely, define mutually independent Gaussian increments
\[
\mathsf z^{(1)} \sim \mathcal N(0,\sigma^2-D_1), \qquad \mathsf z^{(2)} \sim \mathcal N(0,D_1-D_2),
\]
and construct the successive reconstructions by adding these increments,
\[
\widehat{\mathsf w}^{(0)} = 0, \qquad \widehat{\mathsf w}^{(k)} = \widehat{\mathsf w}^{(k-1)} + \mathsf z^{(k)}, \qquad k=1,2.
\]

Furthermore, the additive Gaussian construction satisfies
\[
I(\mathsf{w};\widehat{\mathsf{w}}^{(1)}) = R(D_1),
\qquad
I(\mathsf{w};\widehat{\mathsf{w}}^{(1)},\widehat{\mathsf{w}}^{(2)}) = R(D_2).
\]

Realizing this single-letter distribution operationally requires coding over blocks of length $N$. Construction of the code is below, as we work with the block sequences $\mathsf w^N,\widehat{\mathsf w}^{(1)N},\widehat{\mathsf w}^{(2)N}$. \\

Let $\mathcal C_1$ represent the first codebook, consisting of the codewords
\[
\mathcal C_1 := \big\{\widehat{\mathsf w}^{(1)N}(m_1)\big\}_{m_1=1}^{2^{NR_1}}, \qquad m_1 \in \{1,\ldots,2^{NR_1}\},
\]
with each codeword drawn i.i.d., $\mathcal N(0,\sigma^2-D_1)$. \\

Similarly, for each $m_1$, let $\mathcal C_2(m_1)$ represent the second codebook, consisting of the codewords
\[
\mathcal C_2(m_1) := \big\{\widehat{\mathsf w}^{(2)N}(m_1,m_2)\big\}_{m_2=1}^{2^{NR_{2|1}}}, \qquad m_2 \in \{1,\ldots,2^{NR_{2|1}}\},
\]
with each codeword drawn i.i.d. $\widehat{\mathsf w}^{(2)N}(m_1,m_2) = \widehat{\mathsf w}^{(1)N}(m_1) + \mathsf z^{(2)N}$, $\mathsf z^{(2)}\sim\mathcal N(0,D_1-D_2)$. \\

The encoder $i_N$ selects an index of a codeword in $\mathcal C_1$ that is the closest reconstruction to $\mathsf w^N$,
\[
i_N(\mathsf w^N) := \operatorname*{arg\,min}_{m_1\in\{1,\ldots,2^{NR_1}\}}\ d\big(\mathsf w^N,\widehat{\mathsf w}^{(1)N}(m_1)\big),
\]
and the decoder $g_{1N}$ outputs the codeword at that selected index,
\[
g_{1N}\big(i_N(\mathsf w^N)\big) = \widehat{\mathsf w}^{(1)N}\big(i_N(\mathsf w^N)\big).
\]
Given the selected index $i_N(\mathsf w^N)$, the encoder $j_N$ selects an index of a codeword in $\mathcal C_2(i_N(\mathsf w^N))$ that is the closest reconstruction to $\mathsf w^N$,
\[
j_N(\mathsf w^N) := \operatorname*{arg\,min}_{m_2\in\{1,\ldots,2^{NR_{2|1}}\}}\ d\Big(\mathsf w^N,\widehat{\mathsf w}^{(2)N}\big(i_N(\mathsf w^N),m_2\big)\Big),
\]
and, likewise, given both selected indices, the decoder $g_{2N}$ outputs
\[
g_{2N}\big(i_N(\mathsf w^N),j_N(\mathsf w^N)\big) = \widehat{\mathsf w}^{(2)N}\big(i_N(\mathsf w^N),j_N(\mathsf w^N)\big).
\]

Following Definition~\ref{def:SR-general}, for any distortion pair $D_1 > D_2 > 0$, the constructed encoders and decoders achieve
$$
R_1 = R(D_1) + \varepsilon_1, 
\qquad 
R_{2|1} = R(D_2) - R(D_1) + \varepsilon_2,$$
and hence a total rate
$$
R_{\text{total}} = R_1 + R_{2|1} = R(D_2) + \varepsilon_1 + \varepsilon_2,$$
where $\varepsilon_1, \varepsilon_2 > 0$ can be made arbitrarily small.

Hence, the achievable region coincides exactly with the single-shot 
rate-distortion limits.  
In other words, no additional rate is required when encoding the 
Gaussian source in successive refinement stages compared to encoding 
it once at distortion $D_2$.  
This establishes that Gaussian sources under MSE distortion are 
successively refinable.

\paragraph{Achievability.} Following the approach by Cover \& Thomas \cite[Section~10.5]{textbook}, the achievability is established using standard typicality arguments via the Asymptotic Equipartition Property (AEP). Let $\varepsilon_1,\varepsilon_2>0$, and let typicality penalties $0 < \delta_1 < \varepsilon_1$ and $0 < \delta_2 < \varepsilon_2$.

\emph{Generation of codebooks.} As constructed above, the first codebook $\mathcal C_1$ consists of $2^{NR_1}$ codewords of length $N$, where each symbol of every codeword is drawn i.i.d.\ following $\mathcal N(0,\sigma^2-D_1)$. For each $m_1$, the second codebook $\mathcal C_2(m_1)$ consists of $2^{NR_{2|1}}$ codewords, built by taking the base codeword $\widehat{\mathsf w}^{(1)N}(m_1)$ and adding an independent sequence of Gaussian increments $\mathsf z^{(2)N}$. Reveal these codebooks to the encoder and the decoder.

\emph{Encoding.} The encoder first looks for an index $m_1$ such that $\big(\mathsf w^N, \widehat{\mathsf w}^{(1)N}(m_1)\big) \in A_{d,\delta_1}^{(N)}$, the first-stage distortion typical set. If there is more than one such $m_1$, send the least. If there is no such $m_1$, let $m_1=1$. The chosen index is denoted as $i_N(\mathsf w^N)$. Similar steps are followed to select the second-stage index $j_N(\mathsf w^N) = m_2$ from the refinement codebook $\mathcal C_2(i_N(\mathsf w^N))$ such that the sequences fall within the second-stage distortion typical set $A_{d,\delta_2}^{(N)}$.

\emph{Decoding.} $g_{1N},g_{2N}$ output the selected codewords.

\emph{Calculation.} Let $A_{\gamma}^{(N)}$ denote the typical set of the source for some tolerance $\gamma > 0$. By the joint AEP, for any typical source sequence $\mathsf w^N \in A_{\gamma}^{(N)}$, the probability that a randomly drawn codeword falls within the joint distortion typical set $A_{d,\delta_1}^{(N)}$ satisfies
$$
\Pr\Big[\big(\mathsf w^N,\widehat{\mathsf w}^{(1)N}(m_1)\big) \in A_{d,\delta_1}^{(N)}\Big] \ge 2^{-N(R(D_1)+\delta_1)}
$$
for $N$ sufficiently large. Let $P_e^{(1)}$ denote the probability that \textit{no} codeword in $\mathcal C_1$ falls within this first-stage distortion typical set. We bound this by the probability that the source is atypical, plus the probability that the codebook fails for a typical source:
$$
P_e^{(1)} \le \Pr\big(\mathsf w^N \notin A_{\gamma}^{(N)}\big) + \Big(1-2^{-N(R(D_1)+\delta_1)}\Big)^{2^{N(R(D_1)+\varepsilon_1)}}.
$$
By the standard AEP, the first term $\Pr\big(\mathsf w^N \notin A_{\gamma}^{(N)}\big) \to 0$. For the second term, since $\delta_1 < \varepsilon_1$,
$$
\Big(1-2^{-N(R(D_1)+\delta_1)}\Big)^{2^{N(R(D_1)+\varepsilon_1)}} \le \exp\Big(-2^{N(\varepsilon_1 - \delta_1)}\Big) \xrightarrow[N\to\infty]{} 0.
$$
Therefore, the total probability of encoding error $P_e^{(1)} \to 0$. 
Conditioned on the event that $\mathcal C_1$ succeeds, the identical argument applied to $\mathcal C_2(i_N(\mathsf w^N))$ using the respective typicality bound $\delta_2 < \varepsilon_2$ gives $P_e^{(2)}\to0$.

On the event that both codebooks succeed (which occurs with probability $\ge 1-P_e^{(1)}-P_e^{(2)}\to 1$), the reconstructions satisfy $d(\mathsf w^N,\widehat{\mathsf w}^{(1)N})\le D_1$ and $d(\mathsf w^N,\widehat{\mathsf w}^{(2)N})\le D_2$ by construction. Hence
$$
\mathbb E\big[d(\mathsf w^N,\widehat{\mathsf w}^{(1)N})\big] \to D_1, \qquad \mathbb E\big[d(\mathsf w^N,\widehat{\mathsf w}^{(2)N})\big] \to D_2,
$$
so the reconstructions $\widehat{\mathsf W}_1^{\,N}=g_{1N}(i_N(\mathsf w^N))$ and $\widehat{\mathsf W}_2^{\,N}=g_{2N}(i_N(\mathsf w^N),j_N(\mathsf w^N))$ satisfy the distortion bounds of Definition~\ref{def:SR-general} at rates $R_1=R(D_1)+\varepsilon_1$ and $R_{2|1}=R(D_2)-R(D_1)+\varepsilon_2$.

\subsection{Definition of Successive Refinement for Vector Sources}
Next, we extend this successive refinability result to vector Gaussian sources. 

\begin{definition}[Successive Refinability of Vector Sources]
\label{def:SR-vector}
A vector source $\boldsymbol{\mathsf w} = (\mathsf w_1,\ldots,\mathsf w_m) \in \mathcal W^m$ with independent components is said to be successively refinable across a distortion chain $D^{(1)} > D^{(2)} > \cdots > D^{(K)}$ if there exists a sequence of encoders
\[
i_N^{(1)} : (\mathcal W^m)^N \to \{1,\ldots,2^{N R^{(1)}}\}, \qquad
i_N^{(k)} : (\mathcal W^m)^N \to \{1,\ldots,2^{\,N(R^{(k)}-R^{(k-1)})}\}, \quad k=2,\ldots,K,
\]
and respective reconstruction functions (decoders)
\[
g_N^{(k)} : \{1,\ldots,2^{NR^{(1)}}\} \times \cdots \times \{1,\ldots,2^{\,N(R^{(k)}-R^{(k-1)})}\} \to \big(\widehat{\mathcal W}_k^{\,m}\big)^N, \qquad k=1,\ldots,K,
\]
such that the stage-$k$ reconstructions
\[
\widehat{\boldsymbol{\mathsf w}}^{(k)N} = g_N^{(k)}\!\big(i_N^{(1)}(\boldsymbol{\mathsf w}^N),\ldots,i_N^{(k)}(\boldsymbol{\mathsf w}^N)\big)
\]
satisfy
\[
\limsup_{N\to\infty} \mathbb E\big[d(\boldsymbol{\mathsf w}^N,\widehat{\boldsymbol{\mathsf w}}^{(k)N})\big] \le D^{(k)} + \varepsilon_N^{(k)}, \qquad k=1,\ldots,K,
\]
where $\varepsilon_N^{(k)}\to0$ as $N\to\infty$ and our distortion measure is the mean squared error, $d(\boldsymbol w,\widehat{\boldsymbol w}) = \sum_{i=1}^m (w_i-\widehat w_i)^2$. The cumulative rate through stage $k$ is $R^{(k)} = R(D^{(k)}) + \sum_{j=1}^k \varepsilon_j$, where the rate excesses $\varepsilon_j > 0$ are arbitrarily small constants. A vector source is said to be successively refinable, if successive refinement is achievable across the distortion chain.
\end{definition}

\subsection{Gaussian Vector Sources are Successively Refinable}
The original construction of ~\cite{equitzthesis} considers a single Gaussian source $\mathsf{w} \sim \mathcal{N}(0,\sigma^2)$. We build the analogous construction component-wise for an independent Gaussian vector source
\[
\mathsf{w}_i \sim \mathcal{N}(0,\sigma_i^2), \qquad i = 1,\ldots,m.
\]

\begin{theorem}[Successive Refinability of Gaussian Vector Sources]
\label{thm:gaussian-vector-sr}
Let $\boldsymbol{\mathsf{w}} = (\mathsf{w}_1, \ldots, \mathsf{w}_m)$ be a vector of independent Gaussian sources $\mathsf{w}_i \sim \mathcal{N}(0,\sigma_i^2)$, $i = 1, \ldots, m$, under the mean squared error distortion measure $d(\boldsymbol{w}, \widehat{\boldsymbol{w}}) = \sum_{i=1}^{m} (w_i - \widehat{w}_i)^2$. Then $\boldsymbol{\mathsf{w}}$ is successively refinable across any number of stages $K$.
\end{theorem}

For each stage $k = 1, \ldots, K$, the per-component distortion targets are chosen via reverse water-filling,
\[
    D_i^{(k)} \;=\; \min\!\left(\sigma_i^2,\, \mu_k\right), \qquad
    D_i^{(0)} \;\coloneqq\; \sigma_i^2,
\]
where $\mu_1 > \mu_2 > \cdots > \mu_K > 0$ is a strictly decreasing
sequence of water levels chosen so that $\sum_i D_i^{(k)} = D^{(k)}$ at
each stage. This is the classical solution to the parallel Gaussian
rate-distortion problem~\cite[Ch.~10]{textbook}, guaranteeing that
$R^{(k)} = \sum_i R_i^{(k)}$ equals the true rate-distortion function
$R(D^{(k)})$ at each stage. This gives, in particular, a nested chain of distortion targets for each
component,
\[
    D_i^{(0)} = \sigma_i^2 \;\geq\; D_i^{(1)} \;\geq\; \cdots
    \;\geq\; D_i^{(K)} \;>\; 0,
\]
Define mutually independent Gaussian increments

\[
\mathsf{z}_i^{(k)} \sim \mathcal{N}\!\left(0,\, D_i^{(k-1)} - D_i^{(k)}\right), \qquad k = 1, \ldots, K,
\]
and construct the successive reconstructions by accumulating these increments,
\[
\widehat{\mathsf{w}}_i^{(0)} = 0, \qquad
\widehat{\mathsf{w}}_i^{(k)} = \widehat{\mathsf{w}}_i^{(k-1)} + \mathsf{z}_i^{(k)}, \qquad k = 1, \ldots, K.
\]
By construction, each stage lands exactly on its target distortion,
\[
\mathbb{E}\!\left[ \big(\mathsf{w}_i - \widehat{\mathsf{w}}_i^{(k)}\big)^2 \right] = D_i^{(k)}, 
\qquad k = 1, \ldots, K.
\]

Collecting components into vectors makes the joint structure explicit: let
$\boldsymbol{\mathsf{w}} = (\mathsf{w}_1,\ldots,\mathsf{w}_m)^\top$, 
$\widehat{\boldsymbol{\mathsf{w}}}^{(k)} = (\widehat{\mathsf{w}}_1^{(k)},\ldots,\widehat{\mathsf{w}}_m^{(k)})^\top$, and 
$\boldsymbol{D}^{(k)} = \mathrm{diag}(D_1^{(k)},\ldots,D_m^{(k)})$, with $\boldsymbol{D}^{(0)} = \mathrm{diag}(\sigma_1^2,\ldots,\sigma_m^2)$. Then
\[
\widehat{\boldsymbol{\mathsf{w}}}^{(k)} 
= \widehat{\boldsymbol{\mathsf{w}}}^{(k-1)} + \boldsymbol{\mathsf{z}}^{(k)}, 
\qquad \boldsymbol{\mathsf{z}}^{(k)} \sim \mathcal{N}\!\big(\boldsymbol{0},\, \boldsymbol{D}^{(k-1)} - \boldsymbol{D}^{(k)}\big), 
\quad k = 1, \ldots, K,
\]
where $\boldsymbol{\mathsf{z}}^{(k)} \perp \widehat{\boldsymbol{\mathsf{w}}}^{(k-1)}$. Hence, the Markov chain essential for successive refinability holds by construction:
\[
\widehat{\boldsymbol{\mathsf{w}}}^{(1)} 
\;\rightarrow\; \widehat{\boldsymbol{\mathsf{w}}}^{(2)} 
\;\rightarrow\; \cdots
\;\rightarrow\; \widehat{\boldsymbol{\mathsf{w}}}^{(K)} 
\;\rightarrow\; \boldsymbol{\mathsf{w}}.
\]

\textit{Achievable Rates.} The achievable rate required to reach distortion $D_i^{(k)}$ for component $i$, encoded from scratch, is
\[
R_i^{(k)} = \frac{1}{2}\log\!\frac{\sigma_i^2}{D_i^{(k)}},
\qquad k = 1, \ldots, K,
\]
and for the $m$-component vector, $R^{(k)} = \sum_{i=1}^{m} R_i^{(k)}$. Under the staged construction, the rate spent at stage $k$ is the mutual information between $\mathsf{w}_i$ and its new reconstruction, given the reconstruction from the previous stage:
\[
I\!\big(\mathsf{w}_i;\,\widehat{\mathsf{w}}_i^{(k)}
         \mid \widehat{\mathsf{w}}_i^{(k-1)}\big)
= h(\mathsf{w}_i \mid \widehat{\mathsf{w}}_i^{(k-1)})
 - h(\mathsf{w}_i \mid \widehat{\mathsf{w}}_i^{(k-1)},
                        \widehat{\mathsf{w}}_i^{(k)})
= \tfrac{1}{2}\log\!\frac{D_i^{(k-1)}}{D_i^{(k)}},
\]
and summing over components gives the total incremental rate at stage $k$,
\[
\Delta R^{(k)}
= \frac{1}{2}\sum_{i=1}^{m}\log\!\frac{D_i^{(k-1)}}{D_i^{(k)}}
= R^{(k)} - R^{(k-1)}, \qquad k = 1, \ldots, K.
\]

\textit{Single-Shot Encoding Equivalence.} Encoding directly to any distortion level $D^{(k)}$, without passing through the intermediate stages, requires rate $R^{(k)} = \frac{1}{2}\sum_{i=1}^{m}\log\frac{\sigma_i^2}{D_i^{(k)}}$. Summing the incremental rates $\Delta R^{(j)}$ from stage $1$ through stage $k$ shows this same total is reached by the staged construction, since each intermediate distortion level cancels against itself:
\begin{align*}
\sum_{j=1}^{k} \Delta R^{(j)}
&= \frac{1}{2}\sum_{i=1}^{m}\left[
\log\!\frac{D_i^{(0)}}{D_i^{(1)}}
+ \log\!\frac{D_i^{(1)}}{D_i^{(2)}}
+ \cdots
+ \log\!\frac{D_i^{(k-1)}}{D_i^{(k)}}
\right] \\
&= \frac{1}{2}\sum_{i=1}^{m}\log\!\frac{D_i^{(0)}}{D_i^{(k)}} \\
&= \frac{1}{2}\sum_{i=1}^{m}\log\!\frac{\sigma_i^2}{D_i^{(k)}} \\
&= R^{(k)},
\qquad k = 1, \ldots, K.
\end{align*}
Hence there is no rate penalty for encoding in $K$ successive stages rather than in a single shot at any target distortion level. 

\paragraph{Achievability.}
The achievability is established by using AEP to the parallel scalar channels. Let $\varepsilon_{k,i} > 0$ and let typicality penalties $0 < \delta_{k,i} < \varepsilon_{k,i}$ for each component $i$ at refinement stage $k$.

\emph{Generation of codebooks.} For each active component $i$ at stage $k=1$, generate an independent codebook $\mathcal C_{1,i}$ consisting of $2^{N R_{1,i}}$ codewords drawn i.i.d.\ following $\mathcal N(0,\sigma_i^2-D_i^{(1)})$, where $R_{1,i} = \frac{1}{2}\log(\sigma_i^2/D_i^{(1)}) + \varepsilon_{1,i}$. For successive stages $k > 1$, for each active component $i$, generate additional codebooks consisting of $2^{N(R_{k,i} - R_{k-1,i})}$ additive Gaussian increments drawn i.i.d.\ as $\mathcal N(0,D_i^{(k-1)}-D_i^{(k)})$. Reveal these codebooks to the encoder and the decoder.

\emph{Encoding.} At stage $k$, for each active component $i$, the encoder independently looks for a codeword index $m_{k,i}$ such that the sequence and its reconstruction fall within the scalar distortion typical set $A_{d,\delta_{k,i}}^{(N)}$. If there is more than one such index, send the least. If there is no such index, default to the first index. The total stage-$k$ index transmitted is the concatenation of these parallel indices, with a total incremental rate $\Delta R^{(k)} = \sum_{i=1}^m (R_{k,i} - R_{k-1,i})$.

\emph{Decoding.} The decoders receive the concatenated indices, separate them by component, and output the vector sequence composed of the individual reconstructions.

\emph{Calculation.} Similar to AEP used in Theorem~\ref{thm:gaussian-sr}, the probability $P_{e,i}^{(k)}$ that a specific component $i$ fails to find a typical codeword at stage $k$ is bounded by:
$$
P_{e,i}^{(k)} \le \exp\Big(-2^{N(\varepsilon_{k,i} - \delta_{k,i})}\Big) \xrightarrow[N\to\infty]{} 0.
$$
Because the vector encoder fails if at least one of the active components fails, we apply a union bound over the $m$ independent components. Let $P_{e,\mathrm{total}}^{(k)}$ be the total probability of an encoding error at stage $k$:
$$
P_{e,\mathrm{total}}^{(k)} \le \sum_{i=1}^m P_{e,i}^{(k)} \xrightarrow[N\to\infty]{} 0.
$$
Because $m$ is finite and each individual error probability vanishes double-exponentially, the union bound strictly goes to $0$. When all parallel codebooks succeed, the vector expectation is simply the sum of the scalar expectations. Hence,
$$
\mathbb E\big[d(\boldsymbol{\mathsf w}^N,\widehat{\boldsymbol{\mathsf w}}^{(k)N})\big] = \frac{1}{N}\sum_{i=1}^m \mathbb E\big[(\mathsf w_i^N - \widehat{\mathsf w}_i^{(k)N})^2\big] \to \sum_{i=1}^m D_i^{(k)} = D^{(k)}.
$$
$$$$
The achievable region matches the single-shot rate-distortion bounds with no rate penalty, hence, completing the proof.\\

This construction reappears later: after the diagonalization of $\Sigma_Y$ in Theorem~\ref{thm:Sakrison}, we obtain independent scalar Gaussian components with eigenvalues $\lambda_i$ in place of $\sigma_i^2$, and apply the same per-component test channels to prove Theorem~\ref{thm:WMSE-SR}.
\clearpage
\section{Proofs of Auxiliary Results}
\label{app:detailed_theorems}

\subsection{Rate-Distortion Under Weighted MSE}

The classical rate-distortion function for Gaussian sources under MSE is solved by reverse water-filling~\cite[Ch.~10]{textbook}. Sakrison~\cite{sakrison1968rate} extended this to weighted square error for continuous-time processes using the Karhunen-Lo\`eve expansion. The following theorem adapts this result to the finite-dimensional matrix setting relevant to LLM quantization, replacing the Karhunen-Lo\`eve expansion with a standard eigen-decomposition. Then, it extends it to all structural cases of $\boldsymbol{X}$, which is necessary since the calibration matrix in LLM quantization is typically wide ($d_{\mathrm{in}} < n$).

\begin{theorem}[WMSE Rate-Distortion Equivalence]
\label{thm:Sakrison}
Let $\boldsymbol{\mathsf{W}} \in \mathbb{R}^{d_{\mathrm{out}} \times d_{\mathrm{in}}}$ 
be a zero-mean i.i.d.\ Gaussian random matrix, and let $\boldsymbol{X} \in 
\mathbb{R}^{d_{\mathrm{in}} \times n}$ be a deterministic matrix. Consider 
the weighted MSE distortion
\[
    d_{\mathrm{WMSE}}(\boldsymbol{\mathsf{W}}, \widehat{\boldsymbol{\mathsf{W}}})
    =
    \mathbb{E}\!\left[
    \bigl\|(\boldsymbol{\mathsf{W}} - \widehat{\boldsymbol{\mathsf{W}}})\,\boldsymbol{X}
    \bigr\|_F^2
    \right].
\]
Then, for the transformed source $\boldsymbol{\mathsf{Y}} = 
\boldsymbol{\mathsf{W}}\boldsymbol{X} \in \mathbb{R}^{d_{\mathrm{out}} \times n}$,
\[
    d_{\mathrm{WMSE}}(\boldsymbol{\mathsf{W}}, \widehat{\boldsymbol{\mathsf{W}}})
    \;=\;
    \mathbb{E}\!\left[\|\boldsymbol{\mathsf{Y}} - 
    \widehat{\boldsymbol{\mathsf{Y}}}\|_F^2\right],
\]
and the rate-distortion function for $(\boldsymbol{\mathsf{W}}, d_{\mathrm{WMSE}})$ 
equals the standard Gaussian MSE rate-distortion function for 
$\boldsymbol{\mathsf{Y}}$. In particular, there is a one-to-one correspondence 
between their rate-distortion curves and optimal codebooks at every distortion 
level $D$.
\end{theorem}

\begin{proof}
The proof proceeds in three steps. First, we reduce the WMSE problem for 
$\boldsymbol{\mathsf{W}}$ to a standard MSE problem for a transformed source 
by introducing a linear change of variables. Second, we diagonalize the 
covariance of the transformed source via eigen-decomposition to obtain 
independent scalar Gaussian components, and apply reverse water-filling to 
solve the resulting rate-distortion problem. Third, we pull the optimal 
reconstruction back to $\boldsymbol{\mathsf{W}}$ by inverting the change of 
variables, treating each structural case of $\boldsymbol{X}$ separately.

Throughout Steps~1 and~2, we work with the vectorized representations
\[
    \boldsymbol{\mathsf{w}} \;\coloneqq\; \operatorname{vec}(\boldsymbol{\mathsf{W}}) 
    \;\in\; \mathbb{R}^{d_{\mathrm{out}} \cdot d_{\mathrm{in}}},
    \qquad
    \boldsymbol{\mathsf{y}} \;\coloneqq\; 
    \operatorname{vec}(\boldsymbol{\mathsf{W}}\boldsymbol{X}) 
    \;\in\; \mathbb{R}^{d_{\mathrm{out}} \cdot n},
\]
which allow us to compute the covariance structure needed for eigen-decomposition. We return to matrix form in Step~3 for the pullback.

\textbf{Step 1: Reduction to standard MSE.} Let $\boldsymbol{\mathsf{W}} \sim \mathcal{N}(0, \sigma_w^2 \boldsymbol{I})$ be the zero-mean
i.i.d.\ Gaussian random matrix and $\boldsymbol{X}$ the deterministic matrix from
the theorem statement, with weighted mean squared error distortion (WMSE)
\[
  d_{\mathrm{WMSE}}(\boldsymbol{\mathsf{W}}, \widehat{\boldsymbol{\mathsf{W}}})
  = \mathbb{E}\left[ \left\| (\boldsymbol{\mathsf{W}} - \widehat{\boldsymbol{\mathsf{W}}}) \boldsymbol{X}
  \right\|_F^2 \right].
\]
Introduce the linear change of variables
\[
    \boldsymbol{\mathsf{y}}
    \;\coloneqq\;
    \bigl(\boldsymbol{X}^\top \otimes \boldsymbol{I}_{d_{\mathrm{out}}}\bigr)\,\boldsymbol{\mathsf{w}},
    \qquad
    \widehat{\boldsymbol{\mathsf{y}}}
    \;\coloneqq\;
    \bigl(\boldsymbol{X}^\top \otimes \boldsymbol{I}_{d_{\mathrm{out}}}\bigr)\,\widehat{\boldsymbol{\mathsf{w}}},
\]
where $\otimes$ is the Kronecker product, so that
\[
    d_{\mathrm{WMSE}}\!\left(\boldsymbol{\mathsf{W}},\widehat{\boldsymbol{\mathsf{W}}}\right)
    = \mathbb{E}\!\left[ \| \boldsymbol{\mathsf{y}} - \widehat{\boldsymbol{\mathsf{y}}}
        \|_2^2 \right].
\]
Since $\boldsymbol{\mathsf{w}}$ is zero-mean i.i.d.\ Gaussian with $\boldsymbol{\Sigma}_{\mathsf{w}} = \sigma_w^2 \boldsymbol{I}$,
$\boldsymbol{\mathsf{y}}$ is zero-mean Gaussian with covariance
\[
    \boldsymbol{\Sigma}_{\mathsf{y}}
    \;=\;
    \bigl(\boldsymbol{X}^\top \otimes \boldsymbol{I}_{d_{\mathrm{out}}}\bigr)\,
    \boldsymbol{\Sigma}_{\mathsf{w}}\,
    \bigl(\boldsymbol{X} \otimes \boldsymbol{I}_{d_{\mathrm{out}}}\bigr)
    \;=\;
    \bigl(\sigma_w^2\,\boldsymbol{X}^\top \boldsymbol{X}\bigr) \otimes \boldsymbol{I}_{d_{\mathrm{out}}}.
\]

Thus, coding $\boldsymbol{\mathsf{w}}$ under $d_{\mathrm{WMSE}}$ at rate $R$ and distortion $D$ 
is equivalent to coding $\boldsymbol{\mathsf{y}}$ under standard MSE at the same rate and distortion. It, therefore, suffices to solve the standard Gaussian MSE problem for $\boldsymbol{\mathsf{y}}$.

\textbf{Step 2: Eigen-decomposition and reverse water-filling.}
However, since $\boldsymbol{\mathsf{y}}$ is a correlated Gaussian vector, we cannot apply 
water-filling directly. The eigen-decomposition of $\boldsymbol{\Sigma}_{\mathsf{y}}$
rotates into a basis of independent scalar Gaussians, to which the classical 
reverse water-filling solution applies. Let
\[
\boldsymbol{\Sigma}_{\mathsf{y}}
=
\boldsymbol{U}\,\mathrm{diag}(\lambda_1,\lambda_2,\ldots)\,\boldsymbol{U}^\top
\]
be the eigen-decomposition of the covariance of $\boldsymbol{\mathsf{y}}$.
Expanding in this orthonormal basis,
\[
\boldsymbol{\mathsf{y}}
= \sum_{i} \mathsf{v}_i\,\boldsymbol{u}_i,
\qquad
\widehat{\boldsymbol{\mathsf{y}}}
= \sum_{i} \widehat{\mathsf{v}}_i\,\boldsymbol{u}_i,
\]
where the scalar coefficients $\mathsf{v}_i$ are independent Gaussians with 
variances $\lambda_i$. The optimal test channel for each scalar component is additive Gaussian with a refinement $\mathsf{z}_i$,
\[
{\mathsf{v}}_i = \widehat{\mathsf{v}}_i + \mathsf{z}_i,
\qquad \mathsf{z}_i \sim \mathcal{N}(0, D_i),
\]
achieving per-component distortion $D_i \le \lambda_i$.

\medskip
Since $\boldsymbol{U}$ is orthonormal, the $\ell_2$ norm is preserved under 
the change of basis, and the total MSE distortion of $\boldsymbol{\mathsf{y}}$ decomposes as a sum over components,
\[
D
=
\mathbb{E}\!\left[
  \|\boldsymbol{\mathsf{y}} - \widehat{\boldsymbol{\mathsf{y}}}\|_2^{2}
\right]
=
\sum_{i}
\mathbb{E}\!\left[
  (\mathsf{v}_i - \widehat{\mathsf{v}}_i)^2
\right]
=
\sum_{i} D_i,
\qquad
D_i \le \lambda_i.
\]

and similarly the total rate decomposes as a sum of scalar rates. The 
rate-distortion function for $\boldsymbol{\mathsf{y}}$ under MSE is therefore
\[
    R(D)
    \;=\;
    \min_{D_i}
    \;
    \frac{1}{2}
    \sum_{i}
    \log\!\left(\frac{\lambda_i}{D_i}\right),
    \qquad
    \text{subject to } 
    0 \le D_i \le \lambda_i, \quad \sum_i D_i \le D.
\]
This constrained minimization is solved by reverse water-filling, as illustrated 
in Figure~\ref{fig:def_waterfilling}. Imagine 
pouring water up to a level $\mu$ over a landscape where each coordinate 
$i$ has height $\lambda_i$. The optimal per-component distortion is $
    D_i = \min(\mu, \lambda_i),$ where the water level $\mu$ is chosen so that $\sum_i D_i = D$. This gives 
two regimes for each component:
\begin{itemize}
    \item \textbf{Below the water level} ($\lambda_i \le \mu$): the component 
    is fully submerged, so $D_i = \lambda_i$ and no rate is allocated. The 
    source is too weak to be worth encoding.
    \item \textbf{Above the water level} ($\lambda_i > \mu$): the component 
    is only partially submerged, so $D_i = \mu$ and rate is allocated 
    proportional to $\log_2(\lambda_i / \mu)$. Higher-variance components 
    receive more bits.
\end{itemize}
The total distortion includes contributions from both active and inactive 
components,
\[
    D \;=\; \sum_{i\,:\,\lambda_i \le \mu} \lambda_i
    \;+\;
    \sum_{i\,:\,\lambda_i > \mu} \mu,
\]
while only the active components $\{i : \lambda_i > \mu\}$ contribute to 
the rate,
\[
    R(D) \;=\; \frac{1}{2}\sum_{i\,:\,\lambda_i > \mu}
    \log_2\!\left(\frac{\lambda_i}{\mu}\right).
\]

\begin{figure}[H]
\centering
\begin{tikzpicture}[scale=0.5]

\pgfmathsetmacro{\waterlevel}{4}
\def\barwidth{1.2}

\tikzset{inactive/.style={pattern=north east lines, pattern color=green!50!black}}

\def\xa{0.0}
\fill[blue!10]  (\xa,0)             rectangle (\xa+\barwidth,4.5);
\fill[blue!30]  (\xa,\waterlevel)   rectangle (\xa+\barwidth,4.5);
\fill[green!20] (\xa,0)             rectangle (\xa+\barwidth,\waterlevel);
\draw[blue!60, thick] (\xa,0) rectangle (\xa+\barwidth,4.5);
\node[below, font=\normalsize] at (\xa+\barwidth/2, 0) {$1$};
\node[above, font=\normalsize, blue!80] at (\xa+\barwidth/2, 4.5) {$\lambda_1$};

\def\xb{1.6}
\fill[blue!10]  (\xb,0)           rectangle (\xb+\barwidth,6.5);
\fill[blue!30]  (\xb,\waterlevel) rectangle (\xb+\barwidth,6.5);
\fill[green!20] (\xb,0)           rectangle (\xb+\barwidth,\waterlevel);
\draw[blue!60, thick] (\xb,0) rectangle (\xb+\barwidth,6.5);
\node[below, font=\normalsize] at (\xb+\barwidth/2, 0) {$2$};
\node[above, font=\normalsize, blue!80] at (\xb+\barwidth/2, 6.5) {$\lambda_2$};

\def\xc{3.2}
\fill[green!20] (\xc,0) rectangle (\xc+\barwidth,3);
\path[inactive] (\xc,0) rectangle (\xc+\barwidth,3);
\draw[green!60!black, thick] (\xc,0) rectangle (\xc+\barwidth,3);
\node[below, font=\normalsize] at (\xc+\barwidth/2, 0) {$3$};
\node[above, font=\normalsize, green!50!black] at (\xc+\barwidth/2, 3) {$\lambda_3$};

\def\xd{4.8}
\fill[blue!10]  (\xd,0)           rectangle (\xd+\barwidth,5.5);
\fill[blue!30]  (\xd,\waterlevel) rectangle (\xd+\barwidth,5.5);
\fill[green!20] (\xd,0)           rectangle (\xd+\barwidth,\waterlevel);
\draw[blue!60, thick] (\xd,0) rectangle (\xd+\barwidth,5.5);
\node[below, font=\normalsize] at (\xd+\barwidth/2, 0) {$4$};
\node[above, font=\normalsize, blue!80] at (\xd+\barwidth/2, 5.5) {$\lambda_4$};

\def\xe{6.4}
\fill[green!20] (\xe,0) rectangle (\xe+\barwidth,2);
\path[inactive] (\xe,0) rectangle (\xe+\barwidth,2);
\draw[green!60!black, thick] (\xe,0) rectangle (\xe+\barwidth,2);
\node[below, font=\normalsize] at (\xe+\barwidth/2, 0) {$5$};
\node[above, font=\normalsize, green!50!black] at (\xe+\barwidth/2, 2) {$\lambda_5$};

\def\xf{8.0}
\fill[green!20] (\xf,0) rectangle (\xf+\barwidth,3);
\path[inactive] (\xf,0) rectangle (\xf+\barwidth,3);
\draw[green!60!black, thick] (\xf,0) rectangle (\xf+\barwidth,3);
\node[below, font=\normalsize] at (\xf+\barwidth/2, 0) {$6$};
\node[above, font=\normalsize, green!50!black] at (\xf+\barwidth/2, 3) {$\lambda_6$};

\def\xg{9.6}
\fill[blue!10]  (\xg,0)           rectangle (\xg+\barwidth,7.5);
\fill[blue!30]  (\xg,\waterlevel) rectangle (\xg+\barwidth,7.5);
\fill[green!20] (\xg,0)           rectangle (\xg+\barwidth,\waterlevel);
\draw[blue!60, thick] (\xg,0) rectangle (\xg+\barwidth,7.5);
\node[below, font=\normalsize] at (\xg+\barwidth/2, 0) {$7$};
\node[above, font=\normalsize, blue!80] at (\xg+\barwidth/2, 7.5) {$\lambda_7$};

\draw[decorate, decoration={brace, amplitude=5pt, mirror}, green!60!black]
    (\xg+\barwidth+0.1, 0) -- (\xg+\barwidth+0.1, \waterlevel)
    node[midway, right=6pt, font=\normalsize, green!50!black] {$D_7{=}\mu$};

\draw[decorate, decoration={brace, amplitude=5pt, mirror}, blue!60]
    (\xg+\barwidth+0.1, \waterlevel) -- (\xg+\barwidth+0.1, 7.5)
    node[midway, right=6pt, font=\normalsize, blue!80]
    {$\frac{1}{2}\log_2\!\frac{\lambda_7}{\mu}$};

\draw[red!70!black, thick, dashed] (-0.3,\waterlevel) -- (11.3,\waterlevel);
\node[left, font=\normalsize\bfseries, red!70!black] at (-0.3,\waterlevel) {$\mu$};

\draw[->, thick] (-0.3,0) -- (11.8,0) node[right, font=\normalsize] {$i$};
\draw[->, thick] (-0.3,0) -- (-0.3,8.8) node[above, font=\normalsize] {$\lambda_i$};

\end{tikzpicture}

\vspace{0.3em}

\begin{tikzpicture}[scale=1.0]
  \fill[blue!30]  (1.0,0.4) rectangle (1.5,0.8);
  \draw[blue!60]  (1.0,0.4) rectangle (1.5,0.8);
  \node[right, font=\normalsize] at (1.55,0.6)
      {Rate:\; $\frac{1}{2}\log_2(\lambda_i/\mu)$};
  \fill[green!20] (5.2,0.4) rectangle (5.7,0.8);
  \draw[green!60!black] (5.2,0.4) rectangle (5.7,0.8);
  \node[right, font=\normalsize] at (5.75,0.6)
      {Distortion:\; $D_i = \mu$ \;(active),\; $D_i = \lambda_i$ \;(inactive)};
  \fill[white]  (1.0,-0.2) rectangle (1.5,0.2);
  \draw[black]  (1.0,-0.2) rectangle (1.5,0.2);
  \node[right, font=\normalsize] at (1.55,0.0)
      {Active:\; $\lambda_i > \mu$};
  \fill[white] (5.2,-0.2) rectangle (5.7,0.2);
  \path[pattern=north east lines, pattern color=green!50!black]
      (5.2,-0.2) rectangle (5.7,0.2);
  \draw[black] (5.2,-0.2) rectangle (5.7,0.2);
  \node[right, font=\normalsize] at (5.75,0.0)
      {Inactive:\; $\lambda_i \le \mu$};
\end{tikzpicture}

\caption{Reverse water-filling method}
\label{fig:def_waterfilling}
\end{figure}

As the distortion budget $D$ decreases, the water level $\mu$ drops, 
activating more components and allocating rate to progressively weaker 
directions of $\boldsymbol{\mathsf{y}}$.

\textbf{Step 3: Pullback to $\boldsymbol{\mathsf{W}}$.}
Since the change of variables $\boldsymbol{\mathsf{w}} \mapsto \boldsymbol{\mathsf{y}}$ 
is linear, with an explicit inverse or pseudo-inverse constructed in each of the three 
cases below, the rate-distortion function derived above for 
$\boldsymbol{\mathsf{y}}$ under standard MSE is equivalently the rate-distortion 
function for $\boldsymbol{\mathsf{w}}$ under $d_{\mathrm{WMSE}}$. Since 
$\boldsymbol{\mathsf{w}} = \operatorname{vec}(\boldsymbol{\mathsf{W}})$ and 
$\boldsymbol{\mathsf{y}} = \operatorname{vec}(\boldsymbol{\mathsf{Y}})$ are simply 
vectorized representations of the underlying matrices, we can equivalently write
\[
    \boldsymbol{\mathsf{Y}} \;=\; \boldsymbol{\mathsf{W}}\boldsymbol{X},
    \qquad
    d_{\mathrm{WMSE}}(\boldsymbol{\mathsf{W}},\widehat{\boldsymbol{\mathsf{W}}})
    \;=\;
    \mathbb{E}\!\left[
    \bigl\|(\boldsymbol{\mathsf{W}} - \widehat{\boldsymbol{\mathsf{W}}})\boldsymbol{X}
    \bigr\|_F^2
    \right]
    \;=\;
    \mathbb{E}\!\left[
    \bigl\|\boldsymbol{\mathsf{Y}} - \widehat{\boldsymbol{\mathsf{Y}}}
    \bigr\|_F^2
    \right],
\]
confirming that the distortion equivalence holds in matrix form as well.
The optimal codebook for $\boldsymbol{\mathsf{W}}$ is then recovered from the 
optimal $\widehat{\boldsymbol{\mathsf{Y}}}$ by inverting the relation
$\boldsymbol{\mathsf{Y}} = \boldsymbol{\mathsf{W}}\boldsymbol{X}$, whose precise 
form depends on the structure of $\boldsymbol{X}$ and is treated separately 
in each of the three cases below.

\paragraph{\textit{Case 1}: $\boldsymbol{X}$ is square and invertible ($d_{\mathrm{in}}=n$).}
Set $\widehat{\boldsymbol{\mathsf{W}}} = 
\widehat{\boldsymbol{\mathsf{Y}}}\,\boldsymbol{X}^{-1}$.
Then, $\widehat{\boldsymbol{\mathsf{W}}}\boldsymbol{X} = 
\widehat{\boldsymbol{\mathsf{Y}}}$, and therefore
\[
    \mathbb{E}\!\left[\|(\boldsymbol{\mathsf{W}} - 
    \widehat{\boldsymbol{\mathsf{W}}})\boldsymbol{X}\|_F^2\right]
    \;=\;
    \mathbb{E}\!\left[\|\boldsymbol{\mathsf{Y}} - 
    \widehat{\boldsymbol{\mathsf{Y}}}\|_F^2\right].
\]

\paragraph{\textit{Case 2}: $\boldsymbol{X}$ is tall and $\boldsymbol{X}^{\top}\boldsymbol{X}$ is invertible ($d_{\mathrm{in}}>n$).}
Let $\boldsymbol{X}^{+} = (\boldsymbol{X}^{\top}\boldsymbol{X})^{-1}
\boldsymbol{X}^{\top}$ denote the left inverse of $\boldsymbol{X}$, which 
satisfies $\boldsymbol{X}^{+}\boldsymbol{X} = \boldsymbol{I}_n$. Set 
$\widehat{\boldsymbol{\mathsf{W}}} = \widehat{\boldsymbol{\mathsf{Y}}}\,
\boldsymbol{X}^{+}$. Then,
\[
    \widehat{\boldsymbol{\mathsf{W}}}\boldsymbol{X}
    \;=\; \widehat{\boldsymbol{\mathsf{Y}}}\,(\boldsymbol{X}^{+}\boldsymbol{X})
    \;=\; \widehat{\boldsymbol{\mathsf{Y}}},
\]
and therefore
\[
    \mathbb{E}\!\left[\|(\boldsymbol{\mathsf{W}} - 
    \widehat{\boldsymbol{\mathsf{W}}})\boldsymbol{X}\|_F^2\right]
    \;=\;
    \mathbb{E}\!\left[\|\boldsymbol{\mathsf{Y}} - 
    \widehat{\boldsymbol{\mathsf{Y}}}\|_F^2\right].
\]

\paragraph{\textit{Case 3}: $\boldsymbol{X}$ is wide and $\boldsymbol{X}\boldsymbol{X}^{\top}$ invertible ($d_{\mathrm{in}}<n$).}
Let
$\boldsymbol{X} = \boldsymbol{U}_d\,\boldsymbol{S}_d\,\boldsymbol{V}_d^{\top}$
be the compact singular value decomposition (SVD) of $\boldsymbol{X}$, where
\[
\begin{aligned}
\boldsymbol{U}_d &\in \mathbb{R}^{d_{\mathrm{in}}\times d_{\mathrm{in}}}
      && \text{is orthogonal},\\
\boldsymbol{S}_d &= \mathrm{diag}(\sigma_1,\dots,\sigma_{d_{\mathrm{in}}})
      && \in \mathbb{R}^{d_{\mathrm{in}}\times d_{\mathrm{in}}},\\
\boldsymbol{V}_d &\in \mathbb{R}^{n\times d_{\mathrm{in}}},
      && \boldsymbol{V}_d^{\top}\boldsymbol{V}_d = \boldsymbol{I}_{d_{\mathrm{in}}}.
\end{aligned}
\]

Since $\boldsymbol{X}$ is wide and cannot be inverted directly, we introduce 
the square, invertible surrogate
\[
\widetilde{\boldsymbol{X}} := \boldsymbol{U}_d\,\boldsymbol{S}_d \in \mathbb{R}^{d_{\mathrm{in}}\times d_{\mathrm{in}}},
\qquad
\widetilde{\boldsymbol{X}}^{-1}=\boldsymbol{S}_d^{-1}\boldsymbol{U}_d^{\top},
\]
so that
$
\boldsymbol{X}=\widetilde{\boldsymbol{X}}\,\boldsymbol{V}_d^{\top}.$ Rather than working with $\boldsymbol{\mathsf{Y}} = \boldsymbol{\mathsf{W}}\boldsymbol{X} 
\in \mathbb{R}^{d_{\mathrm{out}} \times n}$, which inherits the non-invertibility 
of $\boldsymbol{X}$, we instead work with the compressed surrogate
\[
    \widetilde{\boldsymbol{\mathsf{Y}}} \;\coloneqq\; 
    \boldsymbol{\mathsf{W}}\,\widetilde{\boldsymbol{X}} 
    \;\in\; \mathbb{R}^{d_{\mathrm{out}}\times d_{\mathrm{in}}}.
\]
Using $\boldsymbol{X}=\widetilde{\boldsymbol{X}}\boldsymbol{V}_d^{\top}$, the original outputs are recovered as
\[ \boldsymbol{\mathsf{Y}} = 
\boldsymbol{\mathsf{W}}\boldsymbol{X}
=\boldsymbol{\mathsf{W}}\,\widetilde{\boldsymbol{X}}\,\boldsymbol{V}_d^{\top}
=\widetilde{\boldsymbol{\mathsf{Y}}}\,\boldsymbol{V}_d^{\top}.
\]

and since $\boldsymbol{V}_d$ has orthonormal columns, right multiplication 
by $\boldsymbol{V}_d^{\top}$ preserves the Frobenius norm, giving
\[
    \mathbb{E}\!\left[\|\boldsymbol{\mathsf{Y}} - 
    \widehat{\boldsymbol{\mathsf{Y}}}\|_F^2\right]
    \;=\;
    \mathbb{E}\!\left[\|\widetilde{\boldsymbol{\mathsf{Y}}} - 
    \widehat{\widetilde{\boldsymbol{\mathsf{Y}}}}\|_F^2\right].
\]
It therefore suffices to solve the rate-distortion problem for 
$\widetilde{\boldsymbol{\mathsf{Y}}}$. Since $\widetilde{\boldsymbol{X}}$ is 
square and invertible, this reduces to Case~1, and the same reverse 
water-filling applied to $\boldsymbol{\Sigma}_{\mathsf{y}}$ yields the 
optimal reconstruction $\widehat{\widetilde{\boldsymbol{\mathsf{Y}}}}$.
The reconstruction $\widehat{\boldsymbol{\mathsf{W}}}$ is then recovered via
\[
    \widehat{\boldsymbol{\mathsf{W}}} \;=\; 
    \widehat{\widetilde{\boldsymbol{\mathsf{Y}}}}\,\widetilde{\boldsymbol{X}}^{-1}.
\]
To verify that this achieves the correct distortion on the original samples,
note that the forward map gives
$\widehat{\boldsymbol{\mathsf{W}}}\boldsymbol{X} = 
\widehat{\widetilde{\boldsymbol{\mathsf{Y}}}}\,\boldsymbol{V}_d^{\top}$. Thus, the output error satisfies
$(\boldsymbol{\mathsf{W}}-\widehat{\boldsymbol{\mathsf{W}}})\boldsymbol{X}
= (\widetilde{\boldsymbol{\mathsf{Y}}}-\widehat{\widetilde{\boldsymbol{\mathsf{Y}}}})\,
\boldsymbol{V}_d^{\top}$.
Using the norm equivalence established above, we conclude
\[
    \mathbb{E}\!\left[\|(\boldsymbol{\mathsf{W}}-\widehat{\boldsymbol{\mathsf{W}}})
    \boldsymbol{X}\|_F^2\right]
    \;=\;
    \mathbb{E}\!\left[\|\widetilde{\boldsymbol{\mathsf{Y}}}-
    \widehat{\widetilde{\boldsymbol{\mathsf{Y}}}}\|_F^2\right]
    \;=\;
    \mathbb{E}\!\left[\|\boldsymbol{\mathsf{Y}}-\widehat{\boldsymbol{\mathsf{Y}}}
    \|_F^2\right].
\]

In all three cases, the recovery $\widehat{\boldsymbol{\mathsf{W}}}$ is 
constructed so that
\[
    \mathbb{E}\!\left[\|(\boldsymbol{\mathsf{W}} - \widehat{\boldsymbol{\mathsf{W}}})
    \boldsymbol{X}\|_F^2\right]
    \;=\;
    \mathbb{E}\!\left[\|\boldsymbol{\mathsf{Y}} - 
    \widehat{\boldsymbol{\mathsf{Y}}}\|_F^2\right],
\]
confirming that the WMSE distortion on $\boldsymbol{\mathsf{W}}$ equals the 
standard MSE distortion on $\boldsymbol{\mathsf{Y}}$ in all cases.

Hence, this theorem shows:
\begin{enumerate}
    \item \textbf{Distortion equivalence.} For any reconstruction 
    $\widehat{\boldsymbol{\mathsf{W}}}$, setting 
    $\widehat{\boldsymbol{\mathsf{Y}}} = \widehat{\boldsymbol{\mathsf{W}}}\boldsymbol{X}$ gives
    \[
        d_{\mathrm{WMSE}}(\boldsymbol{\mathsf{W}}, \widehat{\boldsymbol{\mathsf{W}}})
        \;=\;
        \mathbb{E}\!\left[
        \bigl\|\boldsymbol{\mathsf{Y}} - \widehat{\boldsymbol{\mathsf{Y}}}
        \bigr\|_F^2
        \right].
    \]
    \item \textbf{Rate-distortion equivalence.} The rate-distortion function 
    for $(\boldsymbol{\mathsf{W}}, d_{\mathrm{WMSE}})$ equals the standard 
    Gaussian MSE rate-distortion function for $\boldsymbol{\mathsf{Y}}$, given 
    by reverse water-filling over the eigenvalues $\{\lambda_i\}$ of 
    $\boldsymbol{\Sigma}_{\mathsf{y}}$:
    \[
        R(D)
        \;=\;
        \frac{1}{2}\sum_{i\,:\,\lambda_i > \mu}
        \log_2\!\left(\frac{\lambda_i}{\mu}\right),
    \]
    where $\mu$ is chosen so that $\sum_i \min(\mu, \lambda_i) = D$.
    \item \textbf{Codebook recovery.} The optimal reconstruction 
    $\widehat{\boldsymbol{\mathsf{W}}}$ is recovered from the optimal 
    $\widehat{\boldsymbol{\mathsf{Y}}}$ by inverting 
    $\boldsymbol{\mathsf{Y}} = \boldsymbol{\mathsf{W}}\boldsymbol{X}$. 
    The precise form depends on the structure of $\boldsymbol{X}$:
    \begin{itemize}
        \item \textbf{Square} ($d_{\mathrm{in}} = n$, $\boldsymbol{X}$ invertible$)$: 
        $\widehat{\boldsymbol{\mathsf{W}}} = 
        \widehat{\boldsymbol{\mathsf{Y}}}\,\boldsymbol{X}^{-1}$.
        \item \textbf{Tall} ($d_{\mathrm{in}} > n$), $\boldsymbol{X}^\top\boldsymbol{X}$ 
        invertible$)$: 
        $\widehat{\boldsymbol{\mathsf{W}}} = \widehat{\boldsymbol{\mathsf{Y}}}\,
        (\boldsymbol{X}^\top\boldsymbol{X})^{-1}\boldsymbol{X}^\top$.
        \item \textbf{Wide} ($d_{\mathrm{in}} < n$, via compact SVD 
        $\boldsymbol{X} = \boldsymbol{U}_d\boldsymbol{S}_d\boldsymbol{V}_d^\top)$: 
        $\widehat{\boldsymbol{\mathsf{W}}} = \widehat{\widetilde{\boldsymbol{\mathsf{Y}}}}\,
        \widetilde{\boldsymbol{X}}^{-1}$, where 
        $\widetilde{\boldsymbol{X}} \coloneqq \boldsymbol{U}_d\boldsymbol{S}_d$.
    \end{itemize}
\end{enumerate}

Moreover, in Case~2, where 
$\boldsymbol{\mathsf{w}} \mapsto \boldsymbol{\mathsf{y}}$ is many-to-one, the 
optimal test channel allocates zero bits to the null space of this map, so the 
non-invertibility does not affect $d_{\mathrm{WMSE}}$. In all three cases, the 
recovery map from $\widehat{\boldsymbol{\mathsf{Y}}}$ to
$\widehat{\boldsymbol{\mathsf{W}}}$ is linear, determined by $\boldsymbol{X}$ alone, and
preserves the distortion. We refer to this as the \textit{pullback}.

\end{proof}

\subsection{Successive Refinability of the Transformed Source}

Theorem~\ref{thm:Sakrison} established that the rate-distortion function for
$\boldsymbol{\mathsf{W}}$ under $d_{\mathrm{WMSE}}$ coincides with the standard MSE
rate-distortion function for the transformed source $\boldsymbol{\mathsf{y}}$ at any single
distortion level. We now show that $\boldsymbol{\mathsf{y}}$ is itself successively refinable
under standard MSE.

\begin{theorem}[Correlated Gaussian under MSE is successively refinable]
\label{thm:WMSE-SR}
Let $\boldsymbol{\mathsf{y}}$ be zero-mean Gaussian with covariance
$\boldsymbol{\Sigma}_{\boldsymbol{y}} = \boldsymbol{U}\operatorname{diag}(\lambda_1, \lambda_2, \dots)\boldsymbol{U}^\top$,
as established in the proof of Theorem~\ref{thm:Sakrison}. Then $\boldsymbol{\mathsf{y}}$ is
successively refinable under the mean squared error distortion
measure. That is, for any nested distortion levels
$D_1 > D_2 > \cdots > D_K > 0$, there exists a single embedded
bitstream whose $k$-th prefix achieves the optimal rate-distortion
pair $(R(D_k), D_k)$ simultaneously for all $k$, at the same total
rate as $K$ independent codes each optimally designed for its own
distortion level. Moreover, the code is realized on blocks of $\boldsymbol{\mathsf{W}}$, hence, its stage-$k$ reconstructions satisfy the
same guarantees under $d_{\mathrm{WMSE}}$.
\end{theorem}

\begin{proof}

We treat each realization of the full matrix $\boldsymbol{\mathsf{W}} \in \mathcal{W} \triangleq \mathbb{R}^{d_{\mathrm{out}}\times d_{\mathrm{in}}}$ as an individual source symbol. A block of length $N$ is an i.i.d.\ sequence $\boldsymbol{\mathsf{W}}^N = (\boldsymbol{\mathsf{W}}_1, \dots, \boldsymbol{\mathsf{W}}_N)$, where $\boldsymbol{X}$ remains fixed across all $N$ realizations, and we consider $N \to \infty$. Via the transformation in Theorem~\ref{thm:Sakrison}, each matrix $\boldsymbol{\mathsf{W}}$ is mapped to a vector $\boldsymbol{\mathsf{v}} \in \mathbb{R}^m$ with $m = d_{\mathrm{out}} \min(d_{\mathrm{in}}, n)$. Consequently, the matrix source is an instance of the vector formulation in Definition~\ref{def:SR-vector} with symbol alphabet $\mathcal{W}$. The codebooks act directly on the eigen-basis coordinates $\mathsf{v}_i$, rather than on the original entries of $\boldsymbol{\mathsf{W}}$.

\medskip
By Theorem~\ref{thm:Sakrison}, $\boldsymbol{\mathsf{y}}$ is zero-mean
Gaussian with covariance $\boldsymbol{\Sigma}_{\mathsf{y}} =
(\sigma_w^2 \boldsymbol{X}^\top \boldsymbol{X}) \otimes
\boldsymbol{I}_{d_{\mathrm{out}}}$. Let
\[
    \boldsymbol{\Sigma}_{\mathsf{y}}
    \;=\;
    \boldsymbol{U}\,\mathrm{diag}(\lambda_1, \lambda_2, \ldots)\,
    \boldsymbol{U}^\top
\]
be its eigen-decomposition. Transforming to the eigen-basis via $\boldsymbol{\mathsf{v}} = \boldsymbol{U}^\top \boldsymbol{\mathsf{y}}$ yields independent scalar Gaussian components $\mathsf{v}_i \sim \mathcal{N}(0, \lambda_i)$. Since the
components are independent, it suffices to show successive
refinability for each scalar $\mathsf{v}_i$ separately, then sum
over components.

\medskip
\textit{Scalar test channel construction.}
For each component $\mathsf{v}_i$, superscripts denote the refinement
stage, while subscripts index the
source component. For example, $\widehat{\mathsf{v}}_i^{(k)}$ is the reconstruction of
component $i$ after stage $k$, and $\mathsf{z}_i^{(k)}$ is the
Gaussian increment added at stage $k$. Define nested distortion levels $D_i^{(1)} > D_i^{(2)} > 0$. Then, define mutually independent Gaussian increments
\[
    \mathsf{z}_i^{(1)} \sim \mathcal{N}\!\left(0,\, \lambda_i -
    D_i^{(1)}\right),
    \quad
    \mathsf{z}_i^{(2)} \sim \mathcal{N}\!\left(0,\, D_i^{(1)} -
    D_i^{(2)}\right),
    \quad
    \mathsf{z}_i \sim \mathcal{N}\!\left(0,\, D_i^{(2)}\right),
\]
and construct successive reconstructions\[
    \widehat{\mathsf{v}}_i^{(0)}
    \;=\; 0,
    \qquad
    \widehat{\mathsf{v}}_i^{(1)}
    \;=\; \widehat{\mathsf{v}}_i^{(0)} + \mathsf{z}_i^{(1)},
    \qquad
    \widehat{\mathsf{v}}_i^{(2)}
    \;=\; \widehat{\mathsf{v}}_i^{(1)} + \mathsf{z}_i^{(2)},
    \qquad
    \mathsf{v}_i
    \;=\; \widehat{\mathsf{v}}_i^{(2)} + \mathsf{z}_i.
\]
Each stage $k$ refines the previous estimate by adding an
independent Gaussian increment of variance $D_i^{(k-1)} - D_i^{(k)}$.
By construction, $\mathbb{E}\bigl[(\mathsf{v}_i -
\widehat{\mathsf{v}}_i^{(k)})^2\bigr] = D_i^{(k)}$ for $k = 1, 2$,
and since the increments are mutually independent the Markov chain holds by construction,
\[
    \widehat{\mathsf{v}}_i^{(1)}
    \;\to\; \widehat{\mathsf{v}}_i^{(2)}
    \;\to\; \mathsf{v}_i
\]

\medskip
\textit{No rate penalty.}
The rate $\Delta R_i^{(1)}$ is the rate needed to encode
$\mathsf{z}_i^{(1)}$, i.e.\ the information required to describe the
first-stage reconstruction $\widehat{\mathsf{v}}_i^{(1)}$ from
scratch. The rate $\Delta R_i^{(2)}$ is the additional rate needed to
encode $\mathsf{z}_i^{(2)}$, i.e.\ the new information required to
describe how the reconstruction improves from
$\widehat{\mathsf{v}}_i^{(1)}$ to $\widehat{\mathsf{v}}_i^{(2)}$.
Formally,
\[
    \Delta R_i^{(1)}
    \;=\; I\!\left(\mathsf{v}_i;\,\widehat{\mathsf{v}}_i^{(1)}\right)
    \;=\; \frac{1}{2}\log\frac{\lambda_i}{D_i^{(1)}},
\]
\[
    \Delta R_i^{(2)}
    \;=\; I\!\left(\mathsf{v}_i;\,\widehat{\mathsf{v}}_i^{(2)}
    \,\middle|\, \widehat{\mathsf{v}}_i^{(1)}\right)
    \;=\; h\!\left(\mathsf{v}_i \mid \widehat{\mathsf{v}}_i^{(1)}\right)
    - h\!\left(\mathsf{v}_i \mid \widehat{\mathsf{v}}_i^{(1)},
    \widehat{\mathsf{v}}_i^{(2)}\right)
    \;=\; \frac{1}{2}\log\frac{D_i^{(1)}}{D_i^{(2)}}.
\]
The total two-stage rate is equal to the one-shot rate at
$D_i^{(2)}$:
\[
    \Delta R_i^{(1)} + \Delta R_i^{(2)}
    \;=\; \frac{1}{2}\log\frac{\lambda_i}{D_i^{(1)}}
    + \frac{1}{2}\log\frac{D_i^{(1)}}{D_i^{(2)}}
    \;=\; \frac{1}{2}\log\frac{\lambda_i}{D_i^{(2)}}
    \;=\; R_i(D_i^{(2)}),
\]
confirming no rate penalty for the $i$-th component. The Markov consistency condition
\[
    I\!\left(\mathsf{v}_i;\,\widehat{\mathsf{v}}_i^{(1)}
    \,\middle|\, \widehat{\mathsf{v}}_i^{(2)}\right) = 0
\]
holds because $\mathsf{v}_i = \widehat{\mathsf{v}}_i^{(2)} + \mathsf{z}_i$, where the
final increment $\mathsf{z}_i$ is mutually independent of the earlier reconstructions
$\widehat{\mathsf{v}}_i^{(1)}$ and $\widehat{\mathsf{v}}_i^{(2)}$ by construction. Hence,
conditioned on $\widehat{\mathsf{v}}_i^{(2)}$, the residual $\mathsf{v}_i -
\widehat{\mathsf{v}}_i^{(2)} = \mathsf{z}_i$ carries no information about
$\widehat{\mathsf{v}}_i^{(1)}$, so $\mathsf{v}_i$ is conditionally independent of
$\widehat{\mathsf{v}}_i^{(1)}$ given $\widehat{\mathsf{v}}_i^{(2)}$, establishing the
Markov chain $\widehat{\mathsf{v}}_i^{(1)} \to \widehat{\mathsf{v}}_i^{(2)} \to
\mathsf{v}_i$.

The same construction extends to $K$ stages by defining $K$
independent increments of variances $\lambda_i - D_i^{(1)},\,
D_i^{(1)} - D_i^{(2)},\, \ldots,\, D_i^{(K-1)} - D_i^{(K)}$,
and the joint Markov chain
\[
    \widehat{\boldsymbol{\mathsf{y}}}^{(1)}
    \;\to\; \widehat{\boldsymbol{\mathsf{y}}}^{(2)}
    \;\to\; \cdots
    \;\to\; \widehat{\boldsymbol{\mathsf{y}}}^{(K)}
    \;\to\; \boldsymbol{\mathsf{y}}
\]
holds since the scalar Markov chains hold independently for each
component, with each stage being a strictly finer approximation of
$\boldsymbol{\mathsf{y}}$ than the previous. Moreover, Figures~\ref{fig:three_sr_waterfilling_1} and~\ref{fig:three_sr_waterfilling_2} illustrate the
three possible activation patterns: at the first stage, case~(a) starts with only one component active while the other remains inactive, case~(b) starts with both components already active, and case~(c) starts with both components inactive. In all cases the $\Delta R$
brackets confirm that the incremental rate at each stage equals
exactly $R(D_{k+1}) - R(D_k)$. We next demonstrate the operational achievability of these test channels.

\begin{figure}[H]
\centering
\tikzset{inactive/.style={pattern=north east lines, pattern color=green!50!black}}

\begin{subfigure}[t]{0.35\textwidth}
\centering
\begin{tikzpicture}[scale=0.75]
\pgfmathsetmacro{\lami}{5.5}
\pgfmathsetmacro{\lamii}{2.5}
\pgfmathsetmacro{\muone}{3.5}
\def\bw{1.2}
\fill[blue!10]  (0,0)      rectangle (\bw,\lami);
\fill[blue!30]  (0,\muone) rectangle (\bw,\lami);
\fill[green!20] (0,0)      rectangle (\bw,\muone);
\draw[blue!60, thick] (0,0) rectangle (\bw,\lami);
\node[below, font=\normalsize] at (\bw/2, 0) {$1$};
\node[above, font=\normalsize, blue!80] at (\bw/2, \lami) {$\lambda_1$};
\fill[green!20] (1.8,0) rectangle (1.8+\bw,\lamii);
\path[inactive]  (1.8,0) rectangle (1.8+\bw,\lamii);
\draw[green!60!black, thick] (1.8,0) rectangle (1.8+\bw,\lamii);
\node[below, font=\normalsize] at (1.8+\bw/2, 0) {$2$};
\node[above, font=\normalsize, green!50!black] at (1.8+\bw/2, \lamii) {$\lambda_2$};
\draw[red!70!black, thick, dashed] (-0.3,\muone) -- (3.3,\muone);
\node[left, font=\normalsize\bfseries, red!70!black] at (-0.3,\muone) {$\mu_1$};
\draw[->, thick] (-0.3,0) -- (-0.3,6.5) node[above, font=\normalsize] {$\lambda_i$};
\draw[->, thick] (-0.3,0) -- (3.5,0)    node[right, font=\normalsize] {$i$};
\node[below, font=\normalsize] at (1.5,-0.8) {Stage 1};
\end{tikzpicture}
\end{subfigure}
\hspace{0.3cm}
\begin{subfigure}[t]{0.35\textwidth}
\centering
\begin{tikzpicture}[scale=0.75]
\pgfmathsetmacro{\lami}{5.5}
\pgfmathsetmacro{\lamii}{2.5}
\pgfmathsetmacro{\muone}{3.5}
\pgfmathsetmacro{\mutwo}{1.8}
\def\bw{1.2}
\fill[blue!10]  (0,0)      rectangle (\bw,\lami);
\fill[blue!30]  (0,\mutwo) rectangle (\bw,\lami);
\fill[green!20] (0,0)      rectangle (\bw,\mutwo);
\draw[blue!60, thick] (0,0) rectangle (\bw,\lami);
\node[below, font=\normalsize] at (\bw/2, 0) {$1$};
\node[above, font=\normalsize, blue!80] at (\bw/2, \lami) {$\lambda_1$};
\draw[red!70!black, dashed, thin] (0,\muone) -- (\bw,\muone);
\draw[decorate, decoration={brace, amplitude=4pt, mirror}, purple!70!black]
    (\bw-0.1, \muone) -- (\bw-0.1, \mutwo)
    node[midway, left=1pt, font=\normalsize, purple!70!black] {$\Delta R_1$};
\fill[blue!10]  (1.8,0)      rectangle (1.8+\bw,\lamii);
\fill[blue!30]  (1.8,\mutwo) rectangle (1.8+\bw,\lamii);
\fill[green!20] (1.8,0)      rectangle (1.8+\bw,\mutwo);
\draw[blue!60, thick] (1.8,0) rectangle (1.8+\bw,\lamii);
\node[below, font=\normalsize] at (1.8+\bw/2, 0) {$2$};
\node[above, font=\normalsize, blue!80] at (1.8+\bw/2, \lamii) {$\lambda_2$};
\draw[red!70!black, dashed, thin] (1.8,\muone) -- (1.8+\bw,\muone);
\draw[decorate, decoration={brace, amplitude=4pt, mirror}, purple!70!black]
    (1.8+\bw-0.1, \lamii) -- (1.8+\bw-0.1, \mutwo)
    node[midway, left=1pt, font=\normalsize, purple!70!black] {$\Delta R_2$};
\draw[red!70!black, thick, dashed] (-0.3,\mutwo) -- (3.5,\mutwo);
\node[left, font=\normalsize\bfseries, red!70!black] at (-0.3,\mutwo) {$\mu_2$};
\node[left, font=\normalsize, red!70!black, opacity=0.6] at (-0.3,\muone) {$\mu_1$};
\draw[->, thick] (-0.3,0) -- (-0.3,6.5) node[above, font=\normalsize] {$\lambda_i$};
\draw[->, thick] (-0.3,0) -- (3.8,0)    node[right, font=\normalsize] {$i$};
\node[below, font=\normalsize] at (1.5,-0.8) {Stage 2 ($\mu_2 < \mu_1$)};
\end{tikzpicture}
\end{subfigure}

\vspace{0.5em}

\begin{tikzpicture}[scale=0.75]
    \fill[blue!30]  (0,0) rectangle (0.5,0.35);
    \draw[blue!60]  (0,0) rectangle (0.5,0.35);
    \node[right, font=\normalsize] at (0.6,0.175) {Rate};
    \fill[green!20] (2.5,0) rectangle (3.0,0.35);
    \draw[green!60!black] (2.5,0) rectangle (3.0,0.35);
    \node[right, font=\normalsize] at (3.1,0.175) {Distortion $D_i{=}\mu$};
    \fill[white] (7.8,0) rectangle (8.3,0.35);
    \path[pattern=north east lines, pattern color=green!50!black]
        (7.8,0) rectangle (8.3,0.35);
    \draw[green!60!black] (7.8,0) rectangle (8.3,0.35);
    \node[right, font=\normalsize] at (8.4,0.175) {Inactive $D_i{=}\lambda_i$};
    \draw[red!70!black, dashed, thin] (12.3,0.175) -- (12.8,0.175);
    \node[right, font=\normalsize, red!70!black] at (12.9,0.175) {$\mu_1$ reference};
\end{tikzpicture}

\caption{Successive refinement under reverse water-filling: case~(a) 
shows at Stage~1, component~1 is active ($\lambda_1 > \mu_1$) while component~2 is inactive ($\lambda_2 \leq \mu_1$). At Stage~2, the water level drops to $\mu_2 < \mu_1$, activating component~2 and refining both.}
\label{fig:three_sr_waterfilling_1}
\end{figure}

Under reverse water-filling, the per-component distortion levels
$D_i^{(k)}$ that achieve the rate-distortion function at overall
distortion $D_k$ are given explicitly by
\[
    D_i^{(k)} \;=\; \min\!\left(\lambda_i,\, \mu_k\right),
\]
where $\mu_k > 0$ is the water level chosen so that
$\sum_i D_i^{(k)} = D_k$. Nested distortion levels $D_1 > D_2 >
\cdots > D_K$ therefore correspond to a strictly decreasing sequence
of water levels $\mu_1 > \mu_2 > \cdots > \mu_K$. As $\mu$ decreases
across refinement stages, the distortion assigned to each active
coordinate $i$ (where $\lambda_i > \mu$) equals $\mu$ and therefore
decreases monotonically, while inactive coordinates (where $\lambda_i
\leq \mu$) retain distortion $D_i = \lambda_i$ and contribute zero
rate at that stage, remaining unchanged until $\mu$ drops below
$\lambda_i$ and activates them. Thus, lowering the water level
strictly refines the reconstruction: each stage produces a finer
estimate of $\boldsymbol{\mathsf{y}}$, and the test channels nest
properly across all $K$ distortion levels.

\begin{figure}[H]
\centering
\tikzset{inactive/.style={pattern=north east lines, pattern color=green!50!black}}

\begin{subfigure}[t]{0.2\textwidth}
\centering
\begin{tikzpicture}[scale=0.75]
\pgfmathsetmacro{\lami}{5.5}
\pgfmathsetmacro{\lamii}{4.0}
\pgfmathsetmacro{\muone}{2.5}
\def\bw{1.2}
\fill[blue!10]  (0,0)      rectangle (\bw,\lami);
\fill[blue!30]  (0,\muone) rectangle (\bw,\lami);
\fill[green!20] (0,0)      rectangle (\bw,\muone);
\draw[blue!60, thick] (0,0) rectangle (\bw,\lami);
\node[below, font=\normalsize] at (\bw/2, 0) {$1$};
\node[above, font=\normalsize, blue!80] at (\bw/2, \lami) {$\lambda_1$};
\fill[blue!10]  (1.8,0)      rectangle (1.8+\bw,\lamii);
\fill[blue!30]  (1.8,\muone) rectangle (1.8+\bw,\lamii);
\fill[green!20] (1.8,0)      rectangle (1.8+\bw,\muone);
\draw[blue!60, thick] (1.8,0) rectangle (1.8+\bw,\lamii);
\node[below, font=\normalsize] at (1.8+\bw/2, 0) {$2$};
\node[above, font=\normalsize, blue!80] at (1.8+\bw/2, \lamii) {$\lambda_2$};
\draw[red!70!black, thick, dashed] (-0.3,\muone) -- (3.3,\muone);
\node[left, font=\normalsize\bfseries, red!70!black] at (-0.3,\muone) {$\mu_1$};
\draw[->, thick] (-0.3,0) -- (-0.3,6.5) node[above, font=\normalsize] {$\lambda_i$};
\draw[->, thick] (-0.3,0) -- (3.5,0)    node[right, font=\normalsize] {$i$};
\node[below, font=\normalsize] at (1.5,-0.8) {Stage 1};
\end{tikzpicture}
\end{subfigure}
\hspace{0.3cm}
\begin{subfigure}[t]{0.2\textwidth}
\centering
\begin{tikzpicture}[scale=0.75]
\pgfmathsetmacro{\lami}{5.5}
\pgfmathsetmacro{\lamii}{4.0}
\pgfmathsetmacro{\muone}{2.5}
\pgfmathsetmacro{\mutwo}{1.2}
\def\bw{1.2}
\fill[blue!10]  (0,0)      rectangle (\bw,\lami);
\fill[blue!30]  (0,\mutwo) rectangle (\bw,\lami);
\fill[green!20] (0,0)      rectangle (\bw,\mutwo);
\draw[blue!60, thick] (0,0) rectangle (\bw,\lami);
\node[below, font=\normalsize] at (\bw/2, 0) {$1$};
\node[above, font=\normalsize, blue!80] at (\bw/2, \lami) {$\lambda_1$};
\draw[red!70!black, dashed, thin] (0,\muone) -- (\bw,\muone);
\draw[decorate, decoration={brace, amplitude=4pt, mirror}, purple!70!black]
    (\bw-0.1, \muone) -- (\bw-0.1, \mutwo)
    node[midway, left=1pt, font=\normalsize, purple!70!black] {$\Delta R_1$};
\fill[blue!10]  (1.8,0)      rectangle (1.8+\bw,\lamii);
\fill[blue!30]  (1.8,\mutwo) rectangle (1.8+\bw,\lamii);
\fill[green!20] (1.8,0)      rectangle (1.8+\bw,\mutwo);
\draw[blue!60, thick] (1.8,0) rectangle (1.8+\bw,\lamii);
\node[below, font=\normalsize] at (1.8+\bw/2, 0) {$2$};
\node[above, font=\normalsize, blue!80] at (1.8+\bw/2, \lamii) {$\lambda_2$};
\draw[red!70!black, dashed, thin] (1.8,\muone) -- (1.8+\bw,\muone);
\draw[decorate, decoration={brace, amplitude=4pt, mirror}, purple!70!black]
    (1.8+\bw-0.1, \muone) -- (1.8+\bw-0.1, \mutwo)
    node[midway, left=1pt, font=\normalsize, purple!70!black] {$\Delta R_2$};
\draw[red!70!black, thick, dashed] (-0.3,\mutwo) -- (3.5,\mutwo);
\node[left, font=\normalsize\bfseries, red!70!black] at (-0.3,\mutwo) {$\mu_2$};
\node[left, font=\normalsize, red!70!black, opacity=0.6] at (-0.3,\muone) {$\mu_1$};
\draw[->, thick] (-0.3,0) -- (-0.3,6.5) node[above, font=\normalsize] {$\lambda_i$};
\draw[->, thick] (-0.3,0) -- (3.8,0)    node[right, font=\normalsize] {$i$};
\node[below, font=\normalsize] at (1.5,-0.8) {Stage 2 ($\mu_2 < \mu_1$)};
\end{tikzpicture}
\end{subfigure}

\vspace{1em}

\begin{subfigure}[t]{0.2\textwidth}
\centering
\begin{tikzpicture}[scale=0.75]
\pgfmathsetmacro{\lami}{5.5}
\pgfmathsetmacro{\lamii}{2.0}
\pgfmathsetmacro{\muone}{6.5}
\def\bw{1.2}
\fill[green!20] (0,0) rectangle (\bw,\lami);
\path[inactive]  (0,0) rectangle (\bw,\lami);
\draw[green!60!black, thick] (0,0) rectangle (\bw,\lami);
\node[below, font=\normalsize] at (\bw/2, 0) {$1$};
\node[above, font=\normalsize, green!50!black] at (\bw/2, \lami) {$\lambda_1$};
\fill[green!20] (1.8,0) rectangle (1.8+\bw,\lamii);
\path[inactive]  (1.8,0) rectangle (1.8+\bw,\lamii);
\draw[green!60!black, thick] (1.8,0) rectangle (1.8+\bw,\lamii);
\node[below, font=\normalsize] at (1.8+\bw/2, 0) {$2$};
\node[above, font=\normalsize, green!50!black] at (1.8+\bw/2, \lamii) {$\lambda_2$};
\draw[red!70!black, thick, dashed] (-0.3,\muone) -- (3.3,\muone);
\node[left, font=\normalsize\bfseries, red!70!black] at (-0.3,\muone) {$\mu_1$};
\draw[->, thick] (-0.3,0) -- (-0.3,7.5) node[above, font=\normalsize] {$\lambda_i$};
\draw[->, thick] (-0.3,0) -- (3.5,0)    node[right, font=\normalsize] {$i$};
\node[below, font=\normalsize] at (1.5,-0.8) {Stage 1};
\end{tikzpicture}
\end{subfigure}
\hspace{0.3cm}
\begin{subfigure}[t]{0.2\textwidth}
\centering
\begin{tikzpicture}[scale=0.75]
\pgfmathsetmacro{\lami}{5.5}
\pgfmathsetmacro{\lamii}{2.0}
\pgfmathsetmacro{\muone}{6.5}
\pgfmathsetmacro{\mutwo}{3.0}
\def\bw{1.2}
\fill[blue!10]  (0,0)      rectangle (\bw,\lami);
\fill[blue!30]  (0,\mutwo) rectangle (\bw,\lami);
\fill[green!20] (0,0)      rectangle (\bw,\mutwo);
\draw[blue!60, thick] (0,0) rectangle (\bw,\lami);
\node[below, font=\normalsize] at (\bw/2, 0) {$1$};
\node[above, font=\normalsize, blue!80] at (\bw/2, \lami) {$\lambda_1$};
\draw[red!70!black, dashed, thin] (0,\muone) -- (\bw,\muone);
\draw[decorate, decoration={brace, amplitude=4pt, mirror}, purple!70!black]
    (\bw-0.1, \lami) -- (\bw-0.1, \mutwo)
    node[midway, left=1pt, font=\normalsize, purple!70!black] {$\Delta R_1$};
\fill[green!20] (1.8,0) rectangle (1.8+\bw,\lamii);
\path[inactive]  (1.8,0) rectangle (1.8+\bw,\lamii);
\draw[green!60!black, thick] (1.8,0) rectangle (1.8+\bw,\lamii);
\node[below, font=\normalsize] at (1.8+\bw/2, 0) {$2$};
\node[above, font=\normalsize, green!50!black] at (1.8+\bw/2, \lamii) {$\lambda_2$};
\draw[red!70!black, dashed, thin] (1.8,\muone) -- (1.8+\bw,\muone);
\draw[red!70!black, thick, dashed] (-0.3,\mutwo) -- (3.5,\mutwo);
\node[left, font=\normalsize\bfseries, red!70!black] at (-0.3,\mutwo) {$\mu_2$};
\node[left, font=\normalsize, red!70!black, opacity=0.6] at (-0.3,\muone) {$\mu_1$};
\draw[->, thick] (-0.3,0) -- (-0.3,7.5) node[above, font=\normalsize] {$\lambda_i$};
\draw[->, thick] (-0.3,0) -- (3.8,0)    node[right, font=\normalsize] {$i$};
\node[below, font=\normalsize] at (1.5,-0.8) {Stage 2 ($\mu_2 < \mu_1$)};
\end{tikzpicture}
\end{subfigure}

\vspace{1em}

\begin{tikzpicture}[scale=0.75]
    \fill[blue!30]  (0,0) rectangle (0.5,0.35);
    \draw[blue!60]  (0,0) rectangle (0.5,0.35);
    \node[right, font=\normalsize] at (0.6,0.175) {Rate};
    \fill[green!20] (2.5,0) rectangle (3.0,0.35);
    \draw[green!60!black] (2.5,0) rectangle (3.0,0.35);
    \node[right, font=\normalsize] at (3.1,0.175) {Distortion $D_i{=}\mu$};
    \fill[white] (7.8,0) rectangle (8.3,0.35);
    \path[pattern=north east lines, pattern color=green!50!black]
        (7.8,0) rectangle (8.3,0.35);
    \draw[green!60!black] (7.8,0) rectangle (8.3,0.35);
    \node[right, font=\normalsize] at (8.4,0.175) {Inactive $D_i{=}\lambda_i$};
    \draw[red!70!black, dashed, thin] (12.3,0.175) -- (12.8,0.175);
    \node[right, font=\normalsize, red!70!black] at (12.9,0.175) {$\mu_1$ reference};
\end{tikzpicture}

\caption{Successive refinement under reverse water-filling: case~(b) 
shows both components active from Stage~1 ($\lambda_1, \lambda_2 > 
\mu_1$); case~(c) shows both inactive at Stage~1 ($\lambda_1, \lambda_2 
\leq \mu_1$), with component~1 activating only at Stage~2 when $\mu_2$ 
drops below $\lambda_1$. In all cases, the incremental rates $\Delta R$ 
sum to the single-shot rate at the finest distortion level.}
\label{fig:three_sr_waterfilling_2}
\end{figure}

\paragraph{Achievability.}
Applying the change of variables of Theorem~\ref{thm:Sakrison} symbol by symbol gives
$\boldsymbol{\mathsf{v}}_t \in \mathbb{R}^{m}$ for $t = 1,\ldots,N$. Since the map is
deterministic and applied identically at every $t$, the sequence
$\boldsymbol{\mathsf{v}}^N$ is i.i.d.\ with components $\mathsf{v}_{t,i} \sim
\mathcal{N}(0,\lambda_i)$ independent across $i$, and $m$ is fixed by the layer dimensions
rather than growing with $N$.

\emph{Generation of codebooks.} Apply the codebook construction of
Theorem~\ref{thm:gaussian-vector-sr} to $\boldsymbol{\mathsf{v}}^N$, with $\lambda_i$ in
place of $\sigma_i^2$ and $D_i^{(k)} = \min(\lambda_i,\mu_k)$. For each stage $k$ and
each active component $i$, this generates $2^{N(R_{k,i}-R_{k-1,i})}$ increment sequences
of length $N$, drawn i.i.d.\ $\mathcal{N}\!\left(0,\, D_i^{(k-1)} - D_i^{(k)}\right)$.
Inactive components receive no codebook and no rate at that stage. Reveal these
codebooks to the encoder and the decoder.

\emph{Encoding.} The stage-$k$ encoder $i_N^{(k)}$ transforms each symbol of
$\boldsymbol{\mathsf{W}}^N$, then performs the stage-$k$ encoding of
Theorem~\ref{thm:gaussian-vector-sr} on $\boldsymbol{\mathsf{v}}^N$: for each active
component $i$ it searches independently for an index whose increment sequence is jointly
distortion typical with the residual of component $i$. If there is more than one such index, send the least. If there is no such index, default to the first index. The transmitted stage-$k$ index is the concatenation of these parallel indices, at incremental rate $\Delta R^{(k)}$.

\emph{Decoding.} The stage-$k$ decoder $g_N^{(k)}$ separates the received indices by
component, sums the selected increments through stage $k$ to obtain
$\widehat{\boldsymbol{\mathsf{v}}}^{(k)N}$, and applies the pullback of
Theorem~\ref{thm:Sakrison} symbol by symbol to recover
$\widehat{\boldsymbol{\mathsf{W}}}_t^{(k)}$ for $t = 1,\ldots,N$. A stage-$k$ codeword is
therefore the length-$N$ sequence $\widehat{\boldsymbol{\mathsf{W}}}^{(k)N} =
(\widehat{\boldsymbol{\mathsf{W}}}_1^{(k)}, \ldots,
\widehat{\boldsymbol{\mathsf{W}}}_N^{(k)})$, each letter of which is a
$d_{\mathrm{out}} \times d_{\mathrm{in}}$ matrix. The change of variables and the pullback are deterministic and known to encoder and
decoder.

\emph{Calculation.} The pullback preserves the
distortion, so the distortion typical
sets of the two problems coincide and typicality is applied to the unweighted MSE on
$\boldsymbol{\mathsf{v}}$ alone. The error analysis of
Theorem~\ref{thm:gaussian-vector-sr} therefore transfers unchanged: each component fails
with probability $P_{e,i}^{(k)} \le \exp\!\left(-2^{N(\varepsilon_{k,i} -
\delta_{k,i})}\right)$, and since $m$ is finite and independent of $N$, the union bound
gives $P_{e,\mathrm{total}}^{(k)} \to 0$. On the event that all parallel codebooks
succeed,
\[
    \frac{1}{N}\sum_{t=1}^{N}
    \mathbb{E}\!\left[\left\|\left(\boldsymbol{\mathsf{W}}_t -
    \widehat{\boldsymbol{\mathsf{W}}}_t^{(k)}\right)\boldsymbol{X}\right\|_F^2\right]
    \;\longrightarrow\; \sum_{i=1}^{m} D_i^{(k)} \;=\; D^{(k)},
    \qquad k = 1,\ldots,K.
\]
The encoders $i_N^{(k)}$ and decoders $g_N^{(k)}$ thus satisfy
Definition~\ref{def:SR-vector} at cumulative rates $R^{(k)}$. Hence
$\boldsymbol{\mathsf{y}}$ is successively refinable under MSE, and the code just
constructed achieves the same rates and distortions on blocks of
$\boldsymbol{\mathsf{W}}$ under $d_{\mathrm{WMSE}}$.

\end{proof}

\clearpage
\section{Full Algorithm}
\label{alg:full_alg1} 
Algorithm~\ref{alg:long} presents the complete Drop-by-Drop training procedure, expanding the condensed version given as Algorithm~\ref{alg:drop_by_drop1} in the main text. 

Beam search updates assignments $b$ using only the final-stage reconstruction error $L^{(M)}$, not the full composite loss $\mathcal{L}_{\mathrm{DbyD}}$; evaluating all $M$ partial reconstructions per candidate would multiply beam search cost by $M$. Codebook values $C$ are optimized against the full weighted objective. Block-level fine-tuning optimizes $\theta$ against the fully-reconstructed ($k=M$) output only; this is an intentional design choice made for the same complexity reason as beam search above. We do not apply the Matryoshka weighting during fine-tuning, and rely on the codebook-level loss to preserve the nested structure. Both these decisions are supported empirically by the graceful degradation in perplexity and zero-shot accuracy across $k \le M$ shown in Section~\ref{section:experiments}.

Following AQLM, both phases use early stopping. For codebook optimization, we use a relative-improvement tolerance $\tau = 0.01$, checked once per epoch. For fine-tuning, we use early-stopping after 3 consecutive non-improving epochs on validation loss.

The block size $B_H$ is selected per layer from the candidate set $\{256, 128, 64, 32\}$, taking the largest value that evenly divides $d_{\text{in}}$; if none does, $B_H = 1$ (no blocking is applied).

\begin{algorithm}[h]
\caption{Drop-by-Drop Quantization with Fast Walsh-Hadamard Transformation (FWHT)}
\label{alg:long}
\begin{algorithmic}[1]
\Require model, data
\State $X_{\text{block}} :=$ model.input\_embeddings(data)
\For{$i = 1, \dotsc, \texttt{model.num\_layers}$}
    \State block := model.get\_block($i$)
    \State $Y_{\text{block}} :=$ block($X_{\text{block}}$)
    \For{layer $\in$ linear\_layers(block)}
        \State $W :=$ layer.weight, \quad $d_{\text{in}} := \text{columns}(W)$
        \State $XX^\top := \text{accumulate\_hessian(layer, } X_{\text{block}})$
        
        \State \textcolor{teal}{Select block size $B_H$ and sample random signs $S \sim \mathcal{U}(\{-1, 1\})^{d_{\text{in}}}$}
        \State \textcolor{teal}{$W \gets \text{BlockFWHT}_{1\text{D}}(W \odot S, B_H)$}
        \State \textcolor{teal}{$XX^\top \gets \text{BlockFWHT}_{2\text{D}}(XX^\top \odot (S S^\top), B_H)$}

        \State {$C, b, s :=$ k\_means\_initialize($W$)} 
        \While{loss improves by at least $\tau$}
        \For{\textcolor{blue}{$k = 1, \dotsc, M$}}
            \State \textcolor{blue}{$\widehat{W}^{(k)} \gets$ partial reconstruction using the first $k$ codebooks}
            \State \textcolor{blue}{$L^{(k)} \gets 
            \frac{1}{d_{\text{out}}}
            \operatorname{tr}\!\left[
                (W - \widehat{W}^{(k)})
                \left(\frac{1}{n} XX^\top\right)
                (W - \widehat{W}^{(k)} )^{\top}
            \right]$}
            \EndFor
            \State \textcolor{blue}{Total loss: $\mathcal{L}_{\mathrm{DbyD}} = \sum_k \alpha_k L^{(k)}$}
            \State Update codebooks $C$ by gradient descent on $\mathcal{L}_{\mathrm{DbyD}}$
            \State Update assignments $b$ by beam search against $L^{(M)}$ only
        \EndWhile
        \State layer.weight := AQLMFormat($C, b, s$)
        \State \textcolor{teal}{layer.rotation\_signs := $S$, \quad layer.block\_size := $B_H$}
    \EndFor
    \State $\theta :=$ trainable\_parameters(block)
    \While{loss improves by at least $\tau$}
        \State $\mathcal{L}_{\mathrm{FT}} := \| \texttt{block}(X_{\text{block}}) - Y_{\text{block}} \|_2^2$ \Comment{targets $M$; not re-weighted per prefix}
        \State $\theta :=$ adam($\theta$, $\partial \mathcal{L}_{\mathrm{FT}} / \partial \theta$)
    \EndWhile
    \State $X_{\text{block}} :=$ block($X_{\text{block}}$)
\EndFor

\end{algorithmic}
\end{algorithm}

Algorithms~\ref{alg:inference_plain} and~\ref{alg:inference_reconstruction2} give the corresponding inference-time reconstruction procedures for the standard DbyD and Hadamard-rotated DbyD variant, respectively; the latter rotates the input activation rather than the dequantized weight. The rotation is invertible; rather than undoing it once and materializing a dense weight, we apply the identical rotation to the layer's input activation, $x$, on every forward pass, while the weight remains permanently stored in its rotated, compressed form. This introduces no additional VRAM overhead relative to vanilla AQLM: the codebooks and codes never leave their compressed representation, and the weight is reconstructed only transiently for each matmul. The added cost is a single block-Hadamard transform over the input activation vector and not the weight matrix. Hence, it remains negligible relative to the matmul calculation.

\begin{algorithm}[h]
\caption{Drop-by-Drop Weight Reconstruction for Resource-Constrained Inference}
\label{alg:inference_plain}
\begin{algorithmic}[1]
\Require Compressed weight components: \texttt{codes}, \texttt{codebooks}, \texttt{scales}
\Require Desired number of codebooks \( k \), total codebooks \( M \)
\Require Layer meta-parameters: rotation signs $S$, block size $B_H$
\If{\textcolor{blue}{\( k \) is specified}}
    \State \textcolor{blue}{\( k \gets \min(k, M) \)}
\Else
    \State \textcolor{blue}{\( k \gets M \)}
\EndIf
\State \(W\) \(\gets\) \texttt{Dequantize}(codes, codebooks, scales, \( k \))
\State \Return \(W\)
\end{algorithmic}
\end{algorithm}

\begin{algorithm}[h]
\caption{Drop-by-Drop Forward Pass for Resource-Constrained Inference for FWHT}
\label{alg:inference_reconstruction2}
\begin{algorithmic}[1]
\Require Compressed weight components: \texttt{codes}, \texttt{codebooks}, \texttt{scales}
\Require Input activation $x$
\Require Desired number of codebooks $k$, total codebooks $M$
\Require Layer meta-parameters: rotation signs $S$, block size $B_H$
\If{\textcolor{blue}{\( k \) is specified}}
    \State \textcolor{blue}{\( k \gets \min(k, M) \)}
\Else
    \State \textcolor{blue}{\( k \gets M \)}
\EndIf
\If{\textcolor{teal}{$S$ is present and $B_H > 1$}}
    \State \textcolor{teal}{$x \gets x \odot S$} \Comment{Apply sign flip}
    \State \textcolor{teal}{$x \gets \text{BlockFWHT}_{1\text{D}}(x, B_H)$} \Comment{Apply rotation}
\EndIf

\State $\widehat{W}^{(k)} \gets \texttt{Dequantize}(\text{codes}, \text{codebooks}, \text{scales}, k)$ 
\State \Return \textcolor{teal}{$\widehat{W}^{(k)} x$}
\end{algorithmic}
\end{algorithm}

\clearpage

\section{Ablation Study}
\label{app:ablation-study}

\subsection{Tradeoff Between Codebook Precision and Codebook Count}

A key design question in multi-codebook quantization is whether performance is better preserved by using many low-precision codebooks or fewer high-precision ones, under a fixed overall bit budget. Prior observations from AQLM suggest that increasing codebook precision can be more beneficial than increasing the number of codebooks when the average bit rate is held constant.

To validate whether this behavior persists under our Drop-by-Drop training strategy, we conduct a controlled experiment on \textsc{Gemma-2B} \cite{gemma2024}. We evaluate perplexity on \textsc{WikiText2} and \textsc{C4} under an \emph{extreme} quantization regime, where the average effective bit rate is approximately 2 bits. We compare two configurations trained with identical optimization settings and Drop-by-Drop loss, differing only in how the bit budget is allocated across codebooks. Specifically, we contrast a configuration with fewer high-precision codebooks against one with more low-precision codebooks. Despite having comparable average bit rates (e.g., average bit \(\approx\) 2 bits), the high-precision configuration consistently achieves substantially lower perplexity on both datasets.

\begin{table}[h!]
    \centering
    \caption{Perplexity on \textsc{WikiText2} and \textsc{C4} ($\downarrow$) for codebook precision vs. codebook count ablation.}
    \begin{tabular}{lccc c c}
        \toprule
        \textbf{Configuration} & \textbf{Bits per Codebook} & \textbf{\# Codebooks} & \textbf{\textsc{WikiText2}}& \textbf{\textsc{C4}} \\
        \midrule
        High precision & 8 & 2 & \textbf{19.22} & \textbf{23.37} \\
        More codebooks & 4 & 4 & 34.02 & 77.17 \\
        \bottomrule
    \end{tabular}
    \label{tab:precision_vs_count}
\end{table}

These results indicate that, under tight bit budgets, allocating more bits per codebook is significantly more effective than increasing the number of codebooks. This finding supports the core design of Drop-by-Drop, which prioritizes strong intermediate reconstructions and enables robust performance even when only a small subset of high-precision codebooks is active. This property is critical for flexible deployment across devices with heterogeneous resource constraints.

\subsection{Matryoshka Loss Coefficients}

We next study the impact of different Matryoshka loss weightings on performance. The Matryoshka loss supervises multiple intermediate reconstructions corresponding to different subsets of active codebooks, and its weighting scheme directly influences which intermediate representations are emphasized during training.

We experiment with a wide range of weighting strategies, including highly skewed coefficients, uniform averages, and exponentially increasing and decreasing coefficients, as well as variants that selectively supervise specific intermediate codebook combinations, such as:

\begin{itemize}
    \item \textbf{3W}: $\alpha_3 = 1$, all other $\alpha_k$ values set to zero.
    \item \textbf{345W}: $\alpha_3 = \alpha_4 = \alpha_5 = \frac{1}{3}$, all other $\alpha$ values set to zero.
    \item \textbf{35W}: $\alpha_3 = \alpha_5 = \frac{1}{2}$, all other $\alpha$ values set to zero.
\end{itemize}

\begin{figure}[p]
\centering
\tikzset{
  sup/.style={draw=black!75, very thick},
  uns/.style={draw=black!45, thin, dashed},
  box/.style={draw=black!40, rounded corners=2pt, inner sep=3pt, font=\small, anchor=west},
  gbox/.style={draw=black!25, rounded corners=2pt, inner sep=3pt, font=\small, anchor=west, text=black!45},
  arr/.style={-{Stealth[length=5pt]}, black!65}
}

\begin{subfigure}[t]{\linewidth}
\centering
\begin{tikzpicture}
\draw[sup, fill=black!8]  (0,0) ellipse (2.65 and 3.50);
\draw[sup, fill=black!16] (0,0) ellipse (2.23 and 2.95);
\draw[sup, fill=black!26] (0,0) ellipse (1.81 and 2.40);
\draw[sup, fill=black!38] (0,0) ellipse (1.39 and 1.85);
\draw[sup, fill=black!52] (0,0) ellipse (0.97 and 1.30);
\node[box] (a1) at (3.2,  2.6) {$\widehat{\mathbf{W}}^{(1)} = C_1 \;\rightarrow\; \alpha_1 L^{(1)},\quad \alpha_1 = \tfrac15$};
\node[box] (a2) at (3.2,  1.4) {$\widehat{\mathbf{W}}^{(2)} = C_1 + C_2 \;\rightarrow\; \alpha_2 L^{(2)},\quad \alpha_2 = \tfrac15$};
\node[box] (a3) at (3.2,  0.2) {$\widehat{\mathbf{W}}^{(3)} = C_1 + C_2 + C_3 \;\rightarrow\; \alpha_3 L^{(3)},\quad \alpha_3 = \tfrac15$};
\node[box] (a4) at (3.2, -1.0) {$\widehat{\mathbf{W}}^{(4)} = C_1 + C_2 + C_3 + C_4 \;\rightarrow\; \alpha_4 L^{(4)},\quad \alpha_4 = \tfrac15$};
\node[box] (a5) at (3.2, -2.2) {$\widehat{\mathbf{W}}^{(5)} = C_1 + C_2 + C_3 + C_4 + C_5 \;\rightarrow\; \alpha_5 L^{(5)},\quad \alpha_5 = \tfrac15$};
\draw[arr] (0.25,  1.26) -- (a1.west);
\draw[arr] (0.80,  1.52) -- (a2.west);
\draw[arr] (1.64,  1.01) -- (a3.west);
\draw[arr] (2.22, -0.26) -- (a4.west);
\draw[arr] (2.29, -1.75) -- (a5.west);
\end{tikzpicture}
\caption{DbyD (Uni.): every prefix is supervised.}
\label{fig:sched-uni}
\end{subfigure}

\vspace{1.2em}

\begin{subfigure}[t]{\linewidth}
\centering
\begin{tikzpicture}
\draw[sup, fill=black!8]  (0,0) ellipse (2.65 and 3.50);
\draw[uns, fill=black!16] (0,0) ellipse (2.23 and 2.95);
\draw[sup, fill=black!26] (0,0) ellipse (1.81 and 2.40);
\draw[uns, fill=black!38] (0,0) ellipse (1.39 and 1.85);
\draw[uns, fill=black!52] (0,0) ellipse (0.97 and 1.30);
\node[gbox] (c1) at (3.2,  2.6) {$\widehat{\mathbf{W}}^{(1)} = C_1 \;\rightarrow\; \alpha_1 = 0$};
\node[gbox] (c2) at (3.2,  1.4) {$\widehat{\mathbf{W}}^{(2)} = C_1 + C_2 \;\rightarrow\; \alpha_2 = 0$};
\node[box]  (c3) at (3.2,  0.2) {$\widehat{\mathbf{W}}^{(3)} = C_1 + C_2 + C_3 \;\rightarrow\; \alpha_3 L^{(3)},\quad \alpha_3 = \tfrac12$};
\node[gbox] (c4) at (3.2, -1.0) {$\widehat{\mathbf{W}}^{(4)} = C_1 + C_2 + C_3 + C_4 \;\rightarrow\; \alpha_4 = 0$};
\node[box]  (c5) at (3.2, -2.2) {$\widehat{\mathbf{W}}^{(5)} = C_1 + C_2 + C_3 + C_4 + C_5 \;\rightarrow\; \alpha_5 L^{(5)},\quad \alpha_5 = \tfrac12$};
\draw[arr] (0.25,  1.26) -- (c1.west);
\draw[arr] (0.80,  1.52) -- (c2.west);
\draw[arr] (1.64,  1.01) -- (c3.west);
\draw[arr] (2.22, -0.26) -- (c4.west);
\draw[arr] (2.29, -1.75) -- (c5.west);
\end{tikzpicture}
\caption{DbyD (35W): only $k = 3$ and $k = 5$ are supervised.}
\label{fig:sched-35w}
\end{subfigure}

\caption{Matryoshka coefficient variants, with $\mathcal{L}_{\mathrm{DbyD}} = \sum_k \alpha_k L^{(k)}$ and $M = 5$. Each ellipse is the partial reconstruction $\widehat{\mathbf{W}}^{(k)}$ formed from the first $k$ codebooks, where $C_m$ denotes the contribution of the $m$-th codebook, so nesting encodes prefix containment. Shading intensity indicates distortion, with the coarsest reconstruction in darkest gray; sizes and shades are illustrative and carry no quantitative meaning.}
\label{fig:schedules}
\end{figure}

Overall, we find that aggressively weighting early or late reconstructions alone does not consistently improve performance. Instead, selectively supervising intermediate reconstructions yields the most stable behavior under Drop-by-Drop training. Two visual examples are provided in Figure~\ref{fig:schedules}. 

This observation aligns with our goal of maintaining strong performance at intermediate bit rates, rather than optimizing exclusively for the highest-fidelity reconstruction.

\subsection{Tradeoff Between Number of Supported Bitwidths and Performance}

We further investigate whether restricting the set of supported bitwidths (and concentrating supervision on fewer intermediate reconstructions) yields better performance at those bitwidths. To investigate this, we compare two loss variants specialized for bitwidths \{3, 4\} against the three-bitwidth baseline \{3, 4, 5\}. Hence, for this ablation, models are trained with a maximum of $M = 4$
codebooks (rather than the $M = 5$ used in Section~\ref{section:experiments}), so AQLM (Ind.) and
AQLM (Drop) coincide exactly at the 4-bit setting, as neither involves dropping
at that precision.

\begin{itemize}
    \item \textbf{34W}: $\alpha_1 =\alpha_2 = 0; \alpha_3 = \alpha_4 = \frac{1}{2}$.
    \item \textbf{234W}: $\alpha_1 = 0; \alpha_2 = \alpha_3 = \alpha_4 = \frac{1}{3}$.
\end{itemize}

As shown in Tables~\ref{tab:combined_ppl_wiki_34w} and~\ref{tab:combined_ppl_c4_34w}, both 34W and 234W achieve lower perplexity at 3 and 4 bits compared to the full three-bitwidth DbyD model, \footnote{See
the DbyD (Uni.) and DbyD (35W) rows in Tables~\ref{tab:combined_ppl_c4} and~\ref{tab:combined_ppl_wiki} for these values.} confirming that narrowing the supervision target improves performance at the supported bitwidths. However, this comes at the cost of flexibility: a model trained with 34W does not support a 5-bit precision without retraining. This tradeoff between adaptability and peak performance is a key consideration when selecting a loss configuration for deployment.

\begin{table}[h!]
\centering
\caption{Perplexity on \textsc{WikiText2} ($\downarrow$) for supported bitwidths vs performance ablation.}
\small
\setlength{\tabcolsep}{6pt} %

\begin{tabular}{lcrrrr}
\toprule
\textbf{Model} & \textbf{Bits} & \textbf{AQLM (Ind.)} & \textbf{AQLM (Drop)} & \textbf{DbyD (34W)} & \textbf{DbyD (234W)} \\
\midrule

\multirow{3}{*}{Qwen 0.5B}
  & 4 & 13.76 & 13.76 & \textbf{13.91}    & 13.93  \\
  & 3 & 15.43 & 26.81 & \textbf{23.71}     & 24.07  \\
\midrule

\multirow{3}{*}{Qwen 3B}
  & 4 & 7.71  & 7.71  & \textbf{7.74}      & 7.77   \\
  & 3 & 8.25  & 39.83 &\textbf{9.39}     & 9.47   \\
\midrule

\multirow{3}{*}{Qwen 7B}
  & 4 & 7.08  & 7.08  & \textbf{7.13}      & 7.16   \\
  & 3 & 7.47  & 9.01  & \textbf{8.51}     & 8.64   \\
\bottomrule
\end{tabular}
\label{tab:combined_ppl_wiki_34w}
\end{table}

\begin{table}[h!]
\centering
\caption{Perplexity on \textsc{C4} ($\downarrow$) for supported bitwidths vs performance ablation.}
\small
\setlength{\tabcolsep}{6pt}

\begin{tabular}{lcrrrr}
\toprule
\textbf{Model} & \textbf{Bits} & \textbf{AQLM (Ind.)} & \textbf{AQLM (Drop)} & \textbf{DbyD (34W)} & \textbf{DbyD (234W)} \\
\midrule

\multirow{3}{*}{Qwen 0.5B}
  & 4 & 18.73 & 18.73 & \textbf{18.90} & 19.11 \\
  & 3 & 21.87 & 33.46 & 32.40 & \textbf{30.72} \\
\midrule

\multirow{3}{*}{Qwen 3B}
  & 4 & 11.72 & 11.72 & \textbf{11.76} & 11.84 \\
  & 3 & 12.94 & 48.40 & \textbf{13.91} & 14.12 \\
\midrule

\multirow{3}{*}{Qwen 7B}
  & 4 & 10.94 & 10.94 & \textbf{10.92} & 11.04 \\
  & 3 & 11.85 & 13.41 & \textbf{12.75} & 12.97 \\
\bottomrule
\end{tabular}
\label{tab:combined_ppl_c4_34w}
\end{table}

\subsection{Computational and Storage Benefits}
Drop-by-Drop demonstrates that a single quantized model can serve multiple resource budgets without retraining or maintaining separate checkpoints. Conventional methods such as AQLM require a separate model per target configuration, which in our setup means three independently trained and stored models at $3\times8$, $4\times8$, and $5\times8$. As shown in Figure~\ref{fig:disk_usage_training_time}, DbyD trains and stores only once showcasing significant training hours and disk space benefits. 

\begin{figure}[h]
    \centering
    \begin{minipage}{0.48\textwidth}
        \centering
        \begin{tikzpicture}
            \begin{axis}[
                ybar,
                width=\linewidth,
                height=4.5cm,
                bar width=6pt,
                enlarge x limits=0.15,
                ylabel={Total Training Hours},
                ylabel style={font=\small, yshift=0pt},
                symbolic x coords={Q-0.5,G-2,Q-3,Q-7,M-7,L-8,Q-14,Q-32},
                xtick=data,
                xticklabel style={
                    font=\tiny, 
                    rotate=45, 
                    anchor=north east, 
                    inner sep=2pt
                },
                ymin=0, ymax=450,
                ymajorgrids=true,
                grid style={dashed, gray!30},
                axis line style={draw=none},
                tick style={draw=none},
                legend style={
                    at={(0.02,0.95)}, 
                    anchor=north west, 
                    draw=none, 
                    fill=white, 
                    fill opacity=0.9,
                    font=\tiny,
                    cells={anchor=west}
                },
            ]
            \addplot[fill=gray!50, draw=none] coordinates {
                (Q-0.5,4.9) (G-2,22.5) (Q-3,25.9) (Q-7,68.2)
                (M-7,82.0) (L-8,66.8) (Q-14,140.0) (Q-32,386.5)
            };
            \addlegendentry{AQLM (Cumulative)}

            \addplot[fill=green!60!black, draw=none] coordinates {
                (Q-0.5,2.3) (G-2,9.5) (Q-3,19.1) (Q-7,60.0)
                (M-7,27.0) (L-8,26.0) (Q-14,74.0) (Q-32,240.0)
            };
            \addlegendentry{\textbf{DbyD (Ours)}}
            \end{axis}
        \end{tikzpicture}
    \end{minipage}
    \hfill
    \begin{minipage}{0.48\textwidth}
        \centering
        \begin{tikzpicture}
            \begin{axis}[
                ybar,
                width=\linewidth,
                height=4.5cm,
                bar width=6pt,
                enlarge x limits=0.15,
                ylabel={Total Disk Usage (GB)},
                ylabel style={font=\small, yshift=0pt},
                symbolic x coords={Q-0.5,G-2,Q-3,Q-7,M-7,L-8,Q-14,Q-32},
                xtick=data,
                xticklabel style={
                    font=\tiny, 
                    rotate=45, 
                    anchor=north east, 
                    inner sep=2pt
                },
                ymin=0, ymax=60,
                ymajorgrids=true,
                grid style={dashed, gray!30},
                axis line style={draw=none},
                tick style={draw=none},
                legend style={
                    at={(0.02,0.95)}, 
                    anchor=north west, 
                    draw=none, 
                    fill=white, 
                    fill opacity=0.9,
                    font=\tiny,
                    cells={anchor=west}
                },
            ]
            \addplot[fill=gray!50, draw=none] coordinates {
                (Q-0.5, 1.36) (G-2, 6.13) (Q-3, 6.05) (Q-7, 16.35)
                (M-7, 12.10) (L-8, 16.79) (Q-14, 29.20) (Q-32, 56.21)
            };
            \addlegendentry{AQLM (Cumulative)}

            \addplot[fill=green!60!black, draw=none] coordinates {
                (Q-0.5, 0.50) (G-2, 2.29) (Q-3, 2.36) (Q-7, 6.27)
                (M-7, 4.91) (L-8, 6.47) (Q-14, 11.39) (Q-32, 22.64)
            };
            \addlegendentry{\textbf{DbyD (Ours)}}
            \end{axis}
        \end{tikzpicture}
    \end{minipage}
    \caption{Training time and disk space required to produce a suite of quantized models}
    \label{fig:disk_usage_training_time}
\end{figure}

\clearpage

\section{Quantitative Empirical Results}
\label{app:detailed}
All perplexity and accuracy results report the mean over \textit{three} independent runs. For perplexity, the \(\pm\) value is the standard deviation across runs. For zero-shot accuracy, the \(\pm\) value is the root mean square (RMS) of the standard errors reported by the evaluation code for each run,
\begin{equation*}
    \mathrm{SE}_{\mathrm{RMS}} = \sqrt{\frac{1}{N}\sum_{i=1}^{N} \mathrm{SE}_i^2},
\end{equation*}
where \(N\) is the number of runs and \(\mathrm{SE}_i\) is the standard error of run \(i\). Tables~\ref{tab:combined_ppl_c4}, \ref{tab:combined_ppl_wiki}, and~\ref{tab:combined_avg_acc} report the values plotted in Figures~\ref{fig:perplexity_c4},~\ref{fig:perplexity_wiki}, and~\ref{fig:accuracy_zeroshot}. Lower perplexity (\(\downarrow\)) and higher accuracy (\(\uparrow\)) indicate better performance.

\subsection{Perplexity and Average Zero-Shot Accuracy Values}

\begin{table}[h!]
\small
\caption{Perplexity on \textsc{C4} ($\downarrow$). Values plotted in Figure~\ref{fig:perplexity_c4}.}
\begin{center}
\resizebox{\textwidth}{!}{%
\begin{tabular}{cl r@{\,\(\pm\)\,}l r@{\,\(\pm\)\,}l r@{\,\(\pm\)\,}l r@{\,\(\pm\)\,}l r@{\,\(\pm\)\,}l}
\toprule
\textbf{Model} & \makecell{\textbf{Code-} \\ \textbf{books}} & 
\multicolumn{2}{c}{\makecell{\textbf{AQLM} \\ \textbf{(Ind.)}}} & 
\multicolumn{2}{c}{\makecell{\textbf{AQLM} \\ \textbf{(Drop)}}} & 
\multicolumn{2}{c}{\makecell{\textbf{DbyD} \\ \textbf{(Uni.)}}} & 
\multicolumn{2}{c}{\makecell{\textbf{DbyD} \\ \textbf{(35W)}}} & 
\multicolumn{2}{c}{\makecell{\textbf{DbyD} \\ \textbf{(Had.+35W)}}} \\
\midrule
\multirow{3}{*}{\makecell{Gemma \\ 2B\(^\dagger\)}}
  & \(5 \times 8\) & \multicolumn{2}{c}{12.50} & \multicolumn{2}{c}{12.50} & \multicolumn{2}{c}{12.47} & 12.06 & 0.18 & 11.79 & 0.00 \\
  & \(4 \times 8\) & \multicolumn{2}{c}{12.91} & \multicolumn{2}{c}{13.28} & \multicolumn{2}{c}{14.02} & 13.59 & 0.44 & 13.23 & 0.37 \\
  & \(3 \times 8\) & \multicolumn{2}{c}{15.13} & \multicolumn{2}{c}{18.37} & \multicolumn{2}{c}{18.26} & 15.93 & 0.53 & 17.07 & 1.15 \\
\midrule
\multirow{3}{*}{\makecell{Mistral \\ 7B}}
  & $5 \times 8$ & 7.93 & 0.62 & 7.93 & 0.62 & 7.70 & 0.04 & 7.58 & 0.01 & 7.54 & 0.00 \\
  & $4 \times 8$ & 7.73 & 0.01 & 8.74 & 0.39 & 8.31 & 0.05 & 8.08 & 0.07 & 7.71 & 0.00 \\
  & $3 \times 8$ & 8.23 & 0.01 & 104.10 & 77.90 & 16.09 & 9.05 & 9.14 & 0.18 & 8.32 & 0.00 \\
\midrule
\multirow{3}{*}{\makecell{MetaLlama \\ 8B}}
  & $5 \times 8$ & 8.71 & 0.01 & 8.71 & 0.01 & 9.01 & 0.01 & 8.78 & 0.00 & 8.63 & 0.00\\
  & $4 \times 8$ & 9.29 & 0.01 & 28.22 & 27.49 & 10.66 & 0.36 & 10.39 & 0.21 & 9.38 & 0.04 \\
  & $3 \times 8$ & 10.73 & 0.01 & $7.8{\times}10^4$ & $1.3{\times}10^5$ & 20.61 & 2.93 & 12.97 & 0.36 &  11.53 & 0.05 \\
\midrule
\multirow{3}{*}{\makecell{Qwen \\ 0.5B}}
  & $5 \times 8$ & 17.87 & 0.02 & 17.87 & 0.02 & 18.30 & 0.02 & 18.00 & 0.03 & 17.87 & 0.00 \\
  & $4 \times 8$ & 18.73 & 0.02 & 20.91 & 0.24 & 21.36 & 0.37 & 21.48 & 1.18 & 19.99 & 0.06 \\
  & $3 \times 8$ & 21.87 & 0.09 & 58.76 & 21.65 & 44.37 & 6.07 & 37.58 & 8.31 & 27.67 & 1.22 \\
\midrule
\multirow{3}{*}{\makecell{Qwen \\ 3B}}
  & $5 \times 8$ & 11.35 & 0.01 & 11.35 & 0.01 & 11.50 & 0.00 & 11.39 & 0.00 & 11.32 & 0.01 \\
  & $4 \times 8$ & 11.72 & 0.02 & 22.32 & 15.58 & 12.44 & 0.07 & 12.24 & 0.03 & 11.93 & 0.01 \\
  & $3 \times 8$ & 12.94 & 0.01 & 283.97 & 113.15 & 14.94 & 0.08 & 14.32 & 0.03 & 13.96 & 0.00 \\
\midrule
\multirow{3}{*}{\makecell{Qwen \\ 7B}}
  & $5 \times 8$ & 10.61 & 0.01 & 10.61 & 0.01 & 10.80 & 0.02 & 10.63 & 0.02 & 10.51 & 0.01 \\
  & $4 \times 8$ & 10.94 & 0.00 & 11.79 & 0.16 & 11.66 & 0.07 & 11.39 & 0.01 & 10.93 & 0.03 \\
  & $3 \times 8$ & 11.85 & 0.01 & 15.10 & 0.19 & 13.91 & 0.34 & 13.24 & 0.07 & 12.27 & 0.12 \\
\midrule
\multirow{3}{*}{\makecell{Qwen \\ 14B}}
  & $5 \times 8$ & 9.10 & 0.00 & 9.10 & 0.00 & 9.22 & 0.01 & 9.26 & 0.23 & 9.06 & 0.00 \\
  & $4 \times 8$ & 9.33 & 0.01 & 9.77 & 0.04 & 9.78 & 0.01 & 9.71 & 0.02 & 9.40 & 0.00 \\
  & $3 \times 8$ & 10.01 & 0.02 & 12.02 & 0.20 & 11.44 & 0.47 & 10.82 & 0.26 & 10.38 & 0.01 \\
\midrule
\multirow{3}{*}{\makecell{Qwen \\ 32B}}
  & $5 \times 8$ & 8.84 & 0.00 & 8.84 & 0.00 & 8.89 & 0.01 & 8.84 & 0.01 & 8.79 & 0.00 \\
  & $4 \times 8$ & 8.98 & 0.00 & 9.34 & 0.05 & 9.25 & 0.01 & 9.17 & 0.02 & 8.98 & 0.00 \\
  & $3 \times 8$ & 9.46 & 0.04 & 10.54 & 0.20 & 10.10 & 0.02 & 9.77 & 0.01 & 9.49 & 0.00 \\
\bottomrule
\end{tabular}%
}
\par\smallskip
\begin{minipage}{\textwidth}
\footnotesize
\(^\dagger\)AQLM (Ind.), AQLM (Drop), and DbyD (Uni.) are single runs; in training.
\end{minipage}
\end{center} 
\label{tab:combined_ppl_c4}
\end{table}

\begin{figure}[h!]
    \centering
    \vspace{-3mm}
    \hspace{0.85cm} %
    \begin{tikzpicture}
        \begin{groupplot}[
            group style={
                group size=4 by 2, %
                horizontal sep=0.8cm, %
                vertical sep=1cm,
                xlabels at=edge bottom,
                ylabels at=edge left
            },
            width=3.5cm, %
            height=3.5cm,
            xlabel={Codebooks},
            ylabel={Perplexity},
            symbolic x coords={3x8, 4x8, 5x8},
            xtick=data,
            ymajorgrids=true,
            grid style=dashed,
            tick label style={font=\tiny}, %
            label style={font=\footnotesize},
            title style={font=\bfseries\scriptsize, yshift=-0.5ex}, %
            every axis plot/.append style={
                mark=*,
                mark size=1.5pt, %
                thick
            },
            cycle list={
                {myblue, mark=square*},
                {myred, mark=square},
                {myorange, mark=o},
                {mypurple, mark=*},
                {mygreen, mark=diamond*, mark options={scale=1.4}} %
            },
            legend columns=-1, %
            legend style={
                draw=none, 
                /tikz/every even column/.append style={column sep=0.15cm}, %
                font=\scriptsize %
            },
            legend to name=CommonLegend
        ]

        \nextgroupplot[title=Gemma 2B, ymin=8, ymax=14]
        \addplot coordinates {(3x8, 10.35) (4x8, 9.12) (5x8, 8.94)};
        \addplot coordinates {(3x8, 13.21) (4x8, 9.59) (5x8, 8.94)};
        \addplot coordinates {(3x8, 13.11) (4x8, 9.91) (5x8, 8.79)};
        \addplot coordinates {(3x8, 11.66) (4x8, 9.76) (5x8, 8.61)};
        \addplot coordinates {(3x8, 13.18) (4x8, 9.69) (5x8, 8.49)};

        \nextgroupplot[title=Mistral 7B, ymin=4, ymax=15]
        \addplot coordinates {(3x8, 5.38) (4x8, 5.07) (5x8, 5.14)};
        \addplot coordinates {(3x8, 30.00) (4x8, 5.78) (5x8, 5.14)}; %
        \addplot coordinates {(3x8, 11.42) (4x8, 5.66) (5x8, 5.19)};
        \addplot coordinates {(3x8, 6.19) (4x8, 5.44) (5x8, 5.00)};
        \addplot coordinates {(3x8, 5.62) (4x8, 5.13) (5x8, 4.99)};

        \nextgroupplot[title=MetaLlama 8B, ymin=5, ymax=22]
        \addplot coordinates {(3x8, 6.77) (4x8, 6.14) (5x8, 5.89)};
        \addplot coordinates {(3x8, 50.00) (4x8, 20.10) (5x8, 5.89)};
        \addplot coordinates {(3x8, 12.57) (4x8, 7.09) (5x8, 6.02)};
        \addplot coordinates {(3x8, 8.77) (4x8, 7.04) (5x8, 5.93)};
        \addplot coordinates {(3x8, 8.07) (4x8, 6.47) (5x8, 5.96)};

        \nextgroupplot[title=Qwen 0.5B, ymin=12, ymax=55]
        \addplot coordinates {(3x8, 15.43) (4x8, 13.76) (5x8, 13.30)};
        \addplot coordinates {(3x8, 51.71) (4x8, 16.21) (5x8, 13.30)};
        \addplot coordinates {(3x8, 34.14) (4x8, 16.29) (5x8, 13.54)};
        \addplot coordinates {(3x8, 28.40) (4x8, 16.65) (5x8, 13.36)};
        \addplot coordinates {(3x8, 22.05) (4x8, 15.49) (5x8, 13.36)};

        \nextgroupplot[title=Qwen 3B, ymin=7, ymax=35]
        \addplot coordinates {(3x8, 8.25) (4x8, 7.71) (5x8, 7.54)};
        \addplot coordinates {(3x8, 50.00) (4x8, 30.70) (5x8, 7.54)}; %
        \addplot coordinates {(3x8, 10.25) (4x8, 8.32) (5x8, 7.63)};
        \addplot coordinates {(3x8, 9.91) (4x8, 8.22) (5x8, 7.57)};
        \addplot coordinates {(3x8, 9.73) (4x8, 8.05) (5x8, 7.57)};

        \nextgroupplot[title=Qwen 7B, ymin=6.5, ymax=11]
        \addplot coordinates {(3x8, 7.47) (4x8, 7.08) (5x8, 6.93)};
        \addplot coordinates {(3x8, 10.33) (4x8, 7.84) (5x8, 6.93)};
        \addplot coordinates {(3x8, 9.37) (4x8, 7.65) (5x8, 7.06)};
        \addplot coordinates {(3x8, 9.17) (4x8, 7.61) (5x8, 6.96)};
        \addplot coordinates {(3x8, 8.40) (4x8, 7.28) (5x8, 6.93)};

        \nextgroupplot[title=Qwen 14B, ymin=5, ymax=9]
        \addplot coordinates {(3x8, 6.13) (4x8, 5.60) (5x8, 5.38)};
        \addplot coordinates {(3x8, 8.32) (4x8, 6.06) (5x8, 5.38)};
        \addplot coordinates {(3x8, 7.65) (4x8, 6.08) (5x8, 5.51)};
        \addplot coordinates {(3x8, 7.06) (4x8, 6.01) (5x8, 5.53)};
        \addplot coordinates {(3x8, 6.76) (4x8, 5.78) (5x8, 5.38)};

        \nextgroupplot[title=Qwen 32B, ymin=4.5, ymax=7]
        \addplot+[color=myblue] coordinates {(3x8, 5.48) (4x8, 5.13) (5x8, 4.99)};
        \addlegendentry{AQLM (Ind.)}
        
        \addplot+[color=myred] coordinates {(3x8, 6.65) (4x8, 5.51) (5x8, 4.99)};
        \addlegendentry{AQLM (Drop)}
        
        \addplot+[color=myorange] coordinates {(3x8, 6.29) (4x8, 5.43) (5x8, 5.04)};
        \addlegendentry{DbyD (Uni.)}
        
        \addplot+[color=mypurple] coordinates {(3x8, 6.06) (4x8, 5.38) (5x8, 5.01)};
        \addlegendentry{DbyD (35W)}

        \addplot+[color=mygreen, mark=diamond*, mark options={scale=1.4}] coordinates {(3x8, 5.80) (4x8, 5.21) (5x8, 4.99)};
        \addlegendentry{DbyD (Had.+35W)}

        \end{groupplot}

        \node at ($(group c1r2.south west)!0.5!(group c4r2.south east)$) [yshift=-1.2cm] {\ref{CommonLegend}};
        
    \end{tikzpicture}
    \caption{Perplexity on \textsc{WikiText2} ($\downarrow$). Extreme outliers for AQLM (Drop) are clipped.}
    \label{fig:perplexity_wiki}
\end{figure}

\begin{table}[h!]
\small
\caption{Perplexity on \textsc{WikiText2} ($\downarrow$). Values plotted in Figure~\ref{fig:perplexity_wiki}.} 
\begin{center}
\resizebox{\textwidth}{!}{%
\begin{tabular}{cl r@{\,\(\pm\)\,}l r@{\,\(\pm\)\,}l r@{\,\(\pm\)\,}l r@{\,\(\pm\)\,}l r@{\,\(\pm\)\,}l}
\toprule
\textbf{Model} & \makecell{\textbf{Code-} \\ \textbf{books}} & 
\multicolumn{2}{c}{\makecell{\textbf{AQLM} \\ \textbf{(Ind.)}}} & 
\multicolumn{2}{c}{\makecell{\textbf{AQLM} \\ \textbf{(Drop)}}} & 
\multicolumn{2}{c}{\makecell{\textbf{DbyD} \\ \textbf{(Uni.)}}} & 
\multicolumn{2}{c}{\makecell{\textbf{DbyD} \\ \textbf{(35W)}}} & 
\multicolumn{2}{c}{\makecell{\textbf{DbyD} \\ \textbf{(Had.+35W)}}} \\
\midrule
\multirow{3}{*}{\makecell{Gemma \\ 2B\(^\dagger\)}}
  & $5 \times 8$ & \multicolumn{2}{c}{8.94} & \multicolumn{2}{c}{8.94} & \multicolumn{2}{c}{8.79} & 8.61 & 0.12 & 8.49 & 0.00 \\
  & $4 \times 8$ & \multicolumn{2}{c}{9.12} & \multicolumn{2}{c}{9.59} & \multicolumn{2}{c}{9.91} & 9.76 & 0.33 & 9.69 & 0.31 \\
  & $3 \times 8$ & \multicolumn{2}{c}{10.35} & \multicolumn{2}{c}{13.21} & \multicolumn{2}{c}{13.11} & 11.66 & 0.48 & 13.18 & 1.15 \\
\midrule
\multirow{3}{*}{\makecell{Mistral \\ 7B}}
  & $5 \times 8$ & 5.14 & 0.26 & 5.14 & 0.26 & 5.19 & 0.24 & 5.00 & 0.00 & 4.99 & 0.00 \\
  & $4 \times 8$ & 5.07 & 0.01 & 5.78 & 0.30 & 5.66 & 0.24 & 5.44 & 0.15 & 5.13 & 0.00 \\
  & $3 \times 8$ & 5.38 & 0.08 & 112.75 & 93.38 & 11.42 & 5.14 & 6.19 & 0.21 & 5.62 & 0.02 \\
\midrule
\multirow{3}{*}{\makecell{MetaLlama \\ 8B}}
  & $5 \times 8$& 5.89 & 0.00 & 5.89 & 0.00 & 6.02 & 0.00 & 5.93 & 0.00 & 5.96 & 0.00 \\
  & $4 \times 8$  & 6.14 & 0.01 & 20.10 & 20.10 & 7.09 & 0.22 & 7.04 & 0.12 & 6.47 & 0.04 \\
  & $3 \times 8$ & 6.77 & 0.00 & $4.6{\times}10^4$ & $7.6{\times}10^4$ & 12.57 & 1.10 & 8.77 & 0.17 & 8.07 & 0.05  \\
\midrule
\multirow{3}{*}{\makecell{Qwen \\ 0.5B}}
  & $5 \times 8$  & 13.30 & 0.01 & 13.30 & 0.01 & 13.54 & 0.04 & 13.36 & 0.02 & 13.36 & 0.04 \\
  & $4 \times 8$  & 13.76 & 0.02 & 16.21 & 0.34 & 16.29 & 0.17 & 16.65 & 0.70 & 15.49 & 0.16 \\
  & $3 \times 8$ & 15.43 & 0.07 & 51.71 & 24.90 & 34.14 & 6.33 & 28.40 & 5.83 & 22.05 & 0.67 \\
\midrule
\multirow{3}{*}{\makecell{Qwen \\ 3B}}
  & $5 \times 8$ & 7.54 & 0.00 & 7.54 & 0.00 & 7.63 & 0.01 & 7.57 & 0.01 & 7.57 & 0.00 \\
  & $4 \times 8$  & 7.71 & 0.01 & 30.70 & 30.67 & 8.32 & 0.07 & 8.22 & 0.03 & 8.05 & 0.03 \\
  & $3 \times 8$ & 8.25 & 0.01 & 365.00 & 160.30 & 10.25 & 0.09 & 9.91 & 0.00 & 9.73 & 0.11 \\
\midrule
\multirow{3}{*}{\makecell{Qwen \\ 7B}}
  & $5 \times 8$  & 6.93 & 0.01 & 6.93 & 0.01 & 7.06 & 0.02 & 6.96 & 0.01 & 6.93 & 0.01 \\
  & $4 \times 8$  & 7.08 & 0.01 & 7.84 & 0.02 & 7.65 & 0.04 & 7.61 & 0.05 & 7.28 & 0.01 \\
  & $3 \times 8$  & 7.47 & 0.01 & 10.33 & 0.22 & 9.37 & 0.11 & 9.17 & 0.02 & 8.40 & 0.05 \\
\midrule
\multirow{3}{*}{\makecell{Qwen \\ 14B}}
  & $5 \times 8$ & 5.38 & 0.00 & 5.38 & 0.00 & 5.51 & 0.01 & 5.53 & 0.20 & 5.38 & 0.00 \\
  & $4 \times 8$  & 5.60 & 0.01 & 6.06 & 0.06 & 6.08 & 0.03 & 6.01 & 0.09 & 5.78 & 0.02 \\
  & $3 \times 8$  & 6.13 & 0.00 & 8.32 & 0.24 & 7.65 & 0.38 & 7.06 & 0.04 & 6.76 & 0.04 \\
\midrule
\multirow{3}{*}{\makecell{Qwen \\ 32B}}
  & $5 \times 8$  & 4.99 & 0.01 & 4.99 & 0.01 & 5.04 & 0.01 & 5.01 & 0.00 & 4.99 & 0.01 \\
  & $4 \times 8$ & 5.13 & 0.00 & 5.51 & 0.01 & 5.43 & 0.01 & 5.38 & 0.02 & 5.21 & 0.01 \\
  & $3 \times 8$  & 5.48 & 0.01 & 6.65 & 0.20 & 6.29 & 0.05 & 6.06 & 0.01 & 5.80 & 0.03 \\
\bottomrule
\end{tabular}%
}
\par\smallskip
\begin{minipage}{\textwidth}
\footnotesize
\(^\dagger\)AQLM (Ind.), AQLM (Drop), and DbyD (Uni.) are single runs; in training.
\end{minipage} 
\end{center} 
\label{tab:combined_ppl_wiki}
\end{table}

\begin{table}[h!]
\small
\caption{Average zero-shot accuracy across 5 Tasks ($\uparrow$). Values plotted in Figure~\ref{fig:accuracy_zeroshot}.}
\begin{center}
\resizebox{\textwidth}{!}{%
\begin{tabular}{cl r@{\,\(\pm\)\,}l r@{\,\(\pm\)\,}l r@{\,\(\pm\)\,}l r@{\,\(\pm\)\,}l r@{\,\(\pm\)\,}l}
\toprule
\textbf{Model} & \makecell{\textbf{Code-} \\ \textbf{books}} & 
\multicolumn{2}{c}{\makecell{\textbf{AQLM} \\ \textbf{(Ind.)}}} & 
\multicolumn{2}{c}{\makecell{\textbf{AQLM} \\ \textbf{(Drop)}}} & 
\multicolumn{2}{c}{\makecell{\textbf{DbyD} \\ \textbf{(Uni.)}}} & 
\multicolumn{2}{c}{\makecell{\textbf{DbyD} \\ \textbf{(35W)}}} & 
\multicolumn{2}{c}{\makecell{\textbf{DbyD} \\ \textbf{(Had.+35W)}}} \\
\midrule
\multirow{3}{*}{\makecell{Gemma \\ 2B}}
  & $5 \times 8$ & 65.59 & 1.02 & 65.59 & 1.02 & 64.62 & 1.03 & 65.48 & 1.03 & 65.13 & 1.03 \\
  & $4 \times 8$ & 63.95 & 1.03 & 64.42 & 1.03 & 62.38 & 1.04 & 63.90 & 1.03 & 63.48 & 1.03 \\
  & $3 \times 8$ & 60.37 & 1.05 & 61.22 & 1.05 & 56.51 & 1.05 & 60.61 & 1.05 & 60.17 & 1.05 \\
\midrule
\multirow{3}{*}{\makecell{Mistral \\ 7B}}
  & $5 \times 8$ & 71.72 & 1.05 & 71.72 & 1.05 & 72.15 & 1.04 & 72.85 & 1.04 & 73.08 & 0.97 \\
  & $4 \times 8$ & 72.46 & 1.04 & 69.94 & 1.07 & 69.96 & 1.07 & 70.81 & 1.06 & 72.32 & 1.00 \\
  & $3 \times 8$ & 70.97 & 1.05 & 44.84 & 1.11 & 61.78 & 1.10 & 67.48 & 1.08 & 70.20 & 1.02 \\
\midrule
\multirow{3}{*}{\makecell{MetaLlama \\ 8B}}
  & $5 \times 8$ & 72.42 & 1.04 & 72.42 & 1.04 & 71.62 & 1.04 & 72.09 & 1.05 & 73.37 & 0.97 \\
  & $4 \times 8$ & 71.00 & 1.05 & 65.54 & 1.08 & 68.52 & 1.07 & 69.60 & 1.06 &  72.36 & 1.00\\
  & $3 \times 8$ & 70.00 & 1.06 & 38.08 & 1.09 & 55.80 & 1.11 & 64.81 & 1.09 & 67.81 & 1.03\\
\midrule
\multirow{3}{*}{\makecell{Qwen \\ 0.5B}}
  & $5 \times 8$ & 53.25 & 1.12 & 53.25 & 1.12 & 53.62 & 1.12 & 53.86 & 1.12 & 54.17 & 1.12 \\
  & $4 \times 8$ & 53.60 & 1.11 & 52.80 & 1.12 & 51.86 & 1.12 & 53.14 & 1.12 & 52.40 & 1.12 \\
  & $3 \times 8$ & 49.96 & 1.12 & 44.02 & 1.12 & 45.88 & 1.12 & 48.47 & 1.11 & 48.36 & 1.12 \\
\midrule
\multirow{3}{*}{\makecell{Qwen \\ 3B}}
  & $5 \times 8$ & 67.83 & 1.07 & 67.83 & 1.07 & 68.18 & 1.07 & 68.17 & 1.07 & 68.11 & 1.07 \\
  & $4 \times 8$ & 66.68 & 1.08 & 62.11 & 1.10 & 65.85 & 1.09 & 67.28 & 1.08 & 67.80 & 1.08 \\
  & $3 \times 8$ & 65.48 & 1.09 & 39.64 & 1.10 & 62.53 & 1.10 & 63.06 & 1.10 & 64.32 & 1.10 \\
\midrule
\multirow{3}{*}{\makecell{Qwen \\ 7B}}
  & $5 \times 8$ & 73.59 & 1.04 & 73.59 & 1.04 & 72.75 & 1.05 & 73.72 & 1.04 & 73.28 & 1.04 \\
  & $4 \times 8$ & 72.37 & 1.05 & 71.40 & 1.06 & 71.32 & 1.06 & 72.52 & 1.05 & 71.37 & 1.06 \\
  & $3 \times 8$ & 70.83 & 1.07 & 66.77 & 1.10 & 67.30 & 1.09 & 69.54 & 1.08 & 68.66 & 1.08 \\
\midrule
\multirow{3}{*}{\makecell{Qwen \\ 14B}}
  & $5 \times 8$ & 76.85 & 1.01 & 76.85 & 1.01 & 76.03 & 1.02 & 76.76 & 1.01 & 77.20 & 0.93 \\
  & $4 \times 8$ & 75.94 & 1.02 & 74.73 & 1.03 & 75.80 & 1.02 & 75.99 & 1.01 & 76.63 & 0.96 \\
  & $3 \times 8$ & 74.89 & 1.03 & 70.00 & 1.07 & 73.92 & 1.04 & 73.93 & 1.04 & 75.07 & 0.98 \\
\midrule
\multirow{3}{*}{\makecell{Qwen \\ 32B}}
  & $5 \times 8$ & 74.53 & 1.04 & 74.53 & 1.04 & 74.20 & 1.04 & 74.74 & 1.03 & 74.98 & 1.03 \\
  & $4 \times 8$ & 74.50 & 1.04 & 73.10 & 1.05 & 72.56 & 1.05 & 73.98 & 1.04 & 75.22 & 1.03 \\
  & $3 \times 8$ & 74.94 & 1.03 & 71.08 & 1.07 & 71.93 & 1.06 & 72.94 & 1.05 & 73.88 & 1.04 \\
\bottomrule
\end{tabular}
}
\end{center}
\label{tab:combined_avg_acc}
\end{table}

\clearpage

\subsection{Detailed Accuracy Performance Results: Qwen Family}
\label{app:qwen_results}
In this section, we provide the detailed zero-shot accuracy breakdowns for the Qwen model family across five tasks: ARC Challenge, ARC Easy, HellaSwag, PIQA, and Winogrande.

\begin{table}[h!]
\small
\caption{Zero-shot accuracy on arc\_challenge (\%) ($\uparrow$)}
\begin{center}
\begin{tabular}{cc r@{\,\(\pm\)\,}l r@{\,\(\pm\)\,}l r@{\,\(\pm\)\,}l r@{\,\(\pm\)\,}l r@{\,\(\pm\)\,}l}
\toprule
\textbf{Model} & \makecell{\textbf{Code-} \\ \textbf{books}} & 
\multicolumn{2}{c}{\makecell{\textbf{AQLM} \\ \textbf{(Ind.)}}} & 
\multicolumn{2}{c}{\makecell{\textbf{AQLM} \\ \textbf{(Drop)}}} & 
\multicolumn{2}{c}{\makecell{\textbf{DbyD} \\ \textbf{(Uni.)}}} & 
\multicolumn{2}{c}{\makecell{\textbf{DbyD} \\ \textbf{(35W)}}} & 
\multicolumn{2}{c}{\makecell{\textbf{DbyD} \\ \textbf{(Had.+35W)}}} \\
\midrule
\multirow{3}{*}{\makecell{Qwen \\ 0.5B}}
& $5 \times 8$ & 32.34 & 1.37 & 32.34 & 1.37 & 32.85 & 1.37 & 32.88 & 1.37 & 33.11 & 1.37 \\
& $4 \times 8$ & 32.39 & 1.37 & 31.71 & 1.36 & 32.20 & 1.37 & 31.62 & 1.36 & 31.31 & 1.36 \\
& $3 \times 8$ & 28.81 & 1.33 & 27.16 & 1.30 & 27.11 & 1.30 & 28.64 & 1.32 & 29.12 & 1.33 \\
\midrule
\multirow{3}{*}{\makecell{Qwen \\ 3B}}
& $5 \times 8$ & 46.36 & 1.46 & 46.36 & 1.46 & 47.30 & 1.46 & 46.70 & 1.46 & 47.19 & 1.46 \\
& $4 \times 8$ & 45.76 & 1.46 & 41.90 & 1.44 & 44.77 & 1.45 & 47.13 & 1.46 & 47.78 & 1.46 \\
& $3 \times 8$ & 44.63 & 1.45 & 23.26 & 1.24 & 41.41 & 1.44 & 40.90 & 1.44 & 44.40 & 1.45 \\
\midrule
\multirow{3}{*}{\makecell{Qwen \\ 7B}}
& $5 \times 8$ & 56.46 & 1.45 & 56.46 & 1.45 & 55.69 & 1.45 & 56.31 & 1.45 & 54.55 & 1.46 \\
& $4 \times 8$ & 55.12 & 1.45 & 55.17 & 1.45 & 54.10 & 1.46 & 56.94 & 1.45 & 53.13 & 1.46 \\
& $3 \times 8$ & 54.44 & 1.45 & 51.14 & 1.46 & 49.57 & 1.46 & 53.47 & 1.46 & 49.89 & 1.46 \\
\midrule
\multirow{3}{*}{\makecell{Qwen \\ 14B}}
& $5 \times 8$ & 61.92 & 1.42 & 61.92 & 1.42 & 60.61 & 1.43 & 62.17 & 1.42 & 62.34 & 1.41 \\
& $4 \times 8$ & 60.49 & 1.43 & 58.02 & 1.44 & 60.89 & 1.43 & 60.84 & 1.42 & 60.78 & 1.43 \\
& $3 \times 8$ & 59.16 & 1.44 & 51.20 & 1.46 & 58.16 & 1.44 & 58.82 & 1.44 & 59.24 & 1.44 \\
\midrule
\multirow{3}{*}{\makecell{Qwen \\ 32B}}
& $5 \times 8$ & 58.05 & 1.44 & 58.05 & 1.44 & 57.08 & 1.45 & 57.93 & 1.44 & 58.59 & 1.44 \\
& $4 \times 8$ & 57.17 & 1.45 & 55.57 & 1.45 & 54.18 & 1.46 & 57.54 & 1.44 & 59.36 & 1.43 \\
& $3 \times 8$ & 58.16 & 1.44 & 52.67 & 1.46 & 54.44 & 1.45 & 56.26 & 1.45 & 57.82 & 1.44 \\
\bottomrule
\end{tabular}
\end{center}
\label{tab:arc_challenge}
\end{table}

\begin{table}[h!]
\small
\caption{Zero-shot accuracy on arc\_easy (\%) ($\uparrow$)}
\begin{center}
\begin{tabular}{cc r@{\,\(\pm\)\,}l r@{\,\(\pm\)\,}l r@{\,\(\pm\)\,}l r@{\,\(\pm\)\,}l r@{\,\(\pm\)\,}l}
\toprule
\textbf{Model} & \makecell{\textbf{Code-} \\ \textbf{books}} & 
\multicolumn{2}{c}{\makecell{\textbf{AQLM} \\ \textbf{(Ind.)}}} & 
\multicolumn{2}{c}{\makecell{\textbf{AQLM} \\ \textbf{(Drop)}}} & 
\multicolumn{2}{c}{\makecell{\textbf{DbyD} \\ \textbf{(Uni.)}}} & 
\multicolumn{2}{c}{\makecell{\textbf{DbyD} \\ \textbf{(35W)}}} & 
\multicolumn{2}{c}{\makecell{\textbf{DbyD} \\ \textbf{(Had.+35W)}}} \\
\midrule
\multirow{3}{*}{\makecell{Qwen \\ 0.5B}}
& $5 \times 8$ & 57.00 & 1.02 & 57.00 & 1.02 & 60.35 & 1.00 & 59.46 & 1.01 & 59.82 & 1.00 \\
& $4 \times 8$ & 60.06 & 1.00 & 57.76 & 1.01 & 57.43 & 1.01 & 60.97 & 1.00 & 56.76 & 1.01 \\
& $3 \times 8$ & 53.65 & 1.02 & 41.79 & 1.02 & 48.81 & 1.03 & 53.75 & 1.02 & 49.35 & 1.03 \\
\midrule
\multirow{3}{*}{\makecell{Qwen \\ 3B}}
& $5 \times 8$ & 72.50 & 0.92 & 72.50 & 0.92 & 73.82 & 0.90 & 73.91 & 0.90 & 73.53 & 0.90 \\
& $4 \times 8$ & 70.19 & 0.94 & 64.56 & 0.97 & 69.62 & 0.94 & 72.78 & 0.91 & 74.15 & 0.90 \\
& $3 \times 8$ & 70.90 & 0.93 & 38.06 & 0.99 & 67.16 & 0.96 & 67.33 & 0.96 & 69.04 & 0.95 \\
\midrule
\multirow{3}{*}{\makecell{Qwen \\ 7B}}
& $5 \times 8$ & 81.23 & 0.80 & 81.23 & 0.80 & 79.04 & 0.84 & 81.40 & 0.80 & 80.56 & 0.81 \\
& $4 \times 8$ & 78.77 & 0.84 & 77.67 & 0.85 & 76.17 & 0.88 & 79.25 & 0.83 & 75.94 & 0.88 \\
& $3 \times 8$ & 75.76 & 0.88 & 71.66 & 0.92 & 69.76 & 0.94 & 74.71 & 0.89 & 72.96 & 0.91 \\
\midrule
\multirow{3}{*}{\makecell{Qwen \\ 14B}}
& $5 \times 8$ & 81.18 & 0.80 & 81.18 & 0.80 & 78.97 & 0.84 & 80.50 & 0.81 & 81.99 & 0.79 \\
& $4 \times 8$ & 79.01 & 0.84 & 77.33 & 0.86 & 79.88 & 0.82 & 80.12 & 0.80 & 82.29 & 0.78 \\
& $3 \times 8$ & 77.21 & 0.86 & 69.95 & 0.94 & 77.61 & 0.85 & 77.40 & 0.85 & 80.88 & 0.81 \\
\midrule
\multirow{3}{*}{\makecell{Qwen \\ 32B}}
& $5 \times 8$ & 75.59 & 0.88 & 75.59 & 0.88 & 75.17 & 0.89 & 76.19 & 0.87 & 77.16 & 0.86 \\
& $4 \times 8$ & 75.79 & 0.88 & 74.13 & 0.90 & 72.01 & 0.92 & 74.83 & 0.89 & 78.17 & 0.85 \\
& $3 \times 8$ & 77.03 & 0.86 & 70.45 & 0.94 & 73.36 & 0.90 & 74.45 & 0.89 & 77.19 & 0.86 \\
\bottomrule
\end{tabular}
\end{center}
\label{tab:arc_easy}
\end{table}

\begin{table}[h!]
\small
\caption{Zero-shot accuracy on hellaswag (\%) ($\uparrow$)}
\begin{center}
\begin{tabular}{cc r@{\,\(\pm\)\,}l r@{\,\(\pm\)\,}l r@{\,\(\pm\)\,}l r@{\,\(\pm\)\,}l r@{\,\(\pm\)\,}l}
\toprule
\textbf{Model} & \makecell{\textbf{Code-} \\ \textbf{books}} & 
\multicolumn{2}{c}{\makecell{\textbf{AQLM} \\ \textbf{(Ind.)}}} & 
\multicolumn{2}{c}{\makecell{\textbf{AQLM} \\ \textbf{(Drop)}}} & 
\multicolumn{2}{c}{\makecell{\textbf{DbyD} \\ \textbf{(Uni.)}}} & 
\multicolumn{2}{c}{\makecell{\textbf{DbyD} \\ \textbf{(35W)}}} & 
\multicolumn{2}{c}{\makecell{\textbf{DbyD} \\ \textbf{(Had.+35W)}}} \\
\midrule
\multirow{3}{*}{\makecell{Qwen \\ 0.5B}}
& $5 \times 8$ & 51.57 & 0.50 & 51.57 & 0.50 & 50.42 & 0.50 & 51.20 & 0.50 & 51.48 & 0.50 \\
& $4 \times 8$ & 49.78 & 0.50 & 49.13 & 0.50 & 47.34 & 0.50 & 47.97 & 0.50 & 49.21 & 0.50 \\
& $3 \times 8$ & 45.65 & 0.50 & 38.59 & 0.49 & 39.21 & 0.49 & 41.95 & 0.49 & 43.50 & 0.49 \\
\midrule
\multirow{3}{*}{\makecell{Qwen \\ 3B}}
& $5 \times 8$ & 73.09 & 0.44 & 73.09 & 0.44 & 72.56 & 0.45 & 72.85 & 0.44 & 73.01 & 0.44 \\
& $4 \times 8$ & 72.05 & 0.45 & 64.62 & 0.47 & 70.20 & 0.46 & 71.15 & 0.45 & 71.68 & 0.45 \\
& $3 \times 8$ & 68.27 & 0.46 & 29.61 & 0.45 & 65.84 & 0.47 & 67.19 & 0.47 & 67.10 & 0.47 \\
\midrule
\multirow{3}{*}{\makecell{Qwen \\ 7B}}
& $5 \times 8$ & 80.00 & 0.40 & 80.00 & 0.40 & 79.62 & 0.40 & 80.00 & 0.40 & 80.17 & 0.40 \\
& $4 \times 8$ & 79.46 & 0.40 & 78.91 & 0.41 & 78.72 & 0.41 & 78.97 & 0.41 & 79.45 & 0.40 \\
& $3 \times 8$ & 77.30 & 0.42 & 73.63 & 0.44 & 75.35 & 0.43 & 76.17 & 0.43 & 76.74 & 0.42 \\
\midrule
\multirow{3}{*}{\makecell{Qwen \\ 14B}}
& $5 \times 8$ & 84.22 & 0.36 & 84.22 & 0.36 & 83.80 & 0.37 & 83.93 & 0.37 & 84.08 & 0.36 \\
& $4 \times 8$ & 83.74 & 0.37 & 83.66 & 0.37 & 83.28 & 0.37 & 83.38 & 0.37 & 83.57 & 0.37 \\
& $3 \times 8$ & 81.95 & 0.38 & 80.13 & 0.40 & 81.92 & 0.38 & 81.31 & 0.39 & 81.51 & 0.39 \\
\midrule
\multirow{3}{*}{\makecell{Qwen \\ 32B}}
& $5 \times 8$ & 85.11 & 0.36 & 85.11 & 0.36 & 84.82 & 0.36 & 84.92 & 0.36 & 84.91 & 0.36 \\
& $4 \times 8$ & 84.90 & 0.36 & 84.88 & 0.36 & 84.36 & 0.36 & 84.72 & 0.36 & 84.79  & 0.36 \\
& $3 \times 8$ & 83.79 & 0.37 & 83.50 & 0.37 & 82.40 & 0.38 & 83.09 & 0.37 & 83.48 & 0.37 \\
\bottomrule
\end{tabular}
\end{center}
\label{tab:hellaswag}
\end{table}

\begin{table}[h!]
\small
\caption{Zero-shot accuracy on piqa (\%) ($\uparrow$)}
\begin{center}
\begin{tabular}{cc r@{\,\(\pm\)\,}l r@{\,\(\pm\)\,}l r@{\,\(\pm\)\,}l r@{\,\(\pm\)\,}l r@{\,\(\pm\)\,}l}
\toprule
\textbf{Model} & \makecell{\textbf{Code-} \\ \textbf{books}} & 
\multicolumn{2}{c}{\makecell{\textbf{AQLM} \\ \textbf{(Ind.)}}} & 
\multicolumn{2}{c}{\makecell{\textbf{AQLM} \\ \textbf{(Drop)}}} & 
\multicolumn{2}{c}{\makecell{\textbf{DbyD} \\ \textbf{(Uni.)}}} & 
\multicolumn{2}{c}{\makecell{\textbf{DbyD} \\ \textbf{(35W)}}} & 
\multicolumn{2}{c}{\makecell{\textbf{DbyD} \\ \textbf{(Had.+35W)}}} \\
\midrule
\multirow{3}{*}{\makecell{Qwen \\ 0.5B}}
& $5 \times 8$ & 69.48 & 1.07 & 69.48 & 1.07 & 69.10 & 1.08 & 69.82 & 1.07 & 70.09 & 1.07 \\
& $4 \times 8$ & 69.47 & 1.07 & 68.92 & 1.08 & 67.18 & 1.09 & 67.95 & 1.09 & 68.42 & 1.08 \\
& $3 \times 8$ & 66.14 & 1.10 & 61.03 & 1.14 & 62.02 & 1.13 & 64.31 & 1.11 & 65.07 & 1.11 \\
\midrule
\multirow{3}{*}{\makecell{Qwen \\ 3B}}
& $5 \times 8$ & 78.75 & 0.95 & 78.75 & 0.95 & 78.15 & 0.96 & 78.71 & 0.96 & 78.22 & 0.96 \\
& $4 \times 8$ & 77.60 & 0.97 & 74.61 & 1.01 & 76.93 & 0.98 & 77.64 & 0.97 & 77.62 & 0.97 \\
& $3 \times 8$ & 76.19 & 0.99 & 55.84 & 1.16 & 74.72 & 1.01 & 75.10 & 1.01 & 75.82 & 1.00 \\
\midrule
\multirow{3}{*}{\makecell{Qwen \\ 7B}}
& $5 \times 8$ & 80.10 & 0.93 & 80.10 & 0.93 & 79.67 & 0.94 & 80.16 & 0.93 & 80.01 & 0.93 \\
& $4 \times 8$ & 79.22 & 0.95 & 77.64 & 0.97 & 78.18 & 0.96 & 79.49 & 0.94 & 79.44 & 0.94 \\
& $3 \times 8$ & 77.79 & 0.97 & 74.61 & 1.01 & 76.30 & 0.99 & 76.66 & 0.99 & 76.95 & 0.98 \\
\midrule
\multirow{3}{*}{\makecell{Qwen \\ 14B}}
& $5 \times 8$ & 81.50 & 0.90 & 81.50 & 0.90 & 81.47 & 0.91 & 81.61 & 0.90 & 81.95 & 0.90 \\
& $4 \times 8$ & 81.25 & 0.91 & 80.87 & 0.92 & 81.03 & 0.91 & 80.96 & 0.92 & 81.12 & 0.91 \\
& $3 \times 8$ & 80.52 & 0.93 & 80.05 & 0.93 & 79.54 & 0.94 & 80.16 & 0.93 & 80.76 & 0.92 \\
\midrule
\multirow{3}{*}{\makecell{Qwen \\ 32B}}
& $5 \times 8$ & 81.26 & 0.91 & 81.26 & 0.91 & 81.10 & 0.91 & 81.28 & 0.91 & 81.10 & 0.91 \\
& $4 \times 8$ & 80.90 & 0.92 & 80.18 & 0.93 & 79.60 & 0.94 & 80.07 & 0.93 & 80.96 & 0.92 \\
& $3 \times 8$ & 81.17 & 0.91 & 79.47 & 0.94 & 78.75 & 0.95 & 78.98 & 0.95 & 79.69 & 0.94 \\
\bottomrule
\end{tabular}
\end{center}
\label{tab:piqa}
\end{table}

\begin{table}[h!]
\small
\caption{Zero-shot accuracy on winogrande (\%) ($\uparrow$)}
\begin{center}
\begin{tabular}{cc r@{\,\(\pm\)\,}l r@{\,\(\pm\)\,}l r@{\,\(\pm\)\,}l r@{\,\(\pm\)\,}l r@{\,\(\pm\)\,}l}
\toprule
\textbf{Model} & \makecell{\textbf{Code-} \\ \textbf{books}} & 
\multicolumn{2}{c}{\makecell{\textbf{AQLM} \\ \textbf{(Ind.)}}} & 
\multicolumn{2}{c}{\makecell{\textbf{AQLM} \\ \textbf{(Drop)}}} & 
\multicolumn{2}{c}{\makecell{\textbf{DbyD} \\ \textbf{(Uni.)}}} & 
\multicolumn{2}{c}{\makecell{\textbf{DbyD} \\ \textbf{(35W)}}} & 
\multicolumn{2}{c}{\makecell{\textbf{DbyD} \\ \textbf{(Had.+35W)}}} \\
\midrule
\multirow{3}{*}{\makecell{Qwen \\ 0.5B}}
& $5 \times 8$ & 55.85 & 1.39 & 55.85 & 1.39 & 55.38 & 1.40 & 55.96 & 1.39 & 56.35 & 1.40 \\
& $4 \times 8$ & 56.30 & 1.39 & 56.48 & 1.40 & 55.17 & 1.40 & 57.17 & 1.39 & 56.30 & 1.39 \\
& $3 \times 8$ & 55.57 & 1.40 & 51.54 & 1.40 & 52.25 & 1.40 & 53.70 & 1.40 & 54.75 & 1.40 \\
\midrule
\multirow{3}{*}{\makecell{Qwen \\ 3B}}
& $5 \times 8$ & 68.46 & 1.30 & 68.46 & 1.30 & 69.06 & 1.30 & 68.69 & 1.30 & 68.61 & 1.31 \\
& $4 \times 8$ & 67.80 & 1.31 & 64.88 & 1.34 & 67.72 & 1.32 & 67.69 & 1.32 & 67.77 & 1.31 \\
& $3 \times 8$ & 67.40 & 1.32 & 51.43 & 1.40 & 63.51 & 1.35 & 64.77 & 1.34 & 65.25 & 1.34 \\
\midrule
\multirow{3}{*}{\makecell{Qwen \\ 7B}}
& $5 \times 8$ & 70.14 & 1.28 & 70.14 & 1.28 & 69.75 & 1.29 & 70.72 & 1.28 & 71.12 & 1.28 \\
& $4 \times 8$ & 69.27 & 1.30 & 67.61 & 1.31 & 69.43 & 1.29 & 67.93 & 1.31 & 68.88 & 1.30 \\
& $3 \times 8$ & 68.85 & 1.30 & 62.83 & 1.36 & 65.54 & 1.34 & 66.69 & 1.32 & 66.74 & 1.33 \\
\midrule
\multirow{3}{*}{\makecell{Qwen \\ 14B}}
& $5 \times 8$ & 75.45 & 1.21 & 75.45 & 1.21 & 75.30 & 1.21 & 75.59 & 1.21 & 75.64 & 1.20 \\
& $4 \times 8$ & 75.19 & 1.21 & 73.77 & 1.23 & 73.93 & 1.23 & 74.64 & 1.22 & 75.40 & 1.21 \\
& $3 \times 8$ & 75.61 & 1.21 & 68.69 & 1.30 & 72.37 & 1.25 & 71.98 & 1.26 & 72.95 & 1.25 \\
\midrule
\multirow{3}{*}{\makecell{Qwen \\ 32B}}
& $5 \times 8$ & 72.64 & 1.25 & 72.64 & 1.25 & 72.85 & 1.25 & 73.40 & 1.24 & 73.14 & 1.25 \\
& $4 \times 8$ & 73.74 & 1.24 & 70.72 & 1.28 & 72.66 & 1.25 & 72.72 & 1.25 & 72.85 & 1.25 \\
& $3 \times 8$ & 74.53 & 1.22 & 69.32 & 1.29 & 70.69 & 1.28 & 71.90 & 1.26 & 71.24 & 1.27 \\
\bottomrule
\end{tabular}
\end{center}
\label{tab:winogrande}
\end{table}

\clearpage
\subsection{Detailed Accuracy Performance Results: Gemma, Mistral, and Llama}
\label{app:other_models}
In this section, we provide the detailed zero-shot accuracy breakdowns for the select Gemma, Mistral, and Llama models across five tasks: ARC Challenge, ARC Easy, HellaSwag, PIQA, and Winogrande.

\begin{table}[h!]
\small
\caption{Zero-shot accuracy on arc\_challenge (\%) ($\uparrow$)}
\begin{center}
\begin{tabular}{cl r@{\,\(\pm\)\,}l r@{\,\(\pm\)\,}l r@{\,\(\pm\)\,}l r@{\,\(\pm\)\,}l r@{\,\(\pm\)\,}l}
\toprule
\textbf{Model} & \makecell{\textbf{Code-} \\ \textbf{books}} & 
\multicolumn{2}{c}{\makecell{\textbf{AQLM} \\ \textbf{(Ind.)}}} & 
\multicolumn{2}{c}{\makecell{\textbf{AQLM} \\ \textbf{(Drop)}}} & 
\multicolumn{2}{c}{\makecell{\textbf{DbyD} \\ \textbf{(Uni.)}}} & 
\multicolumn{2}{c}{\makecell{\textbf{DbyD} \\ \textbf{(35W)}}} & 
\multicolumn{2}{c}{\makecell{\textbf{DbyD} \\ \textbf{(Had.+35W)}}} \\
\midrule
\multirow{3}{*}{\makecell{Gemma \\ 2B}}
  & $5 \times 8$ & 41.89 & 1.44 & 41.89 & 1.44 & 41.64 & 1.44 & 42.24 & 1.44 & 40.87 & 1.44 \\
  & $4 \times 8$ & 40.27 & 1.43 & 41.64 & 1.44 & 38.31 & 1.42 & 40.27 & 1.43 & 39.22 & 1.43 \\
  & $3 \times 8$ & 35.49 & 1.40 & 38.91 & 1.42 & 32.17 & 1.37 & 36.83 & 1.41 & 36.32 & 1.40 \\
\midrule
\multirow{3}{*}{\makecell{Mistral \\ 7B}}
  & $5 \times 8$ & 50.28 & 1.46 & 50.28 & 1.46 & 50.45 & 1.46 & 52.02 & 1.46 & 51.59 & 1.46 \\
  & $4 \times 8$ & 50.97 & 1.46 & 48.30 & 1.46 & 48.69 & 1.46 & 49.12 & 1.46 & 50.54 & 1.46 \\
  & $3 \times 8$ & 49.09 & 1.46 & 26.85 & 1.29 & 40.93 & 1.43 & 44.31 & 1.45 & 47.18 & 1.46 \\
\midrule
\multirow{3}{*}{\makecell{MetaLlama \\ 8B}}
  & $5 \times 8$ & 52.19 & 1.46 & 52.19 & 1.46 & 50.08 & 1.46 & 51.82 & 1.46 & 53.47 & 1.46 \\
  & $4 \times 8$ & 49.72 & 1.46 & 43.83 & 1.45 & 46.73 & 1.46 & 48.09 & 1.46 & 52.30 & 1.46 \\
  & $3 \times 8$ & 49.20 & 1.46 & 23.15 & 1.23 & 33.47 & 1.38 & 41.67 & 1.44 & 45.48 & 1.45 \\
\bottomrule
\end{tabular}
\end{center}
\label{tab:results_arc_main}
\end{table}

\begin{table}[h!]
\small
\caption{Zero-shot accuracy on arc\_easy (\%) ($\uparrow$)}
\begin{center}
\begin{tabular}{cl r@{\,\(\pm\)\,}l r@{\,\(\pm\)\,}l r@{\,\(\pm\)\,}l r@{\,\(\pm\)\,}l r@{\,\(\pm\)\,}l}
\toprule
\textbf{Model} & \makecell{\textbf{Code-} \\ \textbf{books}} & 
\multicolumn{2}{c}{\makecell{\textbf{AQLM} \\ \textbf{(Ind.)}}} & 
\multicolumn{2}{c}{\makecell{\textbf{AQLM} \\ \textbf{(Drop)}}} & 
\multicolumn{2}{c}{\makecell{\textbf{DbyD} \\ \textbf{(Uni.)}}} & 
\multicolumn{2}{c}{\makecell{\textbf{DbyD} \\ \textbf{(35W)}}} & 
\multicolumn{2}{c}{\makecell{\textbf{DbyD} \\ \textbf{(Had.+35W)}}} \\
\midrule
\multirow{3}{*}{\makecell{Gemma \\ 2B}}
  & $5 \times 8$ & 72.56 & 0.92 & 72.56 & 0.92 & 71.34 & 0.93 & 72.11 & 0.92 & 71.62 & 0.93 \\
  & $4 \times 8$ & 68.86 & 0.95 & 70.08 & 0.94 & 67.34 & 0.96 & 69.61 & 0.94 & 69.78 & 0.94 \\
  & $3 \times 8$ & 63.85 & 0.99 & 65.28 & 0.98 & 53.41 & 1.02 & 64.31 & 0.98 & 65.24 & 0.98 \\
\midrule
\multirow{3}{*}{\makecell{Mistral \\ 7B}}
  & $5 \times 8$ & 76.49 & 0.87 & 76.49 & 0.87 & 76.57 & 0.87 & 77.51 & 0.86 & 78.27 & 0.84 \\
  & $4 \times 8$ & 78.70 & 0.84 & 74.82 & 0.89 & 75.14 & 0.89 & 75.17 & 0.89 & 77.32 & 0.86 \\
  & $3 \times 8$ & 76.99 & 0.86 & 43.04 & 1.00 & 65.02 & 0.98 & 72.47 & 0.92 & 75.19 & 0.89 \\
\midrule
\multirow{3}{*}{\makecell{MetaLlama \\ 8B}}
  & $5 \times 8$ & 76.50 & 0.87 & 76.50 & 0.87 & 75.62 & 0.88 & 75.98 & 0.88 & 80.89 & 0.81 \\
  & $4 \times 8$ & 74.19 & 0.90 & 68.80 & 0.95 & 70.87 & 0.93 & 72.22 & 0.92 & 78.89 & 0.84 \\
  & $3 \times 8$ & 73.76 & 0.90 & 29.57 & 0.94 & 51.60 & 1.01 & 65.77 & 0.97 & 72.59 & 0.92 \\
\bottomrule
\end{tabular}
\end{center}
\label{tab:results_easy_main}
\end{table}

\begin{table}[h!]
\small
\caption{Zero-shot accuracy on hellaswag (\%) ($\uparrow$)}
\begin{center}
\begin{tabular}{cl r@{\,\(\pm\)\,}l r@{\,\(\pm\)\,}l r@{\,\(\pm\)\,}l r@{\,\(\pm\)\,}l r@{\,\(\pm\)\,}l}

\toprule
\textbf{Model} & \makecell{\textbf{Code-} \\ \textbf{books}} & 
\multicolumn{2}{c}{\makecell{\textbf{AQLM} \\ \textbf{(Ind.)}}} & 
\multicolumn{2}{c}{\makecell{\textbf{AQLM} \\ \textbf{(Drop)}}} & 
\multicolumn{2}{c}{\makecell{\textbf{DbyD} \\ \textbf{(Uni.)}}} & 
\multicolumn{2}{c}{\makecell{\textbf{DbyD} \\ \textbf{(35W)}}} & 
\multicolumn{2}{c}{\makecell{\textbf{DbyD} \\ \textbf{(Had.+35W)}}} \\
\midrule
\multirow{3}{*}{\makecell{Gemma \\ 2B}}
  & $5 \times 8$ & 70.87 & 0.45 & 70.87 & 0.45 & 68.59 & 0.46 & 70.59 & 0.45 & 70.73 & 0.45 \\
  & $4 \times 8$ & 69.18 & 0.46 & 69.27 & 0.46 & 66.96 & 0.47 & 68.99 & 0.46 & 68.70 & 0.46 \\
  & $3 \times 8$ & 64.52 & 0.48 & 63.65 & 0.48 & 62.26 & 0.48 & 64.94 & 0.48 & 63.48 & 0.48 \\
\midrule
\multirow{3}{*}{\makecell{Mistral \\ 7B}}
  & $5 \times 8$ & 79.87 & 0.40 & 79.87 & 0.40 & 79.89 & 0.40 & 80.28 & 0.40 & 80.29 & 0.40 \\
  & $4 \times 8$ & 79.95 & 0.40 & 77.65 & 0.42 & 78.60 & 0.41 & 79.14 & 0.41 & 79.73 & 0.40 \\
  & $3 \times 8$ & 77.92 & 0.41 & 39.56 & 0.48 & 66.14 & 0.47 & 75.81 & 0.43 & 77.67 & 0.42 \\
\midrule
\multirow{3}{*}{\makecell{MetaLlama \\ 8B}}
  & $5 \times 8$ & 78.86 & 0.41 & 78.86 & 0.41 & 78.19 & 0.41 & 78.89 & 0.41 & 78.58 & 0.41 \\
  & $4 \times 8$ & 77.83 & 0.41 & 69.66 & 0.45 & 76.15 & 0.43 & 76.35 & 0.42 & 77.97 & 0.41 \\
  & $3 \times 8$ & 75.91 & 0.43 & 31.47 & 0.46 & 63.93 & 0.48 & 72.82 & 0.44 & 73.35 & 0.44 \\
\bottomrule
\end{tabular}
\end{center}
\label{tab:results_hellaswag_main}
\end{table}

\begin{table}[h!]
\small
\caption{Zero-shot accuracy on piqa (\%) ($\uparrow$)}
\begin{center}
\begin{tabular}{cl r@{\,\(\pm\)\,}l r@{\,\(\pm\)\,}l r@{\,\(\pm\)\,}l r@{\,\(\pm\)\,}l r@{\,\(\pm\)\,}l}

\toprule
\textbf{Model} & \makecell{\textbf{Code-} \\ \textbf{books}} & 
\multicolumn{2}{c}{\makecell{\textbf{AQLM} \\ \textbf{(Ind.)}}} & 
\multicolumn{2}{c}{\makecell{\textbf{AQLM} \\ \textbf{(Drop)}}} & 
\multicolumn{2}{c}{\makecell{\textbf{DbyD} \\ \textbf{(Uni.)}}} & 
\multicolumn{2}{c}{\makecell{\textbf{DbyD} \\ \textbf{(35W)}}} & 
\multicolumn{2}{c}{\makecell{\textbf{DbyD} \\ \textbf{(Had.+35W)}}} \\
\midrule
\multirow{3}{*}{\makecell{Gemma \\ 2B}}
  & $5 \times 8$ & 77.58 & 0.97 & 77.58 & 0.97 & 77.42 & 0.98 & 77.62 & 0.97 & 77.91 & 0.97 \\
  & $4 \times 8$ & 77.04 & 0.98 & 76.77 & 0.99 & 76.17 & 0.99 & 76.53 & 0.99 & 77.06 & 0.98 \\
  & $3 \times 8$ & 74.92 & 1.01 & 75.35 & 1.01 & 74.43 & 1.02 & 74.94 & 1.01 & 74.90 & 1.01 \\
\midrule
\multirow{3}{*}{\makecell{Mistral \\ 7B}}
  & $5 \times 8$ & 80.76 & 0.92 & 80.76 & 0.92 & 81.59 & 0.90 & 81.35 & 0.91 & 81.66 & 0.90 \\
  & $4 \times 8$ & 81.07 & 0.91 & 79.61 & 0.94 & 79.87 & 0.93 & 80.29 & 0.93 & 81.40 & 0.90 \\
  & $3 \times 8$ & 80.38 & 0.92 & 62.80 & 1.12 & 75.28 & 1.01 & 77.84 & 0.97 & 80.11 & 0.93 \\
\midrule
\multirow{3}{*}{\makecell{MetaLlama \\ 8B}}
  & $5 \times 8$ & 80.36 & 0.93 & 80.36 & 0.93 & 80.05 & 0.93 & 80.05 & 0.93 & 80.70 & 0.92 \\
  & $4 \times 8$ & 79.40 & 0.95 & 75.24 & 1.01 & 78.45 & 0.96 & 78.91 & 0.95 & 80.31 & 0.93 \\
  & $3 \times 8$ & 78.46 & 0.96 & 53.77 & 1.16 & 67.39 & 1.09 & 74.66 & 1.02 & 77.97 & 0.97 \\
\bottomrule
\end{tabular}
\end{center}
\label{tab:results_piqa_main}
\end{table}

\begin{table}[h!]
\small
\caption{Zero-shot accuracy on winogrande (\%) ($\uparrow$)}
\begin{center}
\begin{tabular}{cl r@{\,\(\pm\)\,}l r@{\,\(\pm\)\,}l r@{\,\(\pm\)\,}l r@{\,\(\pm\)\,}l r@{\,\(\pm\)\,}l}

\toprule
\textbf{Model} & \makecell{\textbf{Code-} \\ \textbf{books}} & 
\multicolumn{2}{c}{\makecell{\textbf{AQLM} \\ \textbf{(Ind.)}}} & 
\multicolumn{2}{c}{\makecell{\textbf{AQLM} \\ \textbf{(Drop)}}} & 
\multicolumn{2}{c}{\makecell{\textbf{DbyD} \\ \textbf{(Uni.)}}} & 
\multicolumn{2}{c}{\makecell{\textbf{DbyD} \\ \textbf{(35W)}}} & 
\multicolumn{2}{c}{\makecell{\textbf{DbyD} \\ \textbf{(Had.+35W)}}} \\
\midrule
\multirow{3}{*}{\makecell{Gemma \\ 2B}}
  & $5 \times 8$ & 65.04 & 1.34 & 65.04 & 1.34 & 64.09 & 1.35 & 64.85 & 1.34 & 64.54 & 1.34 \\
  & $4 \times 8$ & 64.40 & 1.35 & 64.33 & 1.35 & 63.14 & 1.36 & 64.09 & 1.35 & 62.64 & 1.36 \\
  & $3 \times 8$ & 63.06 & 1.36 & 62.90 & 1.36 & 60.30 & 1.38 & 62.04 & 1.36 & 60.93 & 1.37 \\
\midrule
\multirow{3}{*}{\makecell{Mistral \\ 7B}}
  & $5 \times 8$ & 71.19 & 1.27 & 71.19 & 1.27 & 72.27 & 1.26 & 73.11 & 1.25 & 73.59 & 1.24 \\
  & $4 \times 8$ & 71.61 & 1.27 & 69.30 & 1.30 & 67.51 & 1.32 & 70.35 & 1.29 & 72.61 & 1.25 \\
  & $3 \times 8$ & 70.45 & 1.28 & 51.96 & 1.40 & 61.54 & 1.36 & 66.95 & 1.32 & 70.82 & 1.27 \\
\midrule
\multirow{3}{*}{\makecell{MetaLlama \\ 8B}}
  & $5 \times 8$ & 74.19 & 1.23 & 74.19 & 1.23 & 74.17 & 1.23 & 73.72 & 1.24 & 73.22 & 1.24 \\
  & $4 \times 8$ & 73.85 & 1.24 & 70.16 & 1.28 & 70.40 & 1.28 & 72.43 & 1.25 & 72.35 & 1.26 \\
  & $3 \times 8$ & 72.67 & 1.25 & 52.46 & 1.40 & 62.61 & 1.36 & 69.14 & 1.30 & 69.67 & 1.29 \\
\bottomrule
\end{tabular}
\end{center}
\label{tab:results_winogrande_main}
\end{table}

\clearpage
\section{Quantitative SOTA Results}
\subsection{Perplexity Values for the SOTA Methods}
\label{app:ppldetailedsota}
This section compares the proposed Drop-by-Drop (DbyD) method, specifically the DbyD (Had.+35W) variant, against the closest state-of-the-art post-training quantization methods that support multi-bitwidth inference, namely AnyBCQ, Any Precision LLM (AnyPrecLLM), and MatGPTQ~\cite{anybcq,anyprecisionllm,matgptq}. All methods are evaluated at average precisions of 3, 4, and 5 bits on the \textsc{WikiText2} and \textsc{C4} benchmarks across a diverse model architectures and scales. Tables~\ref{tab:ppl_wiki_combined_sota_3to5} and~\ref{tab:ppl_c4_combined_sota_3to5} report the mean perplexity and standard deviation over \textit{three} independent runs.

Across the \textsc{Qwen} family, DbyD achieves the best perplexity at every evaluated precision. On the broader collection of model families, DbyD consistently outperforms all three baselines on \textsc{Mistral-7B-v0.3} and \textsc{Meta-Llama-3-8B} across all evaluated precisions. On \textsc{Gemma-2B}, DbyD achieves the lowest \textsc{C4} perplexity at every precision, while AnyPrecLLM achieves lower \textsc{WikiText2} perplexity at 3 bits.

Let us clarify certain notation in these tables. The \textit{nan} entries for \textsc{Qwen2.5-7B} arise because AnyBCQ's quantization, run with the authors' default settings, produced a NaN reconstruction loss in the final decoder layer; the resulting checkpoint yields NaN perplexity and chance-level zero-shot accuracy. MatGPTQ entries for \textsc{Qwen2.5-7B} are omitted (denoted as \textit{--}): Cholesky factorization of the layer-wise Hessian failed to converge even under increased damping (\texttt{rel\_damp}), and we were unable to obtain a quantized model at this scale. Also, for \textsc{Qwen2.5-32B}, AnyBCQ was quantized with 1,500 reconstruction iterations rather than the default 5,000, since we observed that the reconstruction loss consistently plateaued before 1,500 iterations. Runs at the default setting terminated during intermediate-layer reconstruction, whereas the reduced configuration completed successfully.

\begin{table}[H]
\small
\centering
\caption{State-of-the-art perplexity comparison ($\downarrow$) on \textsc{C4}. Values condensed in Table~\ref{tab:dense-sota}.}
    \begin{tabular}{lc r@{\;\(\pm\)\;}l r@{\;\(\pm\)\;}l r@{\;\(\pm\)\;}l r@{\;\(\pm\)\;}l r@{\;\(\pm\)\;}l}
    \toprule
    \textbf{Method} & \textbf{Bits} & 
    \multicolumn{2}{c}{\textbf{Qwen 0.5B}} & 
    \multicolumn{2}{c}{\textbf{Qwen 3B}} & 
    \multicolumn{2}{c}{\textbf{Qwen 7B}} & 
    \multicolumn{2}{c}{\textbf{Qwen 14B}} & 
    \multicolumn{2}{c}{\textbf{Qwen 32B}} \\
    \midrule
    DbyD (Had.+35W) & \multirow{4}{*}{5} & \textbf{17.87} & \textbf{0.00} & \textbf{11.32} & \textbf{0.01} & \textbf{10.51} & \textbf{0.01} & \textbf{9.06} & \textbf{0.00} & \textbf{8.79} & \textbf{0.00} \\
    AnyBCQ          &                    & 24.70          & 0.00          & 14.96          & 0.00          & \multicolumn{2}{c}{\textit{nan}} & 11.96          & 0.00          & 11.63         & 0.00       \\
    AnyPrecLLM      &                    & 22.44          & 0.02          & 13.77          & 0.01          & 12.93          & 0.03          & 11.01          & 0.00          & 10.70          & 0.00          \\
    MatGPTQ         &                    & 23.25          & 0.00          & 13.62          & 0.00          & \multicolumn{2}{c}{--}       & 11.07          & 0.00          & 10.59          & 0.00          \\
    \midrule
    DbyD (Had.+35W) & \multirow{4}{*}{4} & \textbf{19.99} & \textbf{0.06} & \textbf{11.93} & \textbf{0.01} & \textbf{10.93} & \textbf{0.03} & \textbf{9.40} & \textbf{0.00} & \textbf{8.98} & \textbf{0.00} \\
    AnyBCQ          &                    & 25.89          & 0.00          & 15.36          & 0.00          & \multicolumn{2}{c}{\textit{nan}} & 12.15          & 0.00          & 11.78        & 0.00       \\
    AnyPrecLLM      &                    & 24.15          & 0.10          & 14.26          & 0.01          & 13.33          & 0.08          & 11.28          & 0.01          & 10.89          & 0.01          \\
    MatGPTQ         &                    & 26.72          & 0.00          & 14.22          & 0.00          & \multicolumn{2}{c}{--}       & 11.47          & 0.00          & 10.81          & 0.00          \\
    \midrule
    DbyD (Had.+35W) & \multirow{4}{*}{3} & \textbf{27.67} & \textbf{1.22} & \textbf{13.96} & \textbf{0.00} & \textbf{12.27} & \textbf{0.12} & \textbf{10.38} & \textbf{0.01} & \textbf{9.49} & \textbf{0.00} \\
    AnyBCQ          &                    & 30.62          & 0.00          & 16.86          & 0.00          & \multicolumn{2}{c}{\textit{nan}} & 12.98          & 0.00          & 12.30        & 0.00     \\
    AnyPrecLLM      &                    & 34.78          & 0.33          & 16.71          & 0.08          & 15.52          & 0.68          & 12.44          & 0.05          & 11.64          & 0.03          \\
    MatGPTQ         &                    & 80.06          & 0.00          & 134.12         & 0.00          & \multicolumn{2}{c}{--}       & 13.23          & 0.00          & 11.90          & 0.00          \\
    \bottomrule
    \end{tabular}

    \vspace{0.4cm} 

    \begin{tabular}{lc r@{\;\(\pm\)\;}l r@{\;\(\pm\)\;}l r@{\;\(\pm\)\;}l}
    \toprule
    \textbf{Method} & \textbf{Bits} & 
    \multicolumn{2}{c}{\textbf{Gemma 2B}} & 
    \multicolumn{2}{c}{\textbf{Mistral 7B}} & 
    \multicolumn{2}{c}{\textbf{MetaLlama 8B}} \\
    \midrule
    DbyD (Had.+35W) & \multirow{4}{*}{5} & \textbf{11.79}         & \textbf{0.00}         & \textbf{7.54} & \textbf{0.00} & \textbf{8.63} & \textbf{0.00} \\
    AnyBCQ          &                    & 25.66          & 0.00          & 17.48          & 0.00          & 10.24          & 0.00          \\
    AnyPrecLLM      &                    & 13.67 & 0.01 & 8.52           & 0.00          & 9.68           & 0.00          \\
    MatGPTQ         &                    & 22.62          & 0.00          & 9.18           & 0.00          & 9.71           & 0.00          \\
    \midrule
    DbyD (Had.+35W) & \multirow{4}{*}{4} & \textbf{13.23}         & \textbf{0.37 }         & \textbf{7.71} & \textbf{0.00} & \textbf{9.38} & \textbf{0.04} \\
    AnyBCQ          &                    & 25.85          & 0.00          & 14.28          & 0.00          & 10.81          & 0.00          \\
    AnyPrecLLM      &                    & 14.40 & 0.05 & 8.67           & 0.01          & 10.28          & 0.00          \\
    MatGPTQ         &                    & 23.25          & 0.00          & 9.95           & 0.00          & 10.60          & 0.00          \\
    \midrule
    DbyD (Had.+35W) & \multirow{4}{*}{3} & \textbf{17.07}       & \textbf{1.15 }       & \textbf{8.32} & \textbf{0.00} & \textbf{11.53} & \textbf{0.05} \\
    AnyBCQ          &                    & 27.79          & 0.00          & 10.49          & 0.00          & 13.34          & 0.00          \\
    AnyPrecLLM      &                    & 17.58 & 0.04 & 9.39       & 0.03         & 13.14          & 0.03          \\
    MatGPTQ         &                    & 52.50          & 0.00          & 11.95          & 0.00          & 14.93          & 0.00          \\
    \bottomrule
    \end{tabular}
\label{tab:ppl_c4_combined_sota_3to5}
\end{table}

\begin{table}[H]
\small
\centering
\caption{State-of-the-art perplexity comparison ($\downarrow$) on \textsc{WikiText2}.}

    \begin{tabular}{lc r@{\;\(\pm\)\;}l r@{\;\(\pm\)\;}l r@{\;\(\pm\)\;}l r@{\;\(\pm\)\;}l r@{\;\(\pm\)\;}l}
    \toprule
    \textbf{Method} & \textbf{Bits} & 
    \multicolumn{2}{c}{\textbf{Qwen 0.5B}} & 
    \multicolumn{2}{c}{\textbf{Qwen 3B}} & 
    \multicolumn{2}{c}{\textbf{Qwen 7B}} & 
    \multicolumn{2}{c}{\textbf{Qwen 14B}} & 
    \multicolumn{2}{c}{\textbf{Qwen 32B}} \\
    \midrule
    DbyD (Had.+35W) & \multirow{4}{*}{5} & \textbf{13.36} & \textbf{0.04} & \textbf{7.57} & \textbf{0.00} & \textbf{6.93} & \textbf{0.01} & \textbf{5.38} & \textbf{0.00} & \textbf{4.99} & \textbf{0.01} \\
    AnyBCQ          &                    & 17.46          & 0.00          & 9.56           & 0.00          & \multicolumn{2}{c}{\textit{nan}} & 7.25           & 0.00          & 6.90 & 0.00      \\
    AnyPrecLLM      &                    & 14.66          & 0.04          & 8.14           & 0.01          & 7.54           & 0.02          & 5.79           & 0.01          & 5.38           & 0.01          \\
    MatGPTQ         &                    & 15.77          & 0.00          & 8.34           & 0.00          & \multicolumn{2}{c}{--}       & 5.91           & 0.00          & 5.42           & 0.00          \\
    \midrule
    DbyD (Had.+35W) & \multirow{4}{*}{4} & \textbf{15.49} & \textbf{0.16} & \textbf{8.05} & \textbf{0.03} & \textbf{7.28} & \textbf{0.01} & \textbf{5.78} & \textbf{0.02} & \textbf{5.21} & \textbf{0.01} \\
    AnyBCQ          &                    & 18.41          & 0.00          & 9.84           & 0.00          & \multicolumn{2}{c}{\textit{nan}} & 7.45           & 0.00          & 7.04 & 0.00        \\
    AnyPrecLLM      &                    & 16.08          & 0.15          & 8.51           & 0.03          & 7.75           & 0.08          & 6.09           & 0.01          & 5.60           & 0.01          \\
    MatGPTQ         &                    & 18.11          & 0.00          & 8.88           & 0.00          & \multicolumn{2}{c}{--}       & 6.27           & 0.00          & 5.64           & 0.00          \\
    \midrule
    DbyD (Had.+35W) & \multirow{4}{*}{3} & \textbf{22.05} & \textbf{0.67} & \textbf{9.73} & \textbf{0.11} & \textbf{8.40} & \textbf{0.05} & \textbf{6.76} & \textbf{0.04} & \textbf{5.80} & \textbf{0.03} \\
    AnyBCQ          &                    & 23.95          & 0.00          & 11.19          & 0.00          & \multicolumn{2}{c}{\textit{nan}} & 8.18           & 0.00          & 7.53 & 0.00           \\
    AnyPrecLLM      &                    & 24.87          & 0.68          & 10.38          & 0.04          & 9.02           & 0.01          & 7.16           & 0.05          & 6.34           & 0.06          \\
    MatGPTQ         &                    & 69.56          & 0.00          & 129.00         & 0.00          & \multicolumn{2}{c}{--}       & 7.85           & 0.00          & 6.70           & 0.00          \\
    \bottomrule
    \end{tabular}

    \vspace{0.4cm} 

    \begin{tabular}{lc r@{\;\(\pm\)\;}l r@{\;\(\pm\)\;}l r@{\;\(\pm\)\;}l}
    \toprule
    \textbf{Method} & \textbf{Bits} & 
    \multicolumn{2}{c}{\textbf{Gemma 2B}} & 
    \multicolumn{2}{c}{\textbf{Mistral 7B}} & 
    \multicolumn{2}{c}{\textbf{MetaLlama 8B}} \\
    \midrule
    DbyD (Had.+35W) & \multirow{4}{*}{5} & \textbf{8.49}& \textbf{0.00}          & \textbf{4.99} & \textbf{0.00} & \textbf{5.96} & \textbf{0.00} \\
    AnyBCQ          &                    & 40.80          & 0.00          & 25.84          & 0.00          & 7.22           & 0.00          \\
    AnyPrecLLM      &                    & 9.10  & 0.01 & 5.35           & 0.00          & 6.25           & 0.00          \\
    MatGPTQ         &                    & 15.77          & 0.00          & 5.94           & 0.00          & 6.41           & 0.00          \\
    \midrule
    DbyD (Had.+35W) & \multirow{4}{*}{4} & \textbf{9.69}         & \textbf{0.31}          & \textbf{5.13} & \textbf{0.00} & \textbf{6.47} & \textbf{0.04} \\
    AnyBCQ          &                    & 43.59          & 0.00          & 16.95          & 0.00          & 7.52           & 0.00          \\
    AnyPrecLLM      &                    & 9.72  & 0.01 & 5.45           & 0.01          & 6.60           & 0.01          \\
    MatGPTQ         &                    & 16.27          & 0.00          & 6.86           & 0.00          & 6.83           & 0.00          \\
    \midrule
    DbyD (Had.+35W) & \multirow{4}{*}{3} & 13.18          & 1.15          & \textbf{5.62} & \textbf{0.02} & \textbf{8.07} & \textbf{0.05} \\
    AnyBCQ          &                    & 48.95          & 0.00          & 10.47          & 0.00          & 8.87           & 0.00          \\
    AnyPrecLLM      &                    & \textbf{12.52} & \textbf{0.10} & 5.98          & 0.03         & 8.32           & 0.02          \\
    MatGPTQ         &                    & 42.44          & 0.00          & 8.44           & 0.00          & 9.70           & 0.00          \\
    \bottomrule
    \end{tabular}
\label{tab:ppl_wiki_combined_sota_3to5}
\end{table}

The similar pattern of DbyD outperforming also holds when each baseline is instead quantized over its own native
bitwidth range, namely 2-4 bits for AnyBCQ and 3-8 bits for AnyPrecLLM,
which we quantize, evaluate and report in Table~\ref{tab:ppl_c4_combined_sota} and Table~\ref{tab:ppl_wiki_combined_sota}. The entries marked with an asterisk ($^{*}$) for the 14B and 32B models use the officially released \textsc{Qwen3-14B} and \textsc{Qwen3-32B} checkpoints, as we wanted to utilize the only \textsc{Qwen} variants released by the authors. 

Moreover, quantizing \textsc{Qwen2.5-32B} for AnyPrecLLM and MatGPTQ, and both
\textsc{Qwen2.5-14B} and \textsc{Qwen2.5-32B} for AnyBCQ, required H100 GPUs rather than
the L40S GPUs used for the remaining models, since these methods required a larger memory capacity for the larger models. Whereas, Drop-by-Drop did not and was quantized on L40S GPUs throughout.

\begin{table}[h!]
\small
\caption{SOTA perplexity comparison ($\downarrow$) on \textsc{C4} in native bitwidth ranges.}
\centering

    \begin{tabular}{lc r@{\;\(\pm\)\;}l r@{\;\(\pm\)\;}l r@{\;\(\pm\)\;}l r@{\;\(\pm\)\;}l r@{\;\(\pm\)\;}l}
    \toprule
    \textbf{Method} & \textbf{Bits} & 
    \multicolumn{2}{c}{\textbf{Qwen 0.5B}} & 
    \multicolumn{2}{c}{\textbf{Qwen 3B}} & 
    \multicolumn{2}{c}{\textbf{Qwen 7B}} & 
    \multicolumn{2}{c}{\textbf{Qwen 14B}} & 
    \multicolumn{2}{c}{\textbf{Qwen 32B}} \\
    \midrule
    DbyD (Had.+35W) & \multirow{4}{*}{5} & \textbf{17.87} & \textbf{0.00} & \textbf{11.32} & \textbf{0.01} & \textbf{10.51} & \textbf{0.01} & \textbf{9.06} & \textbf{0.00} & \textbf{8.79} & \textbf{0.00} \\
    AnyBCQ          &                    & \multicolumn{2}{c}{--}       & \multicolumn{2}{c}{--}       & \multicolumn{2}{c}{--}       & \multicolumn{2}{c}{--}       & \multicolumn{2}{c}{--}       \\
    AnyPrecLLM      &                    & 22.45          & 0.01          & 13.78          & 0.01          & 12.93          & 0.02          & 11.01          & 0.01          & 10.70          & 0.00          \\
    MatGPTQ         &                    & 23.25          & 0.00          & 13.62          & 0.00          & \multicolumn{2}{c}{--}       & 11.07          & 0.00          & 10.59          & 0.00          \\
    \midrule
    DbyD (Had.+35W) & \multirow{4}{*}{4} & \textbf{19.99} & \textbf{0.06} & \textbf{11.93} & \textbf{0.01} & \textbf{10.93} & \textbf{0.03} & \textbf{9.40} & \textbf{0.00} & \textbf{8.98} & \textbf{0.00} \\
    AnyBCQ          &                    & 26.44          & 0.00          & 15.54          & 0.00          & \multicolumn{2}{c}{\textit{nan}}       & 14.66$^{*}$    & 0.00          & 13.06$^{*}$    & 0.00          \\
    AnyPrecLLM      &                    & 24.23          & 0.02          & 14.26          & 0.02          & 13.34          & 0.04          & 11.30          & 0.00          & 10.88          & 0.01          \\
    MatGPTQ         &                    & 26.72          & 0.00          & 14.22          & 0.00          & \multicolumn{2}{c}{--}       & 11.47          & 0.00          & 10.81          & 0.00          \\
    \midrule
    DbyD (Had.+35W) & \multirow{4}{*}{3} & \textbf{27.67} & \textbf{1.22} & \textbf{13.96} & \textbf{0.00} & \textbf{12.27} & \textbf{0.12} & \textbf{10.38} & \textbf{0.01} & \textbf{9.49} & \textbf{0.00} \\
    AnyBCQ          &                    & 32.27          & 0.00          & 17.17          & 0.00          & \multicolumn{2}{c}{\textit{nan}}       & 15.51$^{*}$    & 0.00          & 13.61$^{*}$    & 0.00          \\
    AnyPrecLLM      &                    & 35.33          & 0.26          & 16.70          & 0.06          & 15.44          & 0.21          & 12.47          & 0.01          & 11.63          & 0.02          \\
    MatGPTQ         &                    & 80.06          & 0.00          & 134.12         & 0.00          & \multicolumn{2}{c}{--}       & 13.23          & 0.00          & 11.90          & 0.00          \\
    \bottomrule
    \end{tabular}

    \vspace{0.4cm} 

        \begin{tabular}{lc r@{\;\(\pm\)\;}l r@{\;\(\pm\)\;}l r@{\;\(\pm\)\;}l}
    \toprule
    \textbf{Method} & \textbf{Bits} & 
    \multicolumn{2}{c}{\textbf{Gemma 2B}} & 
    \multicolumn{2}{c}{\textbf{Mistral 7B}} & 
    \multicolumn{2}{c}{\textbf{MetaLlama 8B}} \\
    \midrule
    DbyD (Had.+35W) & \multirow{4}{*}{5} & \textbf{11.79} & \textbf{0.00} & \textbf{7.54} & \textbf{0.00} & \textbf{8.63} & \textbf{0.00} \\
    AnyBCQ          &                    & \multicolumn{2}{c}{--}       & \multicolumn{2}{c}{--}       & \multicolumn{2}{c}{--}       \\
    AnyPrecLLM      &                    & 13.65          & 0.01          & 8.82           & 0.50          & 9.74           & 0.04          \\
    MatGPTQ         &                    & 22.62          & 0.00          & 9.18           & 0.00          & 9.71           & 0.00          \\
    \midrule
    DbyD (Had.+35W) & \multirow{4}{*}{4} & \textbf{13.23} & \textbf{0.37} & \textbf{7.71} & \textbf{0.00} & \textbf{9.38} & \textbf{0.04} \\
    AnyBCQ          &                    & 26.78          & 0.00          & 11.46          & 0.00          & 10.55          & 0.00          \\
    AnyPrecLLM      &                    & 14.36          & 0.01          & 8.67           & 0.00          & 10.34          & 0.05          \\
    MatGPTQ         &                    & 23.25          & 0.00          & 9.95           & 0.00          & 10.60          & 0.00          \\
    \midrule
    DbyD (Had.+35W) & \multirow{4}{*}{3} & \textbf{17.07} & \textbf{1.15} & \textbf{8.32} & \textbf{0.00} & \textbf{11.53} & \textbf{0.05} \\
    AnyBCQ          &                    & 29.41          & 0.00          & 12.11          & 0.00          & 12.41          & 0.00          \\
    AnyPrecLLM      &                    & 17.55          & 0.24          & 9.09           & 0.49          & 13.05          & 0.10          \\
    MatGPTQ         &                    & 52.50          & 0.00          & 11.95          & 0.00          & 14.93          & 0.00          \\
    \bottomrule
    \end{tabular}

\label{tab:ppl_c4_combined_sota}
\end{table}

\begin{table}[h!]
\small
\caption{SOTA perplexity comparison ($\downarrow$) on \textsc{WikiText2} in native bitwidth ranges.}

\centering

    \begin{tabular}{lc r@{\;\(\pm\)\;}l r@{\;\(\pm\)\;}l r@{\;\(\pm\)\;}l r@{\;\(\pm\)\;}l r@{\;\(\pm\)\;}l}
    \toprule
    \textbf{Method} & \textbf{Bits} & 
    \multicolumn{2}{c}{\textbf{Qwen 0.5B}} & 
    \multicolumn{2}{c}{\textbf{Qwen 3B}} & 
    \multicolumn{2}{c}{\textbf{Qwen 7B}} & 
    \multicolumn{2}{c}{\textbf{Qwen 14B}} & 
    \multicolumn{2}{c}{\textbf{Qwen 32B}} \\
    \midrule
    DbyD (Had.+35W) & \multirow{4}{*}{5} & \textbf{13.36} & \textbf{0.04} & \textbf{7.57} & \textbf{0.00} & \textbf{6.93} & \textbf{0.01} & \textbf{5.38} & \textbf{0.00} & \textbf{4.99} & \textbf{0.01} \\
    AnyBCQ          &                    & \multicolumn{2}{c}{--}       & \multicolumn{2}{c}{--}       & \multicolumn{2}{c}{--}       & \multicolumn{2}{c}{--}       & \multicolumn{2}{c}{--}       \\
    AnyPrecLLM      &                    & 14.68          & 0.03          & 8.14           & 0.01          & 7.55           & 0.02          & 5.79           & 0.00          & 5.39           & 0.01          \\
    MatGPTQ         &                    & 15.77          & 0.00          & 8.34           & 0.00          & \multicolumn{2}{c}{--}       & 5.91           & 0.00          & 5.42           & 0.00          \\
    \midrule
    DbyD (Had.+35W) & \multirow{4}{*}{4} & \textbf{15.49} & \textbf{0.16} & \textbf{8.05} & \textbf{0.03} & \textbf{7.28} & \textbf{0.01} & \textbf{5.78} & \textbf{0.02} & \textbf{5.21} & \textbf{0.01} \\
    AnyBCQ          &                    & 19.11          & 0.00          & 10.04          & 0.00          & \multicolumn{2}{c}{\textit{nan}}       & 9.02$^{*}$     & 0.00          & 7.81$^{*}$     & 0.00          \\
    AnyPrecLLM      &                    & 16.16          & 0.14          & 8.49           & 0.02          & 7.77           & 0.04          & 6.09           & 0.03          & 5.61           & 0.02          \\
    MatGPTQ         &                    & 18.11          & 0.00          & 8.88           & 0.00          & \multicolumn{2}{c}{--}       & 6.27           & 0.00          & 5.64           & 0.00          \\
    \midrule
    DbyD (Had.+35W) & \multirow{4}{*}{3} & \textbf{22.05} & \textbf{0.67} & \textbf{9.73} & \textbf{0.11} & \textbf{8.40} & \textbf{0.05} & \textbf{6.76} & \textbf{0.04} & \textbf{5.80} & \textbf{0.03} \\
    AnyBCQ          &                    & 26.20          & 0.00          & 11.45          & 0.00          & \multicolumn{2}{c}{\textit{nan}}       & 9.95$^{*}$     & 0.00          & 8.42$^{*}$     & 0.00          \\
    AnyPrecLLM      &                    & 25.52          & 0.68          & 10.33          & 0.10          & 9.30           & 0.23          & 7.12           & 0.01          & 6.35           & 0.01          \\
    MatGPTQ         &                    & 69.56          & 0.00          & 129.00         & 0.00          & \multicolumn{2}{c}{--}       & 7.85           & 0.00          & 6.70           & 0.00          \\
    \bottomrule
    \end{tabular}

    \vspace{0.4cm} 

        \begin{tabular}{lc r@{\;\(\pm\)\;}l r@{\;\(\pm\)\;}l r@{\;\(\pm\)\;}l}
    \toprule
    \textbf{Method} & \textbf{Bits} & 
    \multicolumn{2}{c}{\textbf{Gemma 2B}} & 
    \multicolumn{2}{c}{\textbf{Mistral 7B}} & 
    \multicolumn{2}{c}{\textbf{MetaLlama 8B}} \\
    \midrule
    DbyD (Had.+35W) & \multirow{4}{*}{5} & \textbf{8.49}  & \textbf{0.00} & \textbf{4.99} & \textbf{0.00} & \textbf{5.96} & \textbf{0.00} \\
    AnyBCQ          &                    & \multicolumn{2}{c}{--}       & \multicolumn{2}{c}{--}       & \multicolumn{2}{c}{--}       \\
    AnyPrecLLM      &                    & 9.09           & 0.01          & 5.35           & 0.00          & 6.31           & 0.05          \\
    MatGPTQ         &                    & 15.77          & 0.00          & 5.94           & 0.00          & 6.41           & 0.00          \\
    \midrule
    DbyD (Had.+35W) & \multirow{4}{*}{4} & 9.69           & 0.31          & \textbf{5.13} & \textbf{0.00} & \textbf{6.47} & \textbf{0.04} \\
    AnyBCQ          &                    & 45.17          & 0.00          & 12.93          & 0.00          & 6.75           & 0.00          \\
    AnyPrecLLM      &                    & \textbf{9.68}  & \textbf{0.03} & 5.45           & 0.00          & 6.64           & 0.04          \\
    MatGPTQ         &                    & 16.27          & 0.00          & 6.86           & 0.00          & 6.83           & 0.00          \\
    \midrule
    DbyD (Had.+35W) & \multirow{4}{*}{3} & 13.18          & 1.15          & \textbf{5.62} & \textbf{0.02} & \textbf{8.07} & \textbf{0.05} \\
    AnyBCQ          &                    & 54.33          & 0.00          & 13.86          & 0.00          & 8.20           & 0.00          \\
    AnyPrecLLM      &                    & \textbf{12.29} & \textbf{0.30} & 5.99           & 0.01          & 8.25           & 0.07          \\
    MatGPTQ         &                    & 42.44          & 0.00          & 8.44           & 0.00          & 9.70           & 0.00          \\
    \bottomrule
    \end{tabular}

\label{tab:ppl_wiki_combined_sota}
\end{table}

\clearpage

\subsection{Zero-shot Accuracy Values for SOTA Methods}
\label{app:accdetailedsota}
Table~\ref{tab:combined_avg_acc_sota} reports average zero-shot accuracy over five tasks. DbyD achieves accuracy comparable to the state-of-the-art baselines, staying within 1.5 points of the best method in most settings, which is within one standard error of the difference between the two. At the same time, DbyD provides benefits over these baselines in perplexity, inference latency, and storage footprint, as detailed in Appendices~\ref{app:ppldetailedsota} and~\ref{app:latencydetailedsota}, and Figure~\ref{fig:disk_usage_comparison}.

\begin{table}[h!]
\small
\caption{Average zero-shot accuracy across 5 Tasks (\(\uparrow\)), compared against SOTA methods.}
\begin{center}
\resizebox{\textwidth}{!}{%
\begin{tabular}{cl r@{\,\(\pm\)\,}l r@{\,\(\pm\)\,}l r@{\,\(\pm\)\,}l | r@{\,\(\pm\)\,}l r@{\,\(\pm\)\,}l r@{\,\(\pm\)\,}l}
\toprule
\textbf{Model} & \textbf{Bits} &
\multicolumn{2}{c}{\makecell{\textbf{AQLM} \\ \textbf{(Ind.)}}} &
\multicolumn{2}{c}{\makecell{\textbf{DbyD} \\ \textbf{(35W)}}} &
\multicolumn{2}{c|}{\makecell{\textbf{DbyD} \\ \textbf{(Had.+35W)}}} &
\multicolumn{2}{c}{\textbf{AnyBCQ}} &
\multicolumn{2}{c}{\makecell{\textbf{AnyPrec}}} &
\multicolumn{2}{c}{\textbf{MatGPTQ}} \\
\midrule
\multirow{3}{*}{\makecell{Gemma \\ 2B}}
  & 5 & 65.59 & 1.02 & 65.48 & 1.03 & 65.13 & 1.03 & 64.50 & 1.03 & 65.22 & 1.03 & 64.37 & 1.03 \\
  & 4 & 63.95 & 1.03 & 63.90 & 1.03 & 63.48 & 1.03 & 63.57 & 1.03 & 64.88 & 1.02 & 63.47 & 1.04 \\
  & 3 & 60.37 & 1.05 & 60.61 & 1.05 & 60.17 & 1.05 & 60.94 & 1.04 & 60.30 & 1.04 & 56.22 & 1.07 \\
\midrule
\multirow{3}{*}{\makecell{Mistral \\ 7B}}
  & 5 & 71.72 & 1.05 & 72.85 & 1.04 & 73.08 & 0.97 & 58.11 & 1.07 & 74.24 & 0.96 & 73.00 & 0.97 \\
  & 4 & 72.46 & 1.04 & 70.81 & 1.06 & 72.32 & 1.00 & 63.65 & 1.05 & 73.59 & 0.96 & 72.88 & 0.97 \\
  & 3 & 70.97 & 1.05 & 67.48 & 1.08 & 70.20 & 1.02 & 70.43 & 0.99 & 72.14 & 0.98 & 70.25 & 1.00 \\
\midrule
\multirow{3}{*}{\makecell{MetaLlama \\ 8B}}
  & 5 & 72.42 & 1.04 & 72.09 & 1.05 & 73.37 & 0.97 & 72.90 & 0.97 & 72.35 & 0.98 & 72.12 & 0.98 \\
  & 4 & 71.00 & 1.05 & 69.60 & 1.06 & 72.36 & 1.00 & 72.33 & 0.98 & 71.41 & 0.98 & 71.72 & 0.98 \\
  & 3 & 70.00 & 1.06 & 64.81 & 1.09 & 67.81 & 1.03 & 69.25 & 1.00 & 68.25 & 1.01 & 64.15 & 1.03 \\
\midrule
\multirow{3}{*}{\makecell{Qwen \\ 0.5B}}
  & 5 & 53.25 & 1.12 & 53.86 & 1.12 & 54.17 & 1.12 & 54.76 & 1.07 & 54.60 & 1.07 & 52.89 & 1.07 \\
  & 4 & 53.60 & 1.11 & 53.14 & 1.12 & 52.40 & 1.12 & 53.79 & 1.07 & 54.67 & 1.07 & 51.98 & 1.07 \\
  & 3 & 49.96 & 1.12 & 48.47 & 1.11 & 48.36 & 1.12 & 50.82 & 1.07 & 48.28 & 1.06 & 44.69 & 1.07 \\
\midrule
\multirow{3}{*}{\makecell{Qwen \\ 3B}}
  & 5 & 67.83 & 1.07 & 68.17 & 1.07 & 68.11 & 1.07 & 68.21 & 1.01 & 68.18 & 1.01 & 68.02 & 1.01 \\
  & 4 & 66.68 & 1.08 & 67.28 & 1.08 & 67.80 & 1.08 & 68.12 & 1.01 & 68.58 & 1.01 & 67.86 & 1.01 \\
  & 3 & 65.48 & 1.09 & 63.06 & 1.10 & 64.32 & 1.10 & 65.19 & 1.03 & 65.32 & 1.03 & 44.47 & 1.07 \\
\midrule
\multirow{3}{*}{\makecell{Qwen \\ 7B}}
  & 5 & 73.59 & 1.04 & 73.72 & 1.04 & 73.28 & 1.04 & 34.38\textsuperscript{\dag} & 1.02 & 73.43 & 0.97 & \multicolumn{2}{c}{--} \\
  & 4 & 72.37 & 1.05 & 72.52 & 1.05 & 71.37 & 1.06 & 34.38\textsuperscript{\dag} & 1.02 & 72.14 & 0.99 & \multicolumn{2}{c}{--} \\
  & 3 & 70.83 & 1.07 & 69.54 & 1.08 & 68.66 & 1.08 & 34.38\textsuperscript{\dag} & 1.02 & 70.13 & 1.00 & \multicolumn{2}{c}{--} \\
\midrule
\multirow{3}{*}{\makecell{Qwen \\ 14B}}
  & 5 & 76.85 & 1.01 & 76.76 & 1.01 & 77.20 & 0.93 & 76.76 & 0.94 & 76.92 & 0.94 & 76.72 & 0.94 \\
  & 4 & 75.94 & 1.02 & 75.99 & 1.01 & 76.63 & 0.96 & 76.04 & 0.95 & 76.04 & 0.95 & 76.19 & 0.95 \\
  & 3 & 74.89 & 1.03 & 73.93 & 1.04 & 75.07 & 0.98 & 75.78 & 0.95 & 73.16 & 0.98 & 72.33 & 0.99 \\
\midrule
\multirow{3}{*}{\makecell{Qwen \\ 32B}}
  & 5 & 74.53 & 1.04 & 74.74 & 1.03 & 74.98 & 1.03 & 74.68\textsuperscript{\ddag} & 0.97 & 74.83 & 0.97 & 75.02 & 0.96 \\
  & 4 & 74.50 & 1.04 & 73.98 & 1.04 & 75.22 & 1.03 & 74.57\textsuperscript{\ddag} & 0.97 & 74.44 & 0.97 & 74.49 & 0.97 \\
  & 3 & 74.94 & 1.03 & 72.94 & 1.05 & 73.88 & 1.04 & 74.92\textsuperscript{\ddag} & 0.96 & 74.13 & 0.97 & 72.87 & 0.99 \\
\bottomrule
\end{tabular}
}
\vspace{2pt}
\begin{minipage}{\textwidth}\footnotesize
\textsuperscript{\dag}~AnyBCQ quantization of Qwen2.5-7B diverged (NaN loss); accuracy is at chance level.\\
\textsuperscript{\ddag}~AnyBCQ Qwen2.5-32B trained for 1{,}500 instead of 5{,}000 iterations.\\
--~MatGPTQ quantization of Qwen2.5-7B did not converge; no results reported.
\end{minipage}
\end{center}
\label{tab:combined_avg_acc_sota}
\end{table}
\clearpage
\subsection{Latency Values for the SOTA Methods}
\label{app:latencydetailedsota}
\paragraph{Motivation.}
AQLM has specialized CUDA kernels only for the 1$\times$16 and 2$\times$8 codebook configurations. Their 2$\times$8 kernel keeps 8 copies of each codebook in shared memory, which does not fit within the per-block limit once more codebooks are added. As a result, for 3--5 codebooks, AQLM falls back to a much slower generic path: a Triton GEMV for 4 codebooks, and a full dequantize-then-GEMM for 3 and 5, since Triton requires power-of-two ranges. We, therefore, wrote \texttt{codeNx8}, a kernel for $N$ 8-bit codebooks over groups of 8 weights, i.e.\ $N$ bits per weight. It uses AQLM's quantized weights (codes, codebooks and per-row scales) as they are; the quantized model is unchanged.

\paragraph{Kernel design.}
A codebook-quantized weight group is the sum of several 8-value codebook vectors. The usual approach rebuilds those weights before taking the dot product with the input. Because the dot product is linear, at batch size 1 \texttt{codeNx8} instead computes the dot product of the input with every codebook entry once, then replaces each weight with a scalar table lookup. The main advances are:
\begin{itemize}
    \item \textbf{Table lookups replace weight reconstruction.} Each code costs one small shared-memory read instead of fetching a codebook vector and performing 8 multiply-adds. The table depends on the input, so it is rebuilt on every call, and all reported latencies include it.
    \item \textbf{Contiguous memory reads.} Codes are reordered once at load time so that each group of 32 threads reads one contiguous block of memory.
    \item \textbf{Work split across the GPU.} Long rows are divided among many thread blocks, so layers with few output rows still spread across many of the GPU's multiprocessors. Partial results are combined within the same kernel launch in a fixed order, so outputs are deterministic.
\end{itemize}

\paragraph{Experimental setup.}
All kernels were benchmarked on a single NVIDIA L40S GPU (48\,GB memory, 96\,MB L2 cache). \texttt{codeNx8} was built with CUDA 12.6 and PyTorch 2.7.1. We evaluate batch-1 fp16 matrix-vector products, the core operation of autoregressive decoding. 

We compare against dense FP16 cuBLAS and three SOTA quantized GEMV kernels: AnyBCQ, AnyPrecisionLLM and MatGPTQ. Each baseline is called through its authors' unmodified kernel code, in the software environment the authors provide, at the same bit width as ours. MatGPTQ has no 5-bit kernel, so it covers only the 3- and 4-bit settings. All methods were run in the same job on the same GPU, with identical shapes and timing procedure.

We measured latency in two settings: with weights allowed to stay in the L2 cache between calls, and with the cache cleared before every call. \texttt{codeNx8} was faster than every baseline on average in both settings. We report the second setting because it matches decoding: the quantized weights of every model we evaluate exceed the L2 cache, so each layer's weights are read from DRAM once per generated token. Each latency is the mean over 3 repetitions of 100 timed calls.

\paragraph{Results.}
Table~\ref{tab:cold_speedup} reports the per-layer speedup of \texttt{codeNx8} over each SOTA baseline for all configurations. Moreover, compared with AQLM's existing path for 3-5 codebooks, \texttt{codeNx8} is $2.25$-$85.6\times$ faster.

For the Hadamard variant, DbyD (Had.+35W), we use \texttt{codeNx8\_Had}, which applies the Hadamard transform to the input vector before building the lookup table; the reported latencies include this transform. Table~\ref{tab:hadamard_speedup} shows that the extra step reduces the speedups slightly, but \texttt{codeNx8\_Had} is still faster than all baselines in all configurations, displayed in the main text with Figure~\ref{fig:hadamard_speedup}.

\begin{figure}[h]
    \centering
    \begin{tikzpicture}
    \begin{axis}[
        width=0.5\textwidth,
        height=4.0cm,
        ymode=log,
        ymin=0.97, ymax=3.1,
        symbolic x coords={Q-0.5, G-2, Q-3, M-7, Q-7, L-8, Q-14, Q-32},
        xtick=data,
        enlarge x limits=0.10,
        xticklabel style={
            font=\scriptsize,
            rotate=30,
            anchor=north east,
            inner sep=1pt
        },
        ytick={1,1.5,2,3}, yticklabels={\(1\times\),\(1.5\times\),\(2\times\),\(3\times\)},
        minor tick num=0,
        yticklabel style={font=\small},
        ymajorgrids=true,
        grid style={dashed, gray!30},
        axis line style={draw=none},
        tick style={draw=none},
        legend style={
            at={(0.5,1.03)},
            anchor=south,
            draw=none,
            fill=none,
            font=\scriptsize,
            cells={anchor=west},
            column sep=5pt
        },
        legend columns=2,
        every axis plot/.append style={thick, mark size=1.6pt},
    ]
    \addplot[gray!80!black, mark=*] coordinates {
        (Q-0.5, 1.59) (G-2, 2.37) (Q-3, 2.16) (M-7, 2.75)
        (Q-7, 2.66)   (L-8, 2.73) (Q-14, 2.82) (Q-32, 2.82)
    };
    \addlegendentry{cuBLAS (FP16)}
    \addplot[cAnyPrecfig10, mark=triangle*] coordinates {
        (Q-0.5, 1.17) (G-2, 1.18) (Q-3, 1.18) (M-7, 1.15)
        (Q-7, 1.21)   (L-8, 1.15) (Q-14, 1.19) (Q-32, 1.17)
    };
    \addlegendentry{AnyPrecLLM}
    \addplot[cAnyBCQfig10, mark=square*] coordinates {
        (Q-0.5, 1.16) (G-2, 1.15) (Q-3, 1.17) (M-7, 1.19)
        (Q-7, 1.18)   (L-8, 1.17) (Q-14, 1.17) (Q-32, 1.16)
    };
    \addlegendentry{AnyBCQ}
    \addplot[cMatGPTQfig10, mark=diamond*] coordinates {
        (Q-0.5, 1.26) (G-2, 1.19) (Q-3, 1.17) (M-7, 1.08)
        (Q-7, 1.08)   (L-8, 1.09) (Q-14, 1.07) (Q-32, 1.06)
    };
    \addlegendentry{MatGPTQ}
    \addplot[black!50, densely dotted, forget plot] coordinates {(Q-0.5, 1) (Q-32, 1)};
    \end{axis}
    \end{tikzpicture}
    \caption{Latency speedup of \texttt{codeNx8} (\(\uparrow\)). Values from Table~\ref{tab:cold_speedup}}
    \label{fig:speedup}
\end{figure}
\begin{table*}[h]
\centering
\footnotesize
\setlength{\tabcolsep}{2.5pt}
\newcommand{\mname}[2]{\begin{tabular}[c]{@{}l@{}}#1\\#2\end{tabular}}
\caption{Latency speedup of \texttt{codeNx8}. Values above $1.00\times$ mean \texttt{codeNx8} is faster.}
\label{tab:cold_speedup}
\begin{tabular}{lrr ccc ccc ccc cc}
\toprule
 & & & \multicolumn{3}{c}{vs.\ cuBLAS} & \multicolumn{3}{c}{vs.\ AnyBCQ} & \multicolumn{3}{c}{vs.\ AnyPrecisionLLM} & \multicolumn{2}{c}{vs.\ MatGPTQ} \\
\cmidrule(lr){4-6} \cmidrule(lr){7-9} \cmidrule(lr){10-12} \cmidrule(lr){13-14}
Model & $N$ & $K$ & 3-bit & 4-bit & 5-bit & 3-bit & 4-bit & 5-bit & 3-bit & 4-bit & 5-bit & 3-bit & 4-bit \\
\midrule
\multirow{3}{*}{\mname{Qwen2.5}{0.5B}}
 & 896   & 4864  & 1.79$\times$ & 1.74$\times$ & 1.59$\times$ & 1.13$\times$ & 1.19$\times$ & 1.14$\times$ & 1.21$\times$ & 1.22$\times$ & 1.26$\times$ & 1.25$\times$ & 1.23$\times$ \\
 & 4864  & 896   & 1.78$\times$ & 1.72$\times$ & 1.60$\times$ & 1.15$\times$ & 1.18$\times$ & 1.17$\times$ & 1.17$\times$ & 1.09$\times$ & 1.13$\times$ & 1.26$\times$ & 1.23$\times$ \\
 & 896   & 896   & 1.38$\times$ & 1.37$\times$ & 1.43$\times$ & 1.14$\times$ & 1.13$\times$ & 1.18$\times$ & 1.14$\times$ & 1.12$\times$ & 1.17$\times$ & 1.29$\times$ & 1.29$\times$ \\
\midrule
\multirow{3}{*}{\mname{Qwen2.5}{3B}}
 & 2048  & 11008 & 2.90$\times$ & 2.53$\times$ & 2.25$\times$ & 1.21$\times$ & 1.11$\times$ & 1.12$\times$ & 1.35$\times$ & 1.26$\times$ & 1.44$\times$ & 1.28$\times$ & 1.21$\times$ \\
 & 11008 & 2048  & 2.79$\times$ & 2.47$\times$ & 2.21$\times$ & 1.22$\times$ & 1.19$\times$ & 1.16$\times$ & 1.11$\times$ & 1.04$\times$ & 1.15$\times$ & 1.06$\times$ & 1.05$\times$ \\
 & 2048  & 2048  & 1.67$\times$ & 1.60$\times$ & 1.54$\times$ & 1.16$\times$ & 1.16$\times$ & 1.17$\times$ & 1.12$\times$ & 1.07$\times$ & 1.14$\times$ & 1.25$\times$ & 1.20$\times$ \\
\midrule
\multirow{3}{*}{\mname{Qwen2.5}{7B}}
 & 3584  & 18944 & 3.62$\times$ & 2.95$\times$ & 2.47$\times$ & 1.16$\times$ & 1.16$\times$ & 1.16$\times$ & 1.30$\times$ & 1.20$\times$ & 1.29$\times$ & 1.10$\times$ & 1.09$\times$ \\
 & 18944 & 3584  & 3.46$\times$ & 2.84$\times$ & 2.41$\times$ & 1.24$\times$ & 1.21$\times$ & 1.21$\times$ & 1.17$\times$ & 1.11$\times$ & 1.18$\times$ & 1.03$\times$ & 1.03$\times$ \\
 & 3584  & 3584  & 2.46$\times$ & 2.19$\times$ & 1.99$\times$ & 1.14$\times$ & 1.18$\times$ & 1.14$\times$ & 1.27$\times$ & 1.12$\times$ & 1.26$\times$ & 1.13$\times$ & 1.08$\times$ \\
\midrule
\multirow{3}{*}{\mname{Qwen2.5}{14B}}
 & 5120  & 13824 & 3.61$\times$ & 2.93$\times$ & 2.41$\times$ & 1.24$\times$ & 1.16$\times$ & 1.13$\times$ & 1.24$\times$ & 1.17$\times$ & 1.22$\times$ & 1.11$\times$ & 1.07$\times$ \\
 & 13824 & 5120  & 3.48$\times$ & 2.86$\times$ & 2.41$\times$ & 1.20$\times$ & 1.21$\times$ & 1.20$\times$ & 1.19$\times$ & 1.14$\times$ & 1.17$\times$ & 1.05$\times$ & 1.05$\times$ \\
 & 5120  & 5120  & 3.01$\times$ & 2.62$\times$ & 2.35$\times$ & 1.18$\times$ & 1.15$\times$ & 1.10$\times$ & 1.23$\times$ & 1.12$\times$ & 1.27$\times$ & 1.09$\times$ & 1.06$\times$ \\
\midrule
\multirow{3}{*}{\mname{Qwen2.5}{32B}}
 & 5120  & 27648 & 3.69$\times$ & 2.91$\times$ & 2.39$\times$ & 1.17$\times$ & 1.17$\times$ & 1.13$\times$ & 1.15$\times$ & 1.12$\times$ & 1.22$\times$ & 1.08$\times$ & 1.06$\times$ \\
 & 27648 & 5120  & 3.66$\times$ & 2.88$\times$ & 2.38$\times$ & 1.25$\times$ & 1.20$\times$ & 1.17$\times$ & 1.14$\times$ & 1.10$\times$ & 1.15$\times$ & 1.05$\times$ & 1.05$\times$ \\
 & 5120  & 5120  & 2.95$\times$ & 2.58$\times$ & 2.32$\times$ & 1.11$\times$ & 1.11$\times$ & 1.11$\times$ & 1.24$\times$ & 1.12$\times$ & 1.28$\times$ & 1.08$\times$ & 1.07$\times$ \\
\midrule
\multirow{3}{*}{\mname{Gemma}{2B}}
 & 2048  & 16384 & 3.33$\times$ & 2.84$\times$ & 2.46$\times$ & 1.13$\times$ & 1.12$\times$ & 1.10$\times$ & 1.32$\times$ & 1.20$\times$ & 1.37$\times$ & 1.25$\times$ & 1.16$\times$ \\
 & 16384 & 2048  & 3.30$\times$ & 2.85$\times$ & 2.49$\times$ & 1.24$\times$ & 1.13$\times$ & 1.15$\times$ & 1.13$\times$ & 1.05$\times$ & 1.15$\times$ & 1.08$\times$ & 1.05$\times$ \\
 & 2048  & 2048  & 1.68$\times$ & 1.66$\times$ & 1.54$\times$ & 1.14$\times$ & 1.21$\times$ & 1.17$\times$ & 1.15$\times$ & 1.12$\times$ & 1.14$\times$ & 1.31$\times$ & 1.30$\times$ \\
\midrule
\multirow{3}{*}{\mname{Mistral}{7B-v0.3}}
 & 4096  & 14336 & 3.61$\times$ & 2.93$\times$ & 2.46$\times$ & 1.35$\times$ & 1.15$\times$ & 1.14$\times$ & 1.24$\times$ & 1.12$\times$ & 1.21$\times$ & 1.11$\times$ & 1.08$\times$ \\
 & 14336 & 4096  & 3.52$\times$ & 2.85$\times$ & 2.45$\times$ & 1.30$\times$ & 1.18$\times$ & 1.21$\times$ & 1.16$\times$ & 1.07$\times$ & 1.12$\times$ & 1.07$\times$ & 1.05$\times$ \\
 & 4096  & 4096  & 2.65$\times$ & 2.41$\times$ & 2.18$\times$ & 1.09$\times$ & 1.13$\times$ & 1.16$\times$ & 1.16$\times$ & 1.09$\times$ & 1.19$\times$ & 1.11$\times$ & 1.09$\times$ \\
\midrule
\multirow{3}{*}{\mname{Llama-3}{8B}}
 & 4096  & 14336 & 3.62$\times$ & 2.96$\times$ & 2.47$\times$ & 1.17$\times$ & 1.17$\times$ & 1.14$\times$ & 1.22$\times$ & 1.13$\times$ & 1.22$\times$ & 1.12$\times$ & 1.08$\times$ \\
 & 14336 & 4096  & 3.51$\times$ & 2.88$\times$ & 2.42$\times$ & 1.21$\times$ & 1.22$\times$ & 1.18$\times$ & 1.15$\times$ & 1.07$\times$ & 1.10$\times$ & 1.07$\times$ & 1.05$\times$ \\
 & 4096  & 4096  & 2.61$\times$ & 2.37$\times$ & 2.12$\times$ & 1.12$\times$ & 1.20$\times$ & 1.14$\times$ & 1.17$\times$ & 1.12$\times$ & 1.20$\times$ & 1.12$\times$ & 1.09$\times$ \\
\bottomrule
\end{tabular}
\end{table*}
\begin{table*}[h]
\centering
\footnotesize
\setlength{\tabcolsep}{2.5pt}
\newcommand{\mname}[2]{\begin{tabular}[c]{@{}l@{}}#1\\#2\end{tabular}}
\caption{Latency speedup of DbyD-Hadamard. Values above $1.00\times$ mean \texttt{codeNx8\_Had} is faster.}
\label{tab:hadamard_speedup}
\begin{tabular}{lrr ccc ccc ccc cc}
\toprule
 & & & \multicolumn{3}{c}{vs.\ cuBLAS} & \multicolumn{3}{c}{vs.\ AnyBCQ} & \multicolumn{3}{c}{vs.\ AnyPrecisionLLM} & \multicolumn{2}{c}{vs.\ MatGPTQ} \\
\cmidrule(lr){4-6} \cmidrule(lr){7-9} \cmidrule(lr){10-12} \cmidrule(lr){13-14}
Model & $N$ & $K$ & 3-bit & 4-bit & 5-bit & 3-bit & 4-bit & 5-bit & 3-bit & 4-bit & 5-bit & 3-bit & 4-bit \\
\midrule
\multirow{3}{*}{\mname{Qwen2.5}{0.5B}}
 & 896   & 4864  & 1.70$\times$ & 1.66$\times$ & 1.55$\times$ & 1.11$\times$ & 1.14$\times$ & 1.13$\times$ & 1.13$\times$ & 1.19$\times$ & 1.25$\times$ & 1.18$\times$ & 1.19$\times$ \\
 & 4864  & 896   & 1.63$\times$ & 1.63$\times$ & 1.55$\times$ & 1.07$\times$ & 1.15$\times$ & 1.15$\times$ & 1.08$\times$ & 1.06$\times$ & 1.07$\times$ & 1.15$\times$ & 1.16$\times$ \\
 & 896   & 896   & 1.28$\times$ & 1.27$\times$ & 1.28$\times$ & 1.08$\times$ & 1.08$\times$ & 1.10$\times$ & 1.09$\times$ & 1.06$\times$ & 1.08$\times$ & 1.18$\times$ & 1.17$\times$ \\
\midrule
\multirow{3}{*}{\mname{Qwen2.5}{3B}}
 & 2048  & 11008 & 2.87$\times$ & 2.48$\times$ & 2.22$\times$ & 1.21$\times$ & 1.10$\times$ & 1.09$\times$ & 1.35$\times$ & 1.22$\times$ & 1.46$\times$ & 1.29$\times$ & 1.17$\times$ \\
 & 11008 & 2048  & 2.71$\times$ & 2.44$\times$ & 2.18$\times$ & 1.12$\times$ & 1.14$\times$ & 1.14$\times$ & 1.09$\times$ & 1.05$\times$ & 1.12$\times$ & 1.02$\times$ & 1.04$\times$ \\
 & 2048  & 2048  & 1.59$\times$ & 1.58$\times$ & 1.50$\times$ & 1.09$\times$ & 1.14$\times$ & 1.16$\times$ & 1.09$\times$ & 1.07$\times$ & 1.14$\times$ & 1.19$\times$ & 1.16$\times$ \\
\midrule
\multirow{3}{*}{\mname{Qwen2.5}{7B}}
 & 3584  & 18944 & 3.58$\times$ & 2.93$\times$ & 2.46$\times$ & 1.15$\times$ & 1.15$\times$ & 1.15$\times$ & 1.29$\times$ & 1.19$\times$ & 1.29$\times$ & 1.11$\times$ & 1.06$\times$ \\
 & 18944 & 3584  & 3.42$\times$ & 2.81$\times$ & 2.38$\times$ & 1.27$\times$ & 1.19$\times$ & 1.19$\times$ & 1.16$\times$ & 1.11$\times$ & 1.17$\times$ & 1.02$\times$ & 1.02$\times$ \\
 & 3584  & 3584  & 2.37$\times$ & 2.14$\times$ & 1.95$\times$ & 1.12$\times$ & 1.13$\times$ & 1.09$\times$ & 1.25$\times$ & 1.11$\times$ & 1.23$\times$ & 1.10$\times$ & 1.07$\times$ \\
\midrule
\multirow{3}{*}{\mname{Qwen2.5}{14B}}
 & 5120  & 13824 & 3.53$\times$ & 2.88$\times$ & 2.42$\times$ & 1.23$\times$ & 1.14$\times$ & 1.14$\times$ & 1.22$\times$ & 1.15$\times$ & 1.23$\times$ & 1.07$\times$ & 1.05$\times$ \\
 & 13824 & 5120  & 3.47$\times$ & 2.83$\times$ & 2.39$\times$ & 1.25$\times$ & 1.20$\times$ & 1.20$\times$ & 1.19$\times$ & 1.12$\times$ & 1.16$\times$ & 1.05$\times$ & 1.04$\times$ \\
 & 5120  & 5120  & 2.86$\times$ & 2.53$\times$ & 2.27$\times$ & 1.10$\times$ & 1.14$\times$ & 1.11$\times$ & 1.20$\times$ & 1.12$\times$ & 1.27$\times$ & 1.05$\times$ & 1.04$\times$ \\
\midrule
\multirow{3}{*}{\mname{Qwen2.5}{32B}}
 & 5120  & 27648 & 3.65$\times$ & 2.89$\times$ & 2.37$\times$ & 1.16$\times$ & 1.16$\times$ & 1.16$\times$ & 1.14$\times$ & 1.11$\times$ & 1.21$\times$ & 1.07$\times$ & 1.06$\times$ \\
 & 27648 & 5120  & 3.63$\times$ & 2.87$\times$ & 2.36$\times$ & 1.24$\times$ & 1.19$\times$ & 1.18$\times$ & 1.13$\times$ & 1.09$\times$ & 1.14$\times$ & 1.05$\times$ & 1.04$\times$ \\
 & 5120  & 5120  & 2.87$\times$ & 2.52$\times$ & 2.26$\times$ & 1.09$\times$ & 1.13$\times$ & 1.09$\times$ & 1.19$\times$ & 1.10$\times$ & 1.26$\times$ & 1.05$\times$ & 1.04$\times$ \\
\midrule
\multirow{3}{*}{\mname{Gemma}{2B}}
 & 2048  & 16384 & 3.25$\times$ & 2.82$\times$ & 2.46$\times$ & 1.12$\times$ & 1.12$\times$ & 1.20$\times$ & 1.30$\times$ & 1.18$\times$ & 1.39$\times$ & 1.21$\times$ & 1.14$\times$ \\
 & 16384 & 2048  & 3.20$\times$ & 2.76$\times$ & 2.45$\times$ & 1.21$\times$ & 1.11$\times$ & 1.17$\times$ & 1.10$\times$ & 1.04$\times$ & 1.13$\times$ & 1.04$\times$ & 1.02$\times$ \\
 & 2048  & 2048  & 1.59$\times$ & 1.58$\times$ & 1.51$\times$ & 1.09$\times$ & 1.15$\times$ & 1.15$\times$ & 1.08$\times$ & 1.08$\times$ & 1.12$\times$ & 1.18$\times$ & 1.17$\times$ \\
\midrule
\multirow{3}{*}{\mname{Mistral}{7B-v0.3}}
 & 4096  & 14336 & 3.50$\times$ & 2.92$\times$ & 2.45$\times$ & 1.28$\times$ & 1.16$\times$ & 1.15$\times$ & 1.19$\times$ & 1.12$\times$ & 1.21$\times$ & 1.07$\times$ & 1.06$\times$ \\
 & 14336 & 4096  & 3.46$\times$ & 2.85$\times$ & 2.41$\times$ & 1.26$\times$ & 1.22$\times$ & 1.18$\times$ & 1.14$\times$ & 1.07$\times$ & 1.11$\times$ & 1.05$\times$ & 1.04$\times$ \\
 & 4096  & 4096  & 2.57$\times$ & 2.33$\times$ & 2.11$\times$ & 1.10$\times$ & 1.13$\times$ & 1.10$\times$ & 1.16$\times$ & 1.08$\times$ & 1.20$\times$ & 1.09$\times$ & 1.07$\times$ \\
\midrule
\multirow{3}{*}{\mname{Llama-3}{8B}}
 & 4096  & 14336 & 3.52$\times$ & 2.91$\times$ & 2.45$\times$ & 1.15$\times$ & 1.18$\times$ & 1.15$\times$ & 1.21$\times$ & 1.12$\times$ & 1.22$\times$ & 1.08$\times$ & 1.06$\times$ \\
 & 14336 & 4096  & 3.44$\times$ & 2.82$\times$ & 2.42$\times$ & 1.26$\times$ & 1.17$\times$ & 1.20$\times$ & 1.13$\times$ & 1.06$\times$ & 1.10$\times$ & 1.04$\times$ & 1.03$\times$ \\
 & 4096  & 4096  & 2.58$\times$ & 2.34$\times$ & 2.12$\times$ & 1.10$\times$ & 1.11$\times$ & 1.09$\times$ & 1.16$\times$ & 1.09$\times$ & 1.20$\times$ & 1.10$\times$ & 1.06$\times$ \\
\bottomrule
\end{tabular}
\end{table*}

\end{document}

%% file: math_commands.tex
\usepackage{amsmath,amsfonts,bm}

\def\eqref#1{equation~\ref{#1}}

\def\1{\bm{1}}

\DeclareMathAlphabet{\mathsfit}{\encodingdefault}{\sfdefault}{m}{sl}
\SetMathAlphabet{\mathsfit}{bold}{\encodingdefault}{\sfdefault}{bx}{n}

